\documentclass[10pt]{article}
\usepackage[preprint]{Style/tmlr}

\usepackage{Style/style}
\usepackage{Style/mdpmacros}
\usepackage[english]{babel}
\usepackage{microtype} 
\usepackage{array}

\def\Pr{\text{\upshape\bfseries\sffamily P}}
\def\EE{\text{\upshape\bfseries\sffamily E}}
\def\NN{\text{\upshape\bfseries\sffamily N}}
\def\PP{\text{\upshape\bfseries\sffamily P}}
\def\RR{\text{\upshape\bfseries\sffamily R}}

\usepackage{hyperref}
\usepackage{url}
\usepackage{subfiles}
\usepackage{multicol}
\usepackage{enumitem}

\title{Learning in Markovian bandits with non-observable states and constrained decision epochs}

\author{%
    \name Thomas Hira\thanks{Equal contribution.} \thanks{IRIT, Université de Toulouse, CNRS, Toulouse INP, Toulouse, France}
    \email thomas.hira@irit.fr
    \\
    \name Victor Boone\footnotemark[1] \footnotemark[2]
    \email victor.boone@irit.fr
    \\
    \name Urtzi Ayesta\footnotemark[2] \thanks{Ikerbasque-UPV/EHU, University of the Basque Country, Bilbao, Spain}
    \email urtzi.ayesta@irit.fr
    \\
    \name Ina Maria Verloop\footnotemark[2]
    \email ina-maria.verloop@irit.fr
}

\date{
    March, 2026
}

\def\month{MM}  
\def\year{YYYY} 
\def\openreview{\url{https://openreview.net/forum?id=XXXX}} 

\begin{document}
\allowdisplaybreaks

\maketitle

\begin{abstract}

    This paper studies the problem of regret minimization in Markovian bandits with \emph{non-observable states} and possibly \emph{constrained} decision epochs. 
    The focus is restricted to a ``pure'' regret benchmark, that compares the performance of the learning algorithm to the best \emph{pure policy} which---akin to optimal policies of stochastic bandits---picks the optimal arm from start to finish without ever switching.
    We introduce a generalization of rested Markovian bandits, \emph{self-degrading Markovian bandits}, for which pure policies are always asymptotically optimal.
    We show that without prior knowledge on the underlying bandit, the regret of algorithms that switch arms rarely necessarily scales super-logarithmically for every bandit, i.e., as $\omega(\log(T))$, where $T$ is the learning horizon.
    Despite the unreachability of the logarithmic regime, we design \algname{}, an optimistic algorithm inspired by \texttt{UCB}, of which the regret is nearly logarithmic. 
    Lastly, we show that given prior knowledge on the Markovian bandit in the form of a bound on the bias functions of its arm, a proper instantiation of \algname{} achieves $\OH(\log(T))$ regret.
    We further show that this prior knowledge allows for a $\OH(\sqrt{T \log(T)})$ worst-case regret bound for \algname{}.
    Notably, our regret bounds do not depend on the number of states of the underlying Markov chains.
    Our findings suggest that the non-observability of states is a mild inconvenience in self-degrading Markovian bandits. 
\end{abstract}

\section{Introduction}
    
        \ifSubfilesClassLoaded{
            \allowdisplaybreaks
            \onecolumn
        }{}

        Multi-armed bandits (MABs) provide a canonical framework for sequential decision-making under uncertainty, balancing exploration and exploitation in environments with unknown reward distributions \citep{lattimore_bandit_2020}. 
        In standard stochastic MABs, the controller pulls arms sequentially and receives rewards generated by the reward probability distributions associated to the arms it pulls. 
        In this work, we are interested in a variant of MABs in which 
        (1) arms are seen as third-party services, that hold a state that we cannot observe; and 
        (2) the controller may be locked to its decision until the reception of a specific signal. 
        This model is motivated by communication networks, where arms are communication channels through which the controller wants to send data, similarly to \cite{liu2010dynamic,hira2025optimal}, although our model can also be interpreted in terms of medical trials or machine maintenance problems as well. 

        For our problem, the relevant Reinforcement Learning (RL) framework is the one of \emph{Markovian bandits} \citep{gittins2011multi}, that is, bandits where the reward mechanism of every arm is determined by a Markov chain with an evolving state.
        We supplement the Markovian bandit structure with 
        (1) the \emph{non-observability of states} and
        (2) \emph{constrained decision epochs}, where the constraints come from the Markovian bandit itself.
        The assumption (1) means that the only information available to the controller are the rewards it obtained through the interaction;
        the other assumption (2) enforces that once an arm is chosen, the controller is to remain committed to it until a special reward is earned.
        Notably, we assume that a single arm is activated at a time---this meets the framework of standard stochastic bandits, although in opposition to the asymptotic theory of Markovian bandits where the controller pulls a \emph{positive proportion} of arms \citep{gittins2011multi}.
        In the end, the resulting setting is especially favorable to the design of learning algorithms inspired from the setting of \emph{rested} Markovian bandits \citep{liu2011logarithmic, liu2012learning, tekin_online_2011, tekin2012online}.
    
        In this paper, we present a learning algorithm for our problem, \algname{}, discuss the importance of the availability of prior information on the properties of the Markovian bandit, and provide regret guarantees for \algname{} depending on the degree of strength of that prior information.
        By ``\emph{regret}'', we mean the regret as defined in stochastic bandits \citep{lattimore_bandit_2020}: the difference between the reward accumulated by the controller (that is learning the best arms) and what the optimal policy can do (that has access to a complete description of the bandit before-hand).

        \subsection{Related work}

        Our work is at the intersection of several literature: 
        (1) it is at the interplay of \emph{rested} and \emph{restless} Markovian bandits;
        (2) it incorporates considerations of \emph{partial observability}---although the latter are pretty common already in the literature on restless Markovian bandits;
        and
        (3) it belongs to the family of RL problems with constrained decision epochs.
        These three literature are discussed in separated paragraphs below. 

        \paragraph{Markovian bandits (MBs).}
        The literature to which this work is the closest to are Markovian bandits, as they stand as the foundation of our learning problem.
        The literature of RL on Markovian bandits can be organized along two axes: the \emph{observation model}, ranging from full states observation to reward-only feedback; and the \emph{regret benchmark}, ranging from the best pure policy (that plays the same arm from start to finish) to the best general policy (that may switch from one arm to another in a complex fashion). 
        The regret benchmark is directly related to the ``type'' of Markovian bandit one is looking at; pure policies are optimal for \emph{rested} bandits (for which non-activated arms stay idle), while they are sub-optimal for \emph{restless} bandits (for which the states of non-activated arms may keep updating).
        In the sequel, we speak of \emph{pure regret} when the performance of the controller is compared to those of \emph{pure} policies, and we speak of \emph{general regret} otherwise.

        \begin{itemize}
            \item \emph{Fully observable MBs.}
                In the fully observable setting, several works have derived logarithmic upper-bounds on the pure regret by extending the principles of \texttt{UCB} \citep{auer_using_2002}, by relying on regeneration arguments (\cite{tekin_online_2011, tekin2012online}) or forced exploration mechanisms \citep{liu2011logarithmic, liu2012learning}. 
                The algorithms displayed by those works are arguably the closest to ours (particularly \texttt{RCA} of \citep{tekin_online_2011,tekin2012online}), but unlike us, they do rely on the observation of states.
                In the fully observable setting, other works worth mentioning include \citep{dai2011non,dai2014online}, that compare to the general optimal policy for a very specific family of Markovian bandits.
                We intend to work with general (although finite) Markovian bandits in this paper.

            \item \emph{Partially observable MBs.}
                A second line of work assumes that only the state of the currently activated arm is observed. 
                For this observation model, \cite{ortner_regret_2012} compares against the best general policy and shows that regret is, in general, $\Omega(\sqrt{T})$. 
                Then, they design \texttt{\upshape colored-UCRL2}, that matches this lower bound up to logarithmic factors.
                On the same observation setting, \cite{jung_regret_2019} develop an algorithm based on Thompson sampling that achieves $\smash{\widetilde{\OH}(\sqrt{T})}$ Bayesian regret in the presence of periodic resets. 
                This was later extended by \cite{jung_thompson_2019}, removing the periodic resets' mechanism.   
                These two works highlight an important aspect of this observation model: when competing against the best general policy and if only the state of the currently played arm is observed, one cannot do better than $\sqrt{T}$ general regret guarantees.
                The pure regret benchmark, although weaker, allows for regret of order $\log(T)$, even in more frugal observation settings. 

            \item \emph{``Non-observable'' MBs, or, MBs with ``full bandit feedback''.}
                Assuming that only the rewards are observed, sometimes referred to as \emph{full bandit feedback}, is the most severe observation model that one can take. 
                It is adopted by \cite{jiang2023online} in particular, who consider restless bandits with reward-only observations and a single activated arm, similarly to this work.
                Their algorithm achieves a Bayesian regret upper bound of $\smash{\widetilde \OH (\sqrt{T})}$ with respect to the best \emph{general} policy. 
                Their algorithm is based on Thompson sampling with episodic explore-then-commit (\texttt{\upshape TSEETC}).
                However, in opposition to our work, they do not have constraints on decision epochs. 
        \end{itemize}

        \paragraph{Constrained decision epochs.}
        Decision epochs are more common in the setting of multi-armed bandits.
        For instance, \cite{komiyama2013multi} propose a framework where decision epochs are predetermined, while \cite{NEURIPS2019_88fee042} propose one where playing an arm makes it unavailable for a fixed number of time slots thereafter. 
        That second work is close in spirit to \emph{batched bandits}, where the decision maker has an upper bound on the number of times it is allowed to change arms.
        Batched bandits have a growing rich literature \citep{gao2019batched,esfandiari2021regret,cesa2013online,dekel2014bandits,pmlr-v139-rouyer21a}.
        In these works however, the constraints are \emph{exogenous} to the bandit.
        This is in opposition to the setting that we design in this paper, where the constraints emerge from the Markovian structure of arms directly, making the constraints \emph{endogenous}.
        In particular, we will see that these constraints eventually disappear in the asymptotic regret bounds that we derive for our algorithm.

        \subsection{Contributions and outline of the paper}
        We first motivate the choice of \emph{pure policies}, as a comparison point for our regret benchmark, that we call \emph{pure regret} (simplified to \emph{regret} in the text). 
        In particular, we provide a general class of restless Markovian bandits in which such policies are asymptotically optimal: \emph{self-degrading} Markovian bandits.
        This class generalizes rested Markovian bandits, but further includes non-trivial restless ones, ranging from the scheduling of non-preemptive Gilbert-Elliott channels from \cite{hira2025optimal} to bandits with infinitely many states inspired by the framework of \cite{nerlove_optimal_1962}.  
        In \Cref{section_learning_partial_information}, we show that in general, the expected regret is $\omega(\log(T))$ where $T$ is the number of learning steps (\Cref{theorem_lowerbound_dependent_universal}).\footnote{
            The notation $f = \omega(g)$ means that $f$ grows strictly faster than $g$, i.e., that $g = \oh(f)$.
        }
        In particular, the well-known regret guarantees of $\OH(\log(T))$ from stochastic bandits \citep{lattimore_bandit_2020} cannot be achieved.
        We then design \texttt{UCB} for \texttt{N}on-\texttt{O}bservable \texttt{M}arkovian bandits (\algname, \Cref{algo:UCB-POMB}), that approaches the idealistic $\log(T)$-regret guarantees arbitrarily closely (\Cref{theorem_upperbound_f_diverging}). 
        Finally, in \Cref{section_bounded_bias}, we discuss a structural assumption under which regret guarantees can be improved: If the biases of the Markov reward processes underlying every arm are uniformly bounded, \algname{} achieves regret $\OH(\log(T))$ with the proper tuning (\Cref{theorem_upperbound_log}).
        We further provide a worst-case regret bound $\OH(\sqrt{T \log(T)})$ under this assumption (\Cref{theorem_worst_case_sqrt}).
        At last, \Cref{sec:confidence_tuning} addresses the tuning of \algname{} with respect to the confidence parameters. 
        An important take-away from our results is that constrained decision epochs do not harm the achievable regret guarantees, and that the non-observability of states does not require involved techniques from the POMDP literature as long as one sticks to the pure regret benchmark, or to self-degrading restless bandits.
    
    \ifSubfilesClassLoaded{
        
        \bibliographystyle{plainnat}
        \bibliography{biblio}
    }{}
    
\end{document}

\section{Learning non-observable Markovian bandits}
\label{sec:POMB}

We begin this paper by describing the learning problem and introducing notations.
Our learning problem is a class of Markovian bandits (\Cref{section_model_description}) in which states are non-observable and decision epochs are possibly constrained. 
In \Cref{section_pure_policies}, we motivate \emph{pure policies} as a benchmark by identifying subclasses of restless bandits in which they are actually optimal. 
Lastly, in \Cref{section_rarely_switching}, we motivate the focus on learning algorithms that switch arms rarely.

\subsection{Non-observable Markovian bandits with decision epochs}
\label{section_model_description}

We consider a Markovian bandit with arms $\arms$, where every arm $\arm \in \arms$ is modeled by a Markov decision process $\model_\arm \equiv (\states_\arm, \braces{0, 1}, \rewardDistribution_\arm, \kernel_\arm^0, \kernel_\arm^1, \initstate_\arm)$ with \emph{activated/non-activated} actions, states $\states_\arm$, a fixed initial state $\initstate_\arm$, a reward function agnostic to action choices $\rewardDistribution_\arm : \states_\arm \to \probabilities([0, 1])$, and two transition kernels $\kernel_\arm^0, \kernel_\arm^1 : \states_\arm \to \probabilities(\states_\arm)$ that govern a state's evolution depending on its activation or non-activation. 
Mean rewards are denoted by $\reward_\arm (\state_\arm) := \integral x\,\dd \rewardDistribution_\arm (\state_\arm) (x)$.
The resulting system is obtained by bundling arms in parallel as $\model \equiv (\model_\arm)_{\arm \in \arms}$.

The interaction between the system and the controller goes as follows.
At time $t \ge 1$, the controller selects a \emph{single} arm $\Arm(t) \in \arms$ to activate, and the system's state evolves as
\begin{equation*}
    \State_{\arm}(t+1) 
    \sim 
    \begin{cases}
        \kernel_{\arm}^1(\State_{\arm} (t)) & \text{if $\Arm(t) = \arm$;}
        \\
        \kernel_{\arm}^0(\State_{\arm} (t)) & \text{if $\Arm(t) \ne \arm$;}
    \end{cases}
    \; .
\end{equation*}
Moreover, one receives an immediate reward $\Reward(t+1) \sim \rewardDistribution_{\arm}(\State_{\arm}(t))$ for $\arm = \Arm(t)$,\footnote{
    Because the activated arm is the only one to produce a reward, the assumption that rewards are agnostic to actions can be made without loss of generality. 
    Equivalently, we could have assumed non-activated arms produce no reward. 
} independently of the state's generation. 
By convention, $\Reward(1) = 0$.
All arms evolve independently conditionally on the current state and activated arm. 
In the sequel,  the \emph{number of pulls} of arm $\arm$ will be denoted $\visits_\arm (T) := \sum_{t=1}^{T-1} \indicator{\Arm(t) = \arm}$.
To simplify the exposition, we further assume that transition kernels under activations $\kernel_\arm^1$ are unichain --- this assumption is fairly standard in restless Markovian bandits. 

\begin{assumption}
    For every $\arm \in \arms$, $\kernel_\arm^1$ is \emph{unichain}, i.e., it has a unique recurrent component of states. 
\end{assumption}

%
 
\subsubsection{Constrained decision epochs}

In opposition to standard restless Markovian bandits, we are interested in systems in which the controller is constrained to switch arms \emph{only at decision epochs}, induced by the stream of rewards.
More specifically, epochs consist of an increasing sequence $(\tau_i)_{i \ge 1}$ of stopping times under the filtration induced by the system's rewards $(\Reward(t))_{t \ge 1}$.
In-between two decision epochs, the controller is not allowed to switch arms, that is,  
\begin{equation}
\label{equation_lockin_constraint}
    \forall i,
    \forall t \in \braces{\tau_i, \ldots, \tau_{i+1}-1},
    \quad
    \Arm(t) = \Arm(\tau_i). 
\end{equation}
Throughout, we make two assumptions on the structure of decision epochs.
In \Cref{assumption_strong_decision_epochs} below, $H(t) = (S(1), A(1), R(2), S(2), A(2), R(3), \ldots, S(t))$ is the \emph{history of play}, i.e., the vector of every state, every played arm and every emitted reward up to time $t$. 

\begin{assumption}
\label{assumtion_activates_at_service}
    There exists a measurable set $\activationset \subseteq [0, 1]$ such that
    $
        \tau_{i+1} 
        =
        \inf \braces*{
            t > \tau_i
            :
            \Reward(t) \in \activationset
        }.
    $
\end{assumption}

\begin{assumption}
\label{assumption_strong_decision_epochs}
    There exists $\strongdelayconstant \in \RR_+$ such that, for all $t \ge 1$,
    $
        \EE \brackets*{
            \inf \braces*{
                \tau_i - t
                :
                \tau_i \ge t
            }
            |
            \History(t)
        }
        \le
        \strongdelayconstant
        .
    $
\end{assumption}

\Cref{assumtion_activates_at_service} means that decision epochs happen when the immediate reward falls within a special set.
For instance, if $\activationset = [0, 1]$, then the decision epochs are trivial with $\tau_i = i$ for all $i \ge 1$, so that \Cref{assumtion_activates_at_service} captures systems without constraints on decision epochs in particular.
\Cref{assumption_strong_decision_epochs} claims that the expected time before the next decision epoch is bounded uniformly over time conditionally on the history.
This assumption holds, for example, if, for every arm $\arm \in \arms$, there exists a state $\state_\arm \in \states_\arm$ that is recurrent under $\kernel_\arm^1$ and such that $\rewardDistribution(\state_\arm)(\activationset) > 0$.
Under that scenario, $C_\tau$ can typically be bounded in terms of the diameter of $\kernel_\arm^1$ (expected time required to travel from a state to another under $\kernel_\arm^1$) and the probability $\rewardDistribution(\state_\arm)(\activationset)$ to emit the right reward from favorable states. 

\subsubsection{Full bandit feedback: non-observability of states}
While the number of arms is known to the controller, the latter is unaware of the current states $\State_{\arm}(t)$ of arms, nor of how many states $\abs{\states_\arm}$ they can possibly be in. 
In other words, states are non-observable.

Non-observability implies that the choice of arm $\Arm(t)$ depends exclusively on the past \emph{history of observations} $\ObservedHistory(t) := (\Arm(1), \Reward(2), \Arm(2), \ldots, \Reward(t))$.
In formal terms, the random arm $\Arm(\tau_i)$ is $\sigma(\ObservedHistory(t))$-measurable.
Note that for each decision epoch $\tau_i$, we have $\braces{\tau_i \le t} \in \sigma (\Reward(2), \ldots, \Reward(t))$. That is, decision epochs are measurable with respect to the aggregate reward vector. 
This assumption is absolutely necessary for our work, as the controller would be unable to properly change arm otherwise.

    \ifSubfilesClassLoaded{
        \allowdisplaybreaks
        \onecolumn
        
        In this file, we detail the learning setting.
    }{}

\subsection{Our learning framework: pure policies and regret}
\label{section_pure_policies}

Our objective is to design learning algorithms that maximize aggregate rewards $\Reward(1) + \ldots + \Reward(t)$, for $t \to \infty$, without knowledge of the underlying environment. 

Let $\models$ be a collection of Markovian bandits as in \Cref{section_model_description}, that is,
\begin{equation*}
    \models 
    \subseteq
    \bigcup_{
        \begin{subarray}{c}
            \states_\arm \subseteq \NN
            \\ 
            \arm \in \arms
        \end{subarray}
    }
    \braces*{
        \parens*{
            \states_\arm,
            \initstate_\arm,
            \braces{0, 1}, 
            \rewardDistribution_\arm, 
            \kernel_\arm^0, 
            \kernel_\arm^1
        }_{\arm \in \arms}
        :
        \text{$\kernel_\arm^1$ is unichain}
    } \; .
\end{equation*}
In particular, all instances of $\models$ have the same set of arms $\arms$, but a given arm may have different state spaces from one instance to another.
The collection $\models$ shall contain every environment for which we want the learning algorithm to converge to optimal play.

\begin{definition}[Pure policies]
    A policy is \emph{pure} if it plays a single action over time, i.e., if $\Arm(t) = \Arm(t+1)$ for all $t \ge 1$. 
    The pure policy that plays $\arm \in \arms$ is denoted~$\policy_\arm$.
\end{definition}

The usual performance measure of a learning algorithm under a model $\model \in \models$ is to compare its aggregate rewards $\Reward(1) + \ldots + \Reward(t)$ against the best algorithm specifically designed for $\model$. 
Motivated by applications where \emph{pure policies} are optimal, we choose those as a benchmark in order to provide much stronger guaranties for those problems than what could be attained in all generality.  
As such, the difference in aggregate reward between the best pure policy and the learning algorithm is called the \emph{regret}.
Formally, the (expected) regret $\regret(T) \equiv \regret(T; \model, \learner$) of a learner $\learner$ under a model $\model$ up to $T \ge 1$ is 
\begin{equation}
\label{equation_regret_definition}
    \regret(T)
    :=
    \max_{\arm \in \arms} \braces*{
        V^{\policy_\arm}(T; \model)
        -
        V^{\learner}(T; \model)
    }
\end{equation}
where $V^{\learner}(T; \model) := \EE^{\learner, M} \brackets{\sum_{t=1}^T \Reward(t)}$ is the expected amount of reward gathered by $\learner$ under $\model$ within $T$ steps. 
Obviously, the smaller the regret is, the closer the algorithm is to optimal play. 

\paragraph{Rewriting policy values via the gain.}
The value function of the policy $\policy_\arm$ can be expressed with respect to the \emph{gain} and the \emph{bias}, with the well-known formula $V^{\policy_\arm}(t) = (t-1) \gain_\arm (\initstate_\arm) + \EE^{\policy_\arm}[\bias_\arm (\initstate_\arm) - \bias_\arm (\State_\arm (t))]$, proved in \Cref{appendix_score_of_pure_policies} for completeness.
Both the gain and the bias depend on the initial state.
Writing $\EE_{\state_\arm}^{\policy_\arm}[-]$ the expectation under the process induced by playing $\policy_\arm$ in $\model$ with $\model_\arm$ initialized in $\state_\arm \in \states_\arm$, the gain and the bias functions are given by\footnote{
    The gain and the bias, as defined in \Cref{equation_gain,equation_bias}, depend on the instance $\model$. 
    The dependency has been dropped for typographic convenience.
    When the dependency may lead to confusion, we will write $\gain_\arm (\state_0; \model)$ and $\bias_\arm (\state_0; \model)$.
}
\begin{align}
\label{equation_gain}
    \gain_\arm (\state_\arm)
    & :=
    \lim_{T \to \infty}
    \EE_{\state_\arm}^{\policy_\arm}
    \brackets*{
        \frac 1T
        \sum_{t=1}^T 
        \Reward(t)
    },
    \\
\label{equation_bias}
    \bias_\arm (\state_\arm)
    & :=
    \lim_{T \to \infty}
    \EE_{\state_\arm}^{\policy_\arm}
    \brackets*{
        \sum_{t=1}^T 
        (\Reward(t) - \gain_\arm (\state_\arm))
    }.
\end{align}
Since $\kernel_\arm^1$ is unichain, the gain $\gain_\arm(\state_\arm)$ is independent of the initial state $\state_\arm \in \states_\arm$, and we can simply write $\gain_\arm$.
We denote the \emph{optimal gain} by 
\begin{equation}
    \gain_* := \max_{\arm \in \arms} \gain_\arm.
\end{equation}
and \emph{bias span} of arm $\arm$ by $\vecspan(\bias_\arm) := \max_{\state \in \states_\arm}\bias_\arm(\state) - \min_{\state \in \states_\arm}\bias_\arm(\state)$.   

    \ifSubfilesClassLoaded{
        
        \bibliographystyle{plainnat}
        \bibliography{biblio}
    }{}

\end{document}

\subsection{Rarely switching algorithms and pseudo-regret}
\label{section_rarely_switching}

In this section, we define the concept of rarely switching policies, which provide a natural class of learning rules when considering regret with respect to pure policies. Indeed, such a learner should be able to identify the best arm by constructing a good estimator for the gain $\gain_\arm$ for every arm.
The gain of a pure policy $\policy_\arm$ only depends on the states that are visited infinitely many times under $\policy_\arm$; the set of \emph{recurrent states} of $\policy_\arm$.
When playing an arm to estimate $\gain_\arm$, the process may initially lie outside this set, requiring several steps before reaching it and thus before the estimation of $\gain_\arm$ can begin. 
The task is further complicated by the fact that $\abs{\states_\arm}$, the number of states of arm $\arm$, is unknown.

\begin{figure}[ht]
    \centering
    \resizebox{0.5\linewidth}{!}{\begin{tikzpicture}[smallmdp]
        \node[small state] (0) at (0, 0) {$0$};
        \node (a) at (0) [right=8.5mm] {$\ldots$};
        \node[small state] (i) at (0) [right=20mm] {$i$};
        \node[small state] (i+1) [right of=i] {$i+1$};
        \node (b) at (i+1) [right=8.5mm] {$\ldots$};
        \node[small state] (n) at (i+1) [right=20mm] {$n$};

        \draw[->] (0) to (a);
        \draw[->] (a) to (i);
        \draw[->] (i) to
            node[midway, above] {\tiny $\reward = \frac 12$}
            (i+1);
        \draw[->] (i+1) to (b);
        \draw[->] (b) to (n);
        \draw[->] (n) 
            to[loop right]
            node[midway, right] {\tiny $\reward = \frac 23$}
            (n);
        \draw[red, ->] (i) to[in=-25, out=180+25] (0); 
        \draw[red, ->] (i+1) to[in=-40, out=180+40] (0); 
        \draw[red, ->] (n) to[in=-55, out=180+55] (0); 
    \end{tikzpicture}}
    \vspace{-2em}
    \caption{
    \label{figure_reach_me}
        An arm that requires to be activated at least $n$ times in a row to be estimated correctly. 
        All transitions are deterministic, with $\kernel_\arm^1$ represented in black and $\kernel_\arm^0$ in \textcolor{red}{red}.
    }
\end{figure}
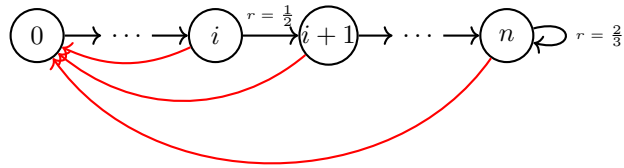

An example is provided in \Cref{figure_reach_me}, where state~$n$ is absorbing under $\policy_\arm$ and the gain of the arm is $\gain_\arm = \frac 23$.
The latter can only be estimated accurately if, infinitely often during the learning process, the controller activates the arm more than $n$ times in a row.
As $n \ge 1$ can be arbitrarily large in general, it becomes tempting to reduce the number of times one switches from one arm to another as much as possible, and to measure the regret by \emph{counting} the number of times bad arms are pulled.
This leads to the introduction of the \emph{pseudo-regret}:
\begin{equation}
\label{equation_pseudo_regret}
    \pseudoregret(T)
    :=
    \sum_{\arm \in \arms}
    \parens*{\gain_* - \gain_\arm}
    \visits_{\arm} (T).
\end{equation}
It is named after its eponym version in multi-armed bandits \cite[§4.9]{lattimore_bandit_2020}. 

The pseudo-regret measures the performance of the learning algorithm as a sum of instantaneous costs: the individual cost of arm $\arm$ is given by the \emph{gain-gap} $\gain_* - \gain_\arm$ (how far its gain is from $\gain_*$), and the pseudo regret is the sum of gain-gaps over time. 
The pseudo-regret is much easier to analyze than the regret, because it is directly expressed in terms of the number of pulls to sub-optimal arms. 
If the pseudo-regret is large, then the regret is large.
The converse is false in general. We now go over a class of learning algorithms for which the regret can nevertheless be upper-bounded in terms of the pseudo-regret: \emph{rarely-switching algorithms}. 

\begin{definition}[Rarely switching algorithms]
\label{definition_rarely_switching_arms}
    Given a sublinear smooth concave function $\switchingfunction : \RR_+ \to \RR_+$, we say that the algorithm is \emph{$\switchingfunction$-rarely switching} if for all $\model \in \models$, it holds 
    \begin{equation*}
        \forall \arm \in \arms,
        \;
        \forall T \ge 1,
        \quad
        \sum_{t=1}^{T-1} 
        \indicator{\Arm(t) = \arm, \Arm(t+1) \ne \arm}
        \le
        \switchingfunction(\visits_{\arm}(T)).
    \end{equation*}
    In other words, an algorithm is \emph{rarely switching} if the number of switches to arm $\arm$ grows sub-linearly with respect to $\visits_\arm (t)$, the number of pulls to arm $\arm$.
\end{definition}

When algorithms are $\switchingfunction$-rarely switching, the \emph{pseudo-regret} \eqref{equation_pseudo_regret} becomes a natural proxy for the regret:
If the algorithm switches arms rarely, then the two quantities are approximately the same (at first order), as stated by \Cref{proposition_regret_and_pseudo_regret} thereafter.
This result encourages to fully focus on the study of pseudo-regret.

\begin{proposition}[Regret and pseudo-regret]
\label[proposition]{proposition_regret_and_pseudo_regret}
    Let $\model \in \models$ be an arbitrary model with a unique optimal arm.
    Let $\learner$ be a $\psi$-rarely switching algorithm.
    Then, there exists a positive smooth concave function $\switchingfunction_0 = \OH(\switchingfunction)$ such that, for every horizon $T \ge 1$, we have 
    \begin{equation*}
        \abs*{
            \regret(T) 
            - \EE [\pseudoregret(T)]
        }
        \le
        \switchingfunction_0 \parens*{
            \EE [\pseudoregret(T)]
        }.
    \end{equation*}
\end{proposition}

The proof of \Cref{proposition_regret_and_pseudo_regret} is deferred to \Cref{appendix_rarely_switching}.

Because $\switchingfunction_0 = \OH(\psi)$ is a sublinear function, \Cref{proposition_regret_and_pseudo_regret} states that the expected regret of a rarely switching algorithm is $\EE[\pseudoregret(T)] + \oh(\EE[\pseudoregret(T)])$.
Accordingly, the regret of such algorithms can be analyzed by controlling the pseudo-regret directly, which is much easier to handle.

        \ifSubfilesClassLoaded{
            \allowdisplaybreaks
            \onecolumn
        }{}
    \section{Why comparing to pure policies is relevant: self-degrading restless bandits}
    \label{section:pure_policies}
    
    For restless bandits in general, comparing the performance of the learning algorithm to that of pure policies is much weaker than comparing to general policies, because optimal policies are typically non-pure.
    The only apparent advantage is that if one asks for regret guarantees exclusively against pure policies, then one allows for simpler learning algorithms, which will make decision epochs much easier to manage.
    This being said, although they are sub-optimal in general, pure policies \emph{can} be asymptotically optimal.
    This is what we discuss in this section:
    We exhibit a sub-class of restless bandits, that we name \emph{self-degrading restless bandits} (\Cref{definition_self_degrading}), for which the set of optimal policies intersects the one of pure policies. 
    As such, defining the regret in comparison with pure policies does not weaken the learning benchmark for those classes of restless bandits.

    \begin{definition}[Self-degrading restless bandits]
    \label{definition_self_degrading}
        Given an instance of restless bandit $\model \in \models$ together with an activation set $\activationset$, an arm $\arm \in \arms$ is said \emph{self-degrading} if, whatever the controller $\learner$ respecting $\activationset$,
        \begin{equation*}
            \Arm(t) \neq \arm 
            \implies
            \EE^{\model, \learner} \brackets*{
                \bias_\arm(\State_\arm (t+1)) 
                \middle|
                \ObservedHistory(t)
            }
            \le
            \EE^{\model, \learner} \brackets*{
                \bias_\arm(\State_\arm (t))
                \middle|
                \ObservedHistory(t)
            }
            \; .
        \end{equation*}
        The instance $\model$ is said \emph{self-degrading} if all of its arms are. 
        The set of all self-degrading instances of Markovian bandits is written $\models_{\mathrm{sd}}$.
    \end{definition}

    Said differently, an arm is \emph{self-degrading} if when left nonactivated between consecutive decision epochs, it tends to move
    towards worse and worse states conditionally on what the controller observes (the subtleties of that conditioning will appear clearly in the example on Gilbert-Elliott channels). 
    This property is enough to make switches between arms undesirable, so that pure policies become asymptotically optimal, see \Cref{proposition_self_degrading} below.
    As a matter of fact, this result is a bit stronger than saying that pure policies are asymptotically optimal: it states that whatever the way the controller chooses its actions, it cannot perform better than the best pure policy up to a \emph{constant} additive factor. 

    \begin{proposition}[Pure policies are optimal in self-degrading instances]
    \label[proposition]{proposition_self_degrading}
        The regret is bounded below on self-degrading instances.
        That is, for every self-degrading instance of restless bandits $\model \in \models_\mathrm{sd}$, there exists a constant $C_\model < \infty$ such that, whatever the learning algorithm $\learner$, we have
        \begin{equation*}
            \EE^{\model, \learner}\brackets*{
                \sum_{t=1}^T
                \Reward(t)
            }
            \le
            \max_{\arm \in \arms}
            V^{\policy_\arm} (T; \model)
            + C_\model
            \; .
        \end{equation*}
    \end{proposition}

    The proof of \Cref{proposition_self_degrading} is deferred to \Cref{appendix:self_degrading_regret}. 

    \subsection{Regret and pseudo-regret in self-degrading instances}
    For self-degrading instances, the regret can be lower-bounded with respect to the pseudo-regret, as shown by \Cref{proposition_self_degrading_regret} below, proved in \Cref{appendix:self_degrading_regret}. This result can be seen as complementary to \Cref{proposition_regret_and_pseudo_regret} above, as it establishes a relation between regret and pseudo-regret without relying on the \emph{rarely switching} property described in \Cref{section_rarely_switching}.

    \begin{proposition}[Regret and pseudo-regret in self-degrading instances]
    \label[proposition]{proposition_self_degrading_regret}
        The regret is bounded below by the pseudo-regret on self-degrading instances.
        That is, for every self-degrading instance of restless bandits $\model \in \models_\mathrm{sd}$, there exists a constant $C_\model < \infty$ such that, whatever the learning algorithm $\learner$, 
        \begin{equation*}
            \forall T \ge 1,
            \qquad
            \regret(T; \model, \learner)
            \ge
            \EE^{\model,\learner} \brackets*{
                \pseudoregret(T)
            }
            - C_\model
            \; .
        \end{equation*}
    \end{proposition}

    The proof of \Cref{proposition_self_degrading_regret} invokes the same arguments as \Cref{proposition_self_degrading}---in fact, it is easy to see that \Cref{proposition_self_degrading} is an immediate consequence of \Cref{proposition_self_degrading_regret}, for $\pseudoregret(T) \ge 0$.

    \subsection{Examples of self-degrading restless bandits}
    \label{section_self-degrading_examples}

    We now provide three examples of self-degrading restless bandits: rested bandits with \Cref{example:rested_bandits}, Gilbert-Elliott channels with \Cref{example:GE_channels}, and a toy version of the Nerlove-Arrow model with \Cref{example:NA_toy_model}.

    \begin{example}[Rested bandits]
    \label[example]{example:rested_bandits}
        A Markovian bandit is said \emph{rested} if inactive arms remain frozen in their current state.
        That is, for all arm $\arm \in \arms$,
        \begin{equation*}
            \kernel_\arm^0 = I
        \end{equation*}
        where $I$ is the identity matrix.
        For nonactivated arms, $\State_\arm (t+1) = \State_\arm (t)$, so such instances are self-degrading for every activation set $\activationset$.
    \end{example}
    
    \begin{example}[Gilbert-Elliott channels]
    \label[example]{example:GE_channels}
        A Gilbert-Elliott channel \citep{gilbert60} is a Markov chain with states \texttt{ON}/\texttt{OFF} representing the availability of the channel, and a non-null reward function from the state \texttt{ON}.
        Those channels evolve with the same transition kernel whether active or inactive.
        \Cref{figure_gilbert_elliot} below represents such a channel.
        Gilbert--Elliott channels have been widely studied in the literature \citep{Elliott1963EstimatesOE, Mushkin1989CapacityAC, Halinger2011TheGM, polyanskiy2011dispersion} etc... We present one instantiation from \cite{hira2025optimal} where the authors study non-observable Gilbert--Elliott channels with a non-preemptive scheduling constraint, that is, where decision epochs correspond to times at which the received reward is not zero. 
        Those assumptions on the observation framework and the decision epochs define a \emph{self-degrading} restless bandit.
        
        \begin{figure}[ht]
            \centering
            \begin{tikzpicture}
                \node[state, minimum width=3em] (on) at (0, 0) {\texttt{ON}};
                \node[state, minimum width=3em] (off) at (4, 0) {\texttt{OFF}};

                \draw[->, thick] (on) to[bend left] node[midway, above] {$y_\arm$} (off);
                \draw[->, thick] (off) to[bend left] node[midway, below] {$x_\arm$} (on);
                \draw[->, thick] (on) to[in=90+33,out=90-33,looseness=5] node[midway, above] {$1 - y_\arm$} (on);
                \draw[->, thick] (off) to[in=90+33,out=90-33,looseness=5] node[midway, above] {$1 - x_\arm$} (off);

                \draw[->, decorate, decoration=snake] (on) to (-1.4, 0) node[anchor=east] {$R \sim \mathrm{Bernoulli}(\mu_\arm)$};
                \draw[->, decorate, decoration=snake] (off) to (5.4, 0) node[anchor=west] {$R \sim \mathrm{Dirac}(0)$};
            \end{tikzpicture}
            \caption{
                \label{figure_gilbert_elliot}
                A Gilbert-Elliott channel with parameters $x_\arm, y_\arm, \mu_\arm$.
            }
        \end{figure}
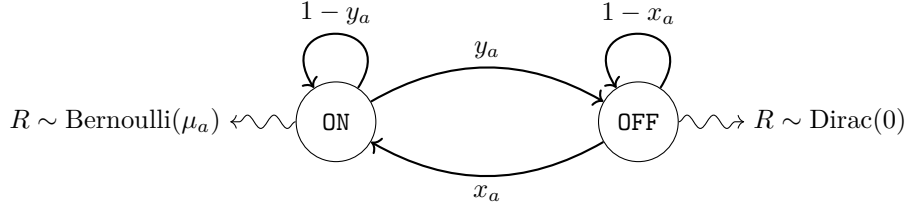
        
        Formally, given $x_\arm, y_\arm, \mu_\arm \in (0, 1)$, the associated Gilbert-Elliott channel has transition kernels and reward distributions given by
        \begin{equation*}
            \kernel_\arm^1 
            = 
            \kernel_\arm^0 
            =
            \begin{pmatrix}
                1-y_\arm & y_\arm \\
                x_\arm & 1-x_\arm
            \end{pmatrix},
            \qquad 
            \rewardDistribution_\arm (\texttt{ON}) = \mathrm{Bernoulli}(\mu_\arm)
            \quad \mathrm{and} \quad
            \rewardDistribution_\arm (\texttt{OFF}) = \mathrm{Dirac}(0), 
            \;
        \end{equation*}
        with an initialization $\initstate_\arm = \texttt{ON}$.\footnote{
            This choice of initial state is made for simplicity.
            The channel is not self-degrading if initialized in state \texttt{OFF}, but it can be shown that it satisfies the self-degrading property as soon as \emph{one} non-null reward is observed.
            In other words, Gilbert-Elliott channels are \emph{eventually self-degrading} regardless of their initialization.
        }
        Although the state of the channel cannot be observed, receiving a positive reward reveals that the channel was \texttt{ON}.
        On the other hand, the absence of a reward does not unequivocally identify the state of the channel.
        The controller can choose a different channel only after receiving a reward, that models a successful transmission through the channel.
        Accordingly, the activation set is $\activationset = \braces{1}$ (see \Cref{assumtion_activates_at_service}).
        We prove in \Cref{appendix_GE_channels} that if the chain $\kernel_\arm$ is \emph{positively auto-correlated}, the resulting restless bandit is self-degrading. 
        \emph{Positive auto-correlation} is a more general concept that translates, in this case, to the algebraic condition $1-x_\arm-y_\arm>0$ (every state is absorbing enough). 
        This condition represents some inertia to change: a channel has a tendency to stick to the state it is on rather than chattering between states.
        This inertia is only relevant when states are non-observable.
        Indeed, if $\bias_\arm$ is the bias function of the channel, then the inequality
        \begin{equation*}
            \bias_\arm (\State_\arm (t+1))
            \le
            \bias_\arm (\State_\arm (t))
        \end{equation*}
        is \emph{not} true in general (when $\State_\arm (t) = \texttt{OFF}$).
        This implies that pure policies are \emph{not} optimal for Gilbert-Elliott channels in the setting where states are observed.
        However, for $\initstate_\arm = \texttt{ON}$, we show in the Appendix that
        \begin{equation*}
            \EE^{\model, \learner} \brackets*{
                \bias_\arm (\State_\arm (t+1))
                |
                \ObservedHistory(t)
            }
            \le
            \EE^{\model, \learner} \brackets*{
                \bias_\arm (\State_\arm (t))
                |
                \ObservedHistory(t)
            }
        \end{equation*}
        holds indeed.
        As controllers that do not observe states take their decisions from $\ObservedHistory(t)$ alone, the above property is enough to guarantee that such controllers can never outperform pure policies.
    \end{example}

    \begin{example}[Nerlove-Arrow toy model]
    \label[example]{example:NA_toy_model}
        For the last example, we introduce a Markovian bandit inspired from the framework of \cite{nerlove_optimal_1962}, that models the growth of skills or interest of human subjects. 

        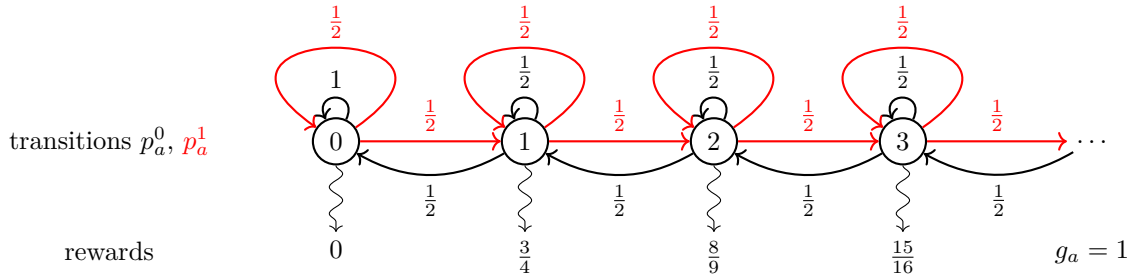
\begin{figure}[ht]
            \centering
            \begin{tikzpicture}
                \node at (-3, 0) {transitions \textcolor{black}{$\kernel_\arm^0$}, \textcolor{red}{$\kernel_\arm^1$}};
                \node[anchor=north] at (-3, -1.2) {rewards};
            
                \node[state, thick] (0) at (0, 0) {$0$};
                \node[state, thick] (1) at (2.5, 0) {$1$};
                \node[state, thick] (2) at (5, 0) {$2$};
                \node[state, thick] (3) at (7.5, 0) {$3$};
                \node (4) at (10, 0) {$\ldots$};

                \draw[->, thick, black] (0) to[in=90+25, out=90-25, looseness=4] node[midway, above] {$1$} (0);
                \draw[->, thick, black] (1) to[in=90+25, out=90-25, looseness=4] node[midway, above] {$\frac 12$} (1);
                \draw[->, thick, black] (1) to[bend left] node[midway, below] {$\frac 12$} (0);
                \draw[->, thick, black] (2) to[in=90+25, out=90-25, looseness=4] node[midway, above] {$\frac 12$} (2);
                \draw[->, thick, black] (2) to[bend left] node[midway, below] {$\frac 12$} (1);
                \draw[->, thick, black] (3) to[in=90+25, out=90-25, looseness=4] node[midway, above] {$\frac 12$} (3);
                \draw[->, thick, black] (3) to[bend left] node[midway, below] {$\frac 12$} (2);
                \draw[->, thick, black] (4) to[bend left] node[midway, below] {$\frac 12$} (3);

                \draw[->, thick, red] (0) to[in=90+55, out=90-55, looseness=12] node[midway, above] {$\frac 12$} (0);
                \draw[->, thick, red] (0) to node[midway, above] {$\frac 12$} (1);
                \draw[->, thick, red] (1) to[in=90+55, out=90-55, looseness=12] node[midway, above] {$\frac 12$} (1);
                \draw[->, thick, red] (1) to node[midway, above] {$\frac 12$} (2);
                \draw[->, thick, red] (2) to[in=90+55, out=90-55, looseness=12] node[midway, above] {$\frac 12$} (2);
                \draw[->, thick, red] (2) to node[midway, above] {$\frac 12$} (3);
                \draw[->, thick, red] (3) to[in=90+55, out=90-55, looseness=12] node[midway, above] {$\frac 12$} (3);
                \draw[->, thick, red] (3) to node[midway, above] {$\frac 12$} (4);

                \draw[->, decorate, decoration=snake] (0) to (0, -1.2) node[anchor=north] {$0$};
                \draw[->, decorate, decoration=snake] (1) to (2.5, -1.2) node[anchor=north] {$\frac 34$};
                \draw[->, decorate, decoration=snake] (2) to (5.0, -1.2) node[anchor=north] {$\frac 89$};
                \draw[->, decorate, decoration=snake] (3) to (7.5, -1.2) node[anchor=north] {$\frac {15}{16}$};
                \node[anchor=north] at (10.0, -1.2) {$\gain_\arm = 1$};
            \end{tikzpicture}
            \caption{
                \label{figure_nerlove_arrow}
                An instance of Nerlove-Arrow toy model. 
            }
        \end{figure}
        
        In human resources or advertisement, skill or goodwill is earned by investing effort into it. 
        When a high quantity is accumulated, more reward is earned. 
        But when abandoned, those fade away. 
        This idea can be modeled as a restless bandit as such: 
        Arms are subjects, states $\states_\arm \equiv \NN$ model the amount of skills, and rewards model how well the target task is performed.
        For each arm $\arm \in \arms$, if $\Arm(t)=\arm$ then the state of this arm tends to increase in expectation, otherwise it tends to decrease.
        We assume $\activationset = [0,1]$, so that every step $t \in \NN$ is a decision epoch. 
        Formally, the evolution of skills satisfies
        \begin{equation*}
            \EE \brackets*{
                \State_\arm (t+1)
                \middle|
                \State_\arm (t),
                \Arm(t) = \arm
            }
            \ge \State_\arm (t)
            \qquad \text{and} \qquad
            \EE \brackets*{
                \State_\arm (t+1)
                \middle|
                \State_\arm (t),
                \Arm(t) \ne \arm
            }
            \le \State_\arm (t)
            \; .
        \end{equation*} 
        The asymptotic skill threshold is $\mu_\arm > 0$ and represents how well the subject $\arm$ performs when fully trained. 
        From state $\state \in \states_\arm$, the subject does not perform at its peak, so that mean reward is of the form
        \[
            \reward_\arm(\state) 
            = 
            \mu_\arm - f_\arm (\state) 
        \]
        where $f_\arm : \states_\arm \to \RR_+$ is non-negative. 
        So that subject $\arm \in \arms$ approaches its peak in finite time, we assume that from any initial point $\state \in \states_\arm$,
        \[
            \EE_{\state}^{\policy_\arm} \brackets*{
                \sum_{t=1}^\infty \parens*{
                    \mu_\arm - \Reward(t+1)
                }
            }
            =
            \parens*{
                \sum_{t=1}^\infty
                \parens{\kernel^1_\arm}^t
                f_\arm
            } (\state)
            < \infty
            \; ,
        \]
        i.e., that the chain has finite bias. 
        The induced restless bandit is \emph{self-degrading}, as shown in \Cref{appendix:NA_toy_model}.
        An example of our Nerlove-Arrow toy model is displayed in \Cref{figure_nerlove_arrow}.
    \end{example}
    
    \ifSubfilesClassLoaded{
        
        \bibliographystyle{plainnat}
        \bibliography{biblio}
    }{}
    
\end{document}

\section{General regret bounds}
\label{section_learning_partial_information}

As a first sketch of the landscape of the theory, we attack the problem from an \emph{instance-dependent} viewpoint. 
That is, given an instance $\model \in \models$ and a learning algorithm $\learner$, we aim at bounding $\regret(T; \model, \learner)$ as accurately as possible when $T \to \infty$.
This is in opposition to \emph{worst-case}, or \emph{minimax} approaches, for which we would bound $\sup_{\model \in \models} \regret(T; \model, \learner)$.
The minimax approach is proved to be ill-behaved when $\models$ is too large, while non-trivial instance-dependent bounds can always be obtained.

Starting with the instance-dependent setting, we show in \Cref{subsec:lower_bound} that the expected regret \eqref{equation_regret_definition} of every rarely-switching algorithm grows super-logarithmically (\Cref{theorem_lowerbound_dependent_universal}). We further conjecture that the rarely-switching assumption can even be dropped on the space of self-degrading instances.
These lower bounds provide a baseline for learning algorithms.
In particular, regret guarantees of order $\OH(\log(T))$ are not achievable. 
In \Cref{subsec:optimistic}, we present our algorithm \texttt{UCB-NOM}, a variant of \texttt{UCB2} (\cite{auer_using_2002}), that takes as input a diverging \emph{span-proxy} function $\spanproxy$ as well as a confidence function $\delta : \NN \to [0,1]$. 
In~\Cref{subsec:general_upperbound}, we show that the instance-dependent regret of \texttt{UCB-NOM} is $\OH(\spanproxy(T)^2 \log(T))$, which implies that it can approximate the idealistic logarithmic regret regime arbitrarily close.
We conclude in \Cref{subsec:worst_case_vacuous} by showing that minimax regret bounds are unattainable for general $\models$.

\subsection{A super-logarithmic regret lower bound}
\label{subsec:lower_bound}

To establish \emph{instance-dependent} lower bounds, we must assume that the learning algorithm of interest is \emph{consistent}, that is, that the algorithm learns the optimal policy in every possible instance in $\models$. 
Formally, a learning algorithm $\learner$ is \emph{consistent} on $\models$ if, for all $\model \in \models$ and $\epsilon > 0$, it achieves regret 
$
    \regret(T; \model, \learner) 
    = \oh(T^\epsilon), 
$
when $T \to \infty$.
This definition is fairly standard in the multi-armed bandit literature \cite[§16]{lattimore_bandit_2020} and is necessary to derive the famous lower bound of \cite{lai1985asymptotically}.
Below, we show that consistency is enough to derive a super-logarithmic lower bound, provided that the algorithm is $\switchingfunction$-rarely switching in addition. 

\begin{theorem}[Lower bound, rarely switching algorithms]
\label{theorem_lowerbound_dependent_universal}
    Let $\learner$ be a learning algorithm that 
    (1) is consistent on $\models$, and
    (2) is $\switchingfunction$-rarely switching.
    For every instance $\model \in \models$ with $\gain_* < 1$ and unique optimal pure policy, it holds that
    \begin{equation}
    \label{equation_lower_bound}
        \forall \state_0 \in \states,
        \quad
        \mathrm{Reg}(T; \model, \learner)
        =
        \omega(\log(T))
        .
    \end{equation}
\end{theorem}

In particular, \Cref{equation_lower_bound} implies that no algorithm achieves $\OH(\log(T))$ regret, the gold standard for a large variety of RL problems, ranging from multi-armed bandits (\cite{lai1985asymptotically,auer_using_2002}) to Markov decision processes (\cite{auer_logarithmic_2006}).

We further believe that the rarely switching property of $\learner$ can be changed to $\model \in \models_\mathrm{sd}$, i.e., that the regret scales super-logarithmically on the class of self-degrading instances, without requiring the rarely switching property.
We note that super-logarithmic lower bounds have also been found in other settings, including  multi-armed bandits problems for which the eligible set of reward distributions is too large, see \cite[Exercise 16.4]{lattimore_bandit_2020}.

\begin{proof}[Proof sketch of \Cref{theorem_lowerbound_dependent_universal} (see complete proof in \Cref{appendix_lowerbound_instance_dependent}.)] 
    The proof involves the construction of an alternative model $\model_\arm^\epsilon$ that  (\emph{1}) is nearly indistinguishable from the original $\model_\arm$ from observations only, and (\emph{2}) has a different optimal arm. 
    The alternative model $\model_\arm^\epsilon$ is constructed by adding an $\epsilon$-transition to an absorbing state $\state_\infty$ that yields maximal reward $\Reward = 1$, see \Cref{figure_transformation}.
    We construct $\model^\epsilon$ from $\model$ by picking a suboptimal arm $\arm$ and changing $\model_\arm$ to $\model_\arm^\epsilon$, hence making $\arm$ the best arm for $\model^\epsilon$.
    
    Due to consistency, the algorithm learns the optimal policy in both $\model$ and in $\model^\epsilon$---yet the two instances do not have the same optimal policy. 
    Hence, to play optimal in $\model$, the algorithm must reject the plausibility that the true underlying environment is $\model^\epsilon$.
    As $\model$ and $\model^\epsilon$ only differ at arm $\arm$, the algorithm can only reject $\model^\epsilon$ by activating arm $\arm$. 
    As a result, this gives a lower bound on the number of pulls of~$\arm$. 

    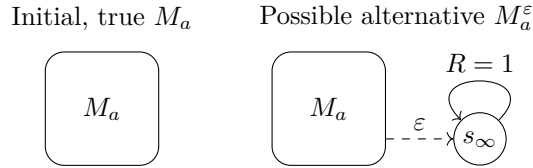
\begin{figure}[ht]
        \centering
        \begin{tikzpicture}
            \begin{scope}[shift={(0, 0)}]
                \node at (0, 1.2) {Initial, true $\model_\arm$}; 
                \node[draw, rectangle, minimum width=1.5cm, minimum height=1.5cm, rounded corners=0.25cm] (Ma) at (0, 0) {$\model_\arm$};
            \end{scope}
            \begin{scope}[shift={(3, 0)}]
                \node at (0.9, 1.2) {Possible alternative $\model_\arm^\epsilon$}; 
                \node[draw, rectangle, minimum width=1.5cm, minimum height=1.5cm, rounded corners=0.25cm] (Ma) at (0, 0) {$\model_\arm$};
                \node[draw, circle, inner sep=0.2em] (soo) at (2, -.4) {$\state_\infty$};
                \draw[->, dashed] (Ma.east) ++ (0,-0.4) to node[midway, above] {$\epsilon$} (soo);
                \draw[->, loop, looseness=5] (soo) to node[midway, above] {$R = 1$} (soo);
            \end{scope}
        \end{tikzpicture}
        \caption{
        \label{figure_transformation}
            An illustration of the transformation $\model_\arm^\epsilon$.
        }
    \end{figure}

    More precisely, we show that $\EE^{\model,\learner} [\visits_\arm(T)] \ge \frac 1\epsilon \log(T)$.
    Since $\epsilon > 0$ is arbitrary, we deduce that $\EE^{\model, \learner} [\visits_{\arm}(T)] = \omega(\log(T))$, which directly lower bounds the expected pseudo-regret.
    \Cref{equation_lower_bound} now follows by linking the pseudo-regret to the regret using the $\switchingfunction$-rarely switching property (\Cref{proposition_regret_and_pseudo_regret}).
\end{proof}

\subsection{An indexed optimistic algorithm: \texttt{UCB-NOM}}
\label{subsec:optimistic}
In this \Cref{subsec:optimistic}, we present \algname{}   (Algorithm~\ref{algo:UCB-POMB}).
\algname{} follows the \emph{optimism-in-the-face-of-uncertainty} (OFU) principle, and begins by constructing an \emph{optimistic index} $\tilde{\gain}_\arm (t)$ for every arm $\arm$.
To construct that index, the rationale goes as follows.
First, we provide a natural estimator $\hat{\gain}_\arm (t)$ for the value $\gain_\arm$, then bound by how much that estimator tends to deviate from $\gain_\arm$, and based on that bound, we finally construct a high-probability upper-bound $\tilde{\gain}_\arm (t)$ of $\gain_\arm$. 
Following the OFU principle, our algorithm simply activates the arm maximizing $\tilde{\gain}_\arm (t)$, and keeps it active until the number of visits $\visits_\arm(t)$ doubles. 

We now provide details on how the index of \algname{} is constructed. 

\begin{algorithm}[ht]
    \begin{algorithmic}[1]
        \REQUIRE{A span proxy function $\spanproxy$. \\  \hspace{2.75em} A confidence function $\delta: \NN \to [0,1]$.}
        \STATE $t \gets 1$;
        \FOR {episodes $k=1, 2, \ldots$}
            \STATE Set $t_k \gets t$;
            \STATE Compute indexes $\algindex_{\arm} (t_k)$ as 
            \begin{equation}
            \label{equation_optimistic_estimate}
                \algindex_\arm (t_k)
                :=
                \hat{\gain}_\arm (t_k)
                + \frac{
                    \spanproxy(t_k) \cdot \errorfunction_\arm (t_k, \delta)
                }{\visits_\arm(t_k)}
            \end{equation}
            \STATE Pick an arm $\Arm(t_k)\in \arg \max \tilde{\gain}_\arm (t_k)$ maximizing the index;
            \WHILE{$\visits_{\Arm(t_k)} (t) < 2 \visits_{\Arm(t_k)} (t_k)$ or $t$ isn't an activation time}
                \STATE Activate the arm $\Arm(t) := \Arm(t_k)$;
                \STATE Observe the new reward $\Reward(t+1)$;
                \STATE Increment $t \gets t + 1$;
            \ENDWHILE
        \ENDFOR
    \end{algorithmic}
    \caption{
    \label{algo:UCB-POMB}
        \texttt{UCB} for \texttt{N}on-\texttt{O}bservable \texttt{M}arkovian bandits (\algname).
    }
\end{algorithm}

\paragraph{The empirical value of an arm.}
To estimate the value of an arm (its gain $\gain_\arm$), we simply aggregate and take the average of all the rewards collected by pulling the said arm.
We call it the \emph{empirical value} of arm $\arm$, and denote it $\hat{\gain}_\arm (t)$.
Formally,
\begin{equation}
\label{equation_empirical_estimate}
    \hat{g}_\arm(t)
    :=
    \frac {1}{
        \visits_\arm (t)
    } 
    \sum_{i=1}^{t-1} 
    \indicator{\Arm(i)=\arm}\, \Reward(i+1), 
\end{equation}
with the convention $\hat{g}_\arm(t) = 1$ when $\visits_\arm(t) = 0$.
As dictated by the OFU principle, we are interested in constructing a \emph{high probability upper bound} of $\gain_\arm$ from $\hat{\gain}_\arm (t)$.
We do so by bounding $\abs{\hat{\gain}_\arm(t) - \gain_\arm}$ via a concentration result (\Cref{lemma_estimation_error}), that bounds that error as a function of the number of pulls to arm~$\arm$, the number of switches to arm~$\arm$, and the bias span.

\Cref{lemma_estimation_error} is at the heart of the construction of the index $\tilde{\gain}_\arm (t)$ of our algorithm as such, we provide its proof in the main body.

\begin{proposition}[Error of the empirical estimator]
\label[proposition]{lemma_estimation_error}
    For all arm $\arm \in \arms$ and confidence threshold $\delta > 0$, 
    \begin{equation*}
        \PP \parens[\big]{
            \exists t \ge 1,
            \visits_\arm (t) 
            \abs*{\hat{\gain}_\arm (t) - \gain_\arm}
            >
            D \cdot \errorfunction_\arm (t, \delta)
        }
        \le
        \delta,
    \end{equation*} 
    where $D \equiv 1 + \max_{\arm \in \arms} \vecspan(\bias_\arm)$ and the error function is given by
    \begin{equation}
    \label{equation_error_function}
        \errorfunction_\arm (t, \delta)
        :=
        \sqrt{
            \visits_\arm (t)
            \log \parens*{
                \frac{1 + \visits_\arm (t)}\delta
            }
        }
        + \sum_{i=1}^{t-1}
        \indicator{
            A(i-1) \ne a,
            A(i) = a
        }.
    \end{equation}
\end{proposition}

\begin{proof}
    Consider the sequence of stopping times given by $\tau_{1}^\arm := \inf \braces{t \ge 1 : \Arm(t) = \arm}$ and $\tau_{k+1}^\arm := \inf \braces{t > \tau_k^\arm : \Arm(t) = \arm}$, that enumerates the time-instants when the played arm is $\arm \in \arms$.
    First, introduce the noise terms
    \begin{equation*}
    \begin{aligned}
        \noise_\arm^\reward (t) 
        & := 
        \sum_{k=1}^{\visits_{\arm}(t)} 
        \parens[\Big]{
            \Reward(\tau_k^a+1) 
            - 
            \reward_\arm (\State(\tau_k^\arm))
        }
        \qquad\text{and}\qquad
        \noise_\arm^\kernel (t) 
        := 
        \sum_{k=1}^{\visits_{\arm}(t)} 
        \parens[\Big]{
            \bias_\arm (\State_\arm (\tau_k^\arm+1)) 
            - 
            \kernel_\arm^1 (\State_\arm(\tau_k^\arm)) \bias_\arm
        },
    \end{aligned}
    \end{equation*}
    that respectively measure how much the immediate rewards and transitions---both obtained by activating arm $\arm$---deviate from their means.
    Further, introduce their sum $\noise_\arm (t) := \noise_\arm^\reward (t) + \noise_\arm^\kernel (t)$.
    
    For $t \ge 1$, we have
    \begin{align}
    \nonumber
        \visits_\arm (t)
        \hat{\gain}_\arm (t)
        & =
        \sum_{i=1}^{t-1}
        \indicator{\Arm(i) = \arm}
        \Reward(i+1)
        \\
    \nonumber
        & =
        \sum_{k=1}^\infty
        \indicator{\tau_k^\arm < t}
        \Reward(\tau_k^\arm + 1)
        \\
    \nonumber
        & \overset{\eqnum{1}}=
        \sum_{k=1}^{\visits_\arm (t)}
        \Reward(\tau_k^\arm + 1)
        \\
    \nonumber
        & \overset{\eqnum{2}}=
        \sum_{k=1}^{\visits_\arm (t)}
        \parens*{
            \gain_\arm
            + \parens*{
                \unit_{\State_\arm(\tau_k^\arm)} 
                - \kernel_\arm^1 (\State_\arm(\tau_k^\arm))
            } \bias_\arm
        }
        +  \noise_\arm^\reward(t)
        \\
    \label{equation_estimation_error_1}
        & \overset{\eqnum{3}}=
        \visits_\arm (t) 
        \gain_\arm
        +
        \sum_{k=1}^{\visits_\arm (t)} \parens*{
            \unit_{\State_\arm(\tau_k^\arm)}
            - \unit_{\State_\arm(\tau_k^\arm+1)}
        } \bias_\arm
        + \noise_\arm (t)
    \end{align}
    where 
    \texteqnum{1} uses that $\visits_\arm (t) := \sup \braces{k : \tau_k^\arm < t}$; and
    \texteqnum{2} introduces the canonical basis $(\unit_\state)_{\state \in \states_\arm}$ of $\RR^{\states_\arm}$ and invokes the Poisson equation under $\policy^\arm$, given by $\gain_\arm + \bias_\arm (\state_\arm) = \reward(\state_\arm) + \kernel_\arm^1 (\state_\arm) \bias_\arm$\footnote{Here a compact vector notation is used.}.
    In \texteqnum{3}, we recognize two error terms.
    The first is linked to the number of switches, while the second is a martingale noise.

    First, the martingale noise $\noise_\arm (t)$ is bounded using a time-uniform Azuma-Hoeffding's inequality (\cite{bourel_tightening_2020}). 
    We get that, uniformly for $t \ge 1$ and with probability at least $1 - \delta$, 
    \begin{equation*}
        \abs{
            \noise_\arm^\reward (t)
            +
            \noise_\arm^\kernel (t)
        }
        \le
        \parens*{1 + \vecspan(\bias_\arm)}
        \sqrt{
            \visits_\arm (t) 
            \log \parens*{
                \tfrac{1 + \visits_\arm (t)}\delta
            }
        }.
    \end{equation*}
    For the first error term of \Cref{equation_estimation_error_1}, we have:
    \begin{align*}
        \Bigg|
            \sum_{k=1}^{\visits_\arm (t)}
            \Big( 
                \unit_{\State(\tau_k^\arm)}
                - \unit_{\State(\tau_k^\arm+1)}
            \Big) \bias_\arm
        \Bigg|
         & \le 
        \vecspan(\bias_\arm)
        \parens*{
            1
            + 
            \sum_{k=1}^{\visits_\arm (t)}
            \indicator{
                \State(\tau_k^\arm + 1) \ne \State(\tau_{k+1}^\arm)
            }
        }
        \\
        & \le 
        \vecspan(\bias_\arm)
        \parens*{
            1
            + 
            \sum_{k=1}^{\visits_\arm (t)}
            \indicator{
                \tau_k^\arm + 1 \ne \tau_{k+1}^\arm
            }
        }
        \\
        & \le
        \vecspan(\bias_\arm)
        \sum_{i=1}^{t-1}
        \indicator{\Arm(i-1) \ne \arm, \Arm(i) = \arm}
        .
    \end{align*}
    We conclude by plugging everything   back in \Cref{equation_estimation_error_1}.
\end{proof}

\paragraph{Building an optimistic estimator.}
\label{subsec:optimistic_estimator}
The bound on the estimation error established in \Cref{lemma_estimation_error} allows for the construction of an \emph{optimistic estimator} $\tilde{\gain}_\arm (t)$ of $\gain_\arm$.
By optimistic, we mean that $\algindex_\arm(t)$ serves as a high probability upper bound on the arm’s value $\gain_\arm$, that is, 
$
    \Pr \parens*{
        \exists t \ge 1,
        \algindex_\arm (t) < \gain_\arm
    }
    \le
    \delta.
$
A direct consequence of \Cref{lemma_estimation_error} is that 
\begin{equation}
\label{equation_idealistic_index}
    \Pr \parens*{
        \exists t \ge 1
        :
        \hat{\gain}_\arm (t)
        + \frac{D \cdot \errorfunction_\arm (t, \delta)}{\visits_\arm (t)}
        < 
        \gain_\arm
    }
    \le
    \delta.
\end{equation}
\Cref{equation_idealistic_index} suggests that $\hat{\gain}_\arm (t)
+ \frac{D \cdot \errorfunction_\arm (t, \delta)}{\visits_\arm (t)}$ is a natural choice for $\algindex_\arm (t)$.
Unfortunately, the quantity $D$ is unknown and varies depending on the underlying instance of Markovian bandit. 
To circumvent this difficulty, we consider a \emph{span proxy function} $\spanproxy : \NN \to \RR_+$ diverging to $+\infty$.
Eventually, $f(t)$ exceeds $D$, so that for every instance $\model \in \models$, there is a time threshold $\beta_\spanproxy \in \RR_+$, s.t., 
\begin{equation}
\label{equation_optimistic_index}
    \Pr \parens*{
        \exists t \ge \beta_\spanproxy
        :
        \hat{\gain}_\arm (t)
        + \frac{\spanproxy(t) \cdot \errorfunction_\arm (t, \delta)}{\visits_\arm (t)}
        < 
        \gain_\arm
    }
    \le
    \delta.
\end{equation}
That time threshold is $\beta_f := 1 + \sup\braces{t \ge 1: f(t) \le D}$.
The above justifies our choice of the index $\algindex_\arm(t)$ in  \Cref{equation_optimistic_estimate} of \Cref{algo:UCB-POMB}.

\paragraph{Taming the number of switches.}
\label{subsec:doubling_trick}
It is preferable for the optimistic estimator to converge to the empirical estimator when $\visits_\arm (t)$ clearly overshoots, or said differently, that the confidence region for $\gain_\arm$ shrinks to a singleton when $\arm$ is the most visited arm for instance.
To make sure of it, we need the associated error $\errorfunction_\arm (t, \delta) / \visits_\arm (t)$ to vanish asymptotically.
In \Cref{equation_error_function}, we see that the error function $\errorfunction_\arm (t,\delta)$ increases with the number of switches $\sum_{i=1}^{t-1} \indicator{\Arm(i-1) \ne \arm, \Arm(i) = \arm}$.
So, $\errorfunction_\arm (t, \delta) / \visits_\arm (t)$ can only be vanishing if the number of switches grows sublinearly in $\visits_\arm (t)$; In other words, when the algorithm satisfies a rarely switching property (\Cref{definition_rarely_switching_arms}).

A simple method to obtain a rarely switching algorithm is to rely on the famous \emph{doubling trick} (see \cite{auer_gambling_1995} and \cite[\texttt{UCB2}]{auer_using_2002}): Upon activating arm~$\arm$, activate it until the visit counts of $\visits_\arm (t)$ doubles. 
Doing so, the number of switches to arm $\arm$ becomes of order $\log (1 + \visits_{\arm}(t))$, which makes the algorithm $\OH(\log)$-rarely switching.

Combining the above, we recover \algname{} (\Cref{algo:UCB-POMB}).

\paragraph{Comparison with \texttt{RCA}.}
Our algorithm \algname{} is closely related to \texttt{RCA} of \cite{tekin2012online}, as both are generalizations of \texttt{UCB} \citep{auer_using_2002} to Markovian bandits.
Two major differences with \texttt{RCA} are worth mentioning.

\begin{itemize}
    \item \emph{Observability of states and regeneration.}
        \texttt{RCA} stands for \emph{Regenerative Cycle Algorithm}, and the algorithm of \cite{tekin2012online} observes the stream of states to calibrate its arm estimates better. 
        In particular, \texttt{RCA} always make sure to switch off from arm $\arm$ while arm $\arm$ is in a reference state $\state_\arm \in \states_\arm$.
        Moreover, when it switches onto arm $\arm$, it ignores the first pool of rewards until $\state_\arm$ is reached. 
        This makes the empirical estimator of $\gain_\arm$ --- the $\hat{\gain}_\arm (t)$ --- of \cite{tekin2012online} different from our own.
        Their estimator is unbiased---hence better---because it does not require the correction term $\sum_{i=1}^{t-1} \indicator{\Arm(i-1) \ne \arm, \Arm(i) = \arm}$ due to switching arms in the optimistic bonus.
        However, because states are non-observable in our setting, the regeneration technique is unachievable. 
        
    \item \emph{Optimistic bonus.}
        In the end, the optimistic bonus is a correction term.
        It is about how much the aggregate rewards of a Markov reward process tends to deviate from its mean, i.e., is about isolating the trajectory of rewards $(\Reward_i^\arm)_{i\ge1}$ generated by arm $\arm \in \arms$ and comparing $\sum_{i=1}^t \Reward_i^\arm$ to $t \cdot \gain_\arm$.
        For both, the optimistic bonus is of the form $\sqrt{L / \visits_\arm (t)}$, where
        \begin{equation*}
            L 
            \asymp
            \begin{cases}
                \frac{\abs{\states_\arm}^2 \hat{\pi}_\mathrm{max}^2}{\epsilon_\mathrm{min}}
                & \text{for \cite{tekin2012online};}
                \\
                1 + \vecspan(\bias_\arm)^2
                & \text{for us,}
            \end{cases}
        \end{equation*}
        where $\hat{\pi}_\mathrm{max} \le 1$ is related to the stationary distribution of the chain $\model_\arm$ and $\epsilon_\mathrm{\min} > 0$ is the spectral gap, see \citep{tekin2012online}.
        The presence of the spectral gap in the bound of \cite{tekin2012online} naturally emerges from their analysis: to compare $\sum_{i=1}^t \Reward_i^\arm$ and $t \cdot \gain_\arm$, the authors focus on the speed of convergence of the empirical measure of presence of $(\State^\arm_i)_{i \ge 1}$, relying on standard Markov chain techniques. 
        In general, the spectral gap is related to the mixing time as $t_\mathrm{mix} = \smash{\widetilde{O}(\abs{\states} / \epsilon_\mathrm{min})}$,\footnote{$\widetilde{O}(-)$ hides logarithmic factors.} see \citep{jerison_general_2013}, while the bias span is related to the mixing time with $\vecspan(\bias) = \OH(t_\mathrm{mix})$, see \citep{wang_near_2022}.
        A dependency in $\vecspan(\bias)$ is better than one in $(\abs{\states}, \epsilon_\mathrm{min})$: we have no dependency in the number of states and $\vecspan(\bias)$ ties the transition structure together with the reward mechanism.
        However, our $\vecspan(\bias)^2$ is a bit loose --- it could be improved using variance reduction techniques \cite[Chapter~5]{boone_thesis_2024}, but proof details would then be more delicate.
\end{itemize}

\subsection{An asymptotic regret bound for \texttt{UCB-NOM} matching the general regret lower bound} 
\label{subsec:general_upperbound}

We have shown that the regret of every algorithm must be $\omega(\log(T))$ and we have described an algorithm (\Cref{algo:UCB-POMB}), parameterized by a bias proxy $\spanproxy$. 
The next theorem demonstrates that the idealistic $\log(T)$ regime is approached arbitrarily closely by \algname{}.

\begin{theorem}[Upper bound]
\label{theorem_upperbound_f_diverging}
    Let $\spanproxy : \RR_+ \to \RR_+$ an increasing function diverging to infinity and a confidence function $\delta(t) : \NN \to \RR_+^*$ such that $\sum_{t \ge 1}\delta(t) < \infty$ and $\log (\frac 1{\delta(t)}) = \Theta(\log(t))$.\footnote{
        For instance, any $\delta(t)$ of the form $\delta(t) = (\frac 1t)^{1 + \epsilon}$ with $\epsilon > 0$ works.
    } 
    For all instance $\model \in \models$, 
    \begin{equation}
    \label{equation_upperbound_f_diverging}
        \regret(T; \model, \algname(\spanproxy)) 
        = 
        \OH \parens*{
            \sum_{\arm \notin \arms^*}
            \frac{\spanproxy(T)^2 \log(T)}{\gain_* - \gain_\arm}
            + 
            \sum_{\arm \notin \arms^*}
            \parens*{\gain_* - \gain_\arm} 
            (\beta_f)^2
        },
    \end{equation}
    where $\arms^* := \arg\max_{\arm \in \arms} \gain_\arm$ is the set of optimal arms and $\beta_f := 1 + \sup \braces*{t \ge 1: f(t) \le D}$ is a burn-in cost (see \Cref{equation_optimistic_index}), with $D := \max \braces{\max_\arm \vecspan(\bias_\arm) + 1, C_\tau}$ a threshold that depends on $\model$.
\end{theorem}

The proof, together with error terms of higher orders, can be found in \Cref{appendix:upper_bound_instance_dependent}.
Note that the dominant term is of order $\spanproxy(T)^2 \log(T)$ and that notably, the delay constant $\strongdelayconstant$ is absent from it.  

In particular \Cref{theorem_upperbound_f_diverging} shows that the lower bound of \Cref{theorem_lowerbound_dependent_universal} is \emph{tight}, and that \algname{} achieves it.
Indeed, the definition of $\omega(\log(T))$ provided by the lower bound reads as follows: for a given instance~$\model$ and an algorithm~$\learner$, there exists a function $\spanproxy_0 \to \infty$ such that $\liminf_{T \to \infty} \regret(T; \model, \learner)/(\spanproxy_0 (T) \log(T)) > 0$.
For \algname{} with span proxy $\spanproxy_1 \equiv (\spanproxy_0)^{\frac 13}$, it holds that,
\begin{equation*}
    \limsup_{T \to \infty}
    \frac{
        \regret(T; \model, \algname(\spanproxy_1))
    }{
        \regret(T; \model, \learner)
    }
    = 0.
\end{equation*}
Hence, our $\algname(\spanproxy_1)$ has better regret than $\learner$ on $\model$.
It can even be further improved by running it with span proxy $\spanproxy_2 \equiv (\spanproxy_1)^{\frac 13}$, or $\spanproxy_3 \equiv (\spanproxy_2)^{\frac 13}$, and so and so forth.
This means that whatever the algorithm is, it can be forever improved, getting closer and closer to the idealistic (but non-reachable) $\log(T)$ regret regime.

\subsection{Worst case regret: what is the best way to choose the span proxy \texorpdfstring{$\spanproxy$}{f}?}
\label{subsec:worst_case_vacuous}

From an asymptotic viewpoint, the regret is simply $\OH(\spanproxy(T)^2 \log(T))$.
The best choice for $\spanproxy$ is thus a function diverging $+\infty$ as slowly as possible.
Unfortunately, there is no such function, because $\sqrt{\spanproxy}$ always diverges slower than $\spanproxy$. 
However, this observation only tells half of the story. 

As it turns out, due to the reliance of \algname{} on an optimistic estimator $\tilde{\gain}_\arm (t)$ to estimate the value $\gain_\arm$ of arm $\arm$, the asymptotic regret guarantee of \Cref{theorem_upperbound_f_diverging} hides a burn-in cost $\beta_f:=1 + \sup\braces{t \ge 1: f(t) \le D}$  (see \Cref{equation_optimistic_index}) that is eventually negligible, but dominant nonetheless in the early stages of the learning process. 
The regret upper-bound of \Cref{equation_upperbound_f_diverging} should therefore be read as such: The first term (in $\spanproxy(T)^2 \log(T)/(\gain_* - \gain_\arm)$) holds the asymptotic behavior regret, while the second term (in $\beta_f$) represents the regret that may be accumulated before the optimistic indices become valid---a \emph{burn-in cost}.

From a non-asymptotic viewpoint, no proper tradeoff between the asymptotic $\spanproxy(T)^2 \log(T)/(\gain_* - \gain_\arm)$ and the burn-in cost $\beta_f$ can be found. 
\Cref{theorem_linear_worst_case} states that independently of the choice of span proxy $f$, for every horizon $T \ge 1$, there will be an instance where the regret scales linearly.

\begin{theorem}[Worst case lower bound]
\label{theorem_linear_worst_case}
    Let $\learner$ be a learning algorithm.
    Let $\activationset \subseteq [0, 1]$ be arbitrary.
    For $T \ge 1$, we have
    \begin{equation}
        \sup_{\model \in \models}
        \regret(T; \model, \learner)
        \ge 
        \frac {11}{4800} \cdot T.
    \end{equation}
\end{theorem}

The proof can be found in \Cref{appendix_lowerbound_worst_case}.

\Cref{theorem_linear_worst_case} shows that our class of models is so broad that \emph{every} algorithm has linear regret in the worst case (in particular \algname{}, whatever the choice of $f$).
Hence, no single span proxy $f$ could be uniformly optimal. 
Different choices of $f$ will perform better on some instances and worse on others.

\section{Regret bounds with uniformly bounded bias span}
\label{section_bounded_bias}

In \Cref{section_learning_partial_information}, we have shown that for the space of all Markovian bandit problems, the regret of consistent algorithms must scale as $\omega(\log(T))$ and that this bound is achieved by \algname{} (\Cref{algo:UCB-POMB}).
To achieve this result, we have argued how important it is to make the span proxy $\spanproxy$ diverge to $+\infty$, so that $\spanproxy$ eventually upper-bounds the unknown worst bias span $\max_{\arm \in \arms} \vecspan(\bias_\arm)$.

What if a bound on the worst bias span is known?

In this section, we show that given prior knowledge on $\max_{\arm \in \arms} \vecspan(\bias_\arm)$, the span proxy function does not need to diverge to $+\infty$ anymore and the logarithmic regret regime can even be achieved.
In particular, by running \algname{} with span proxy $\spanproxy \equiv \max_{\arm \in \arms} \vecspan(\bias_\arm)$, we obtain instance dependent regret $\OH(\log(T))$ (\Cref{section_instance_dependent_bound}) and worst-case regret $\OH(\sqrt{T \log(T)})$ (\Cref{section_worst_case_bound}).

For $D \in \RR_+$, we introduce 
\begin{equation*}
    \models_{D}
    :=
    \braces*{
        \model \in \models
        :
        \max \braces*{
            \max_{\arm \in \arms}
            \vecspan(\bias_\arm)+1
            ,
            C_\tau
        } 
        \le 
        D
    },
\end{equation*}
the collection of instances with bias span and decision epoch delays $\delayconstant$ (\Cref{assumption_strong_decision_epochs}) both bounded by $D$. The second assumption ``$\delayconstant \le D$'' is not strictly necessary, but will significantly simplify the theory. 

\subsection{Learning with a bound on the span proxy}
\label{section_instance_dependent_bound}

Under the assumption that $\model \in \models_D$, the idealistic logarithmic regret regime is achieved by \algname, when equipped with the span proxy function $\spanproxy \equiv D$, see \Cref{theorem_upperbound_log} below.

\begin{theorem}[Instance-dependent regret]
\label{theorem_upperbound_log}
    Let $D \in \RR_+$ and $\model \in \models_{D}$ with a \emph{unique} optimal arm $\arm^*$.
    Let $\delta(t) : \NN \to \RR_+^*$ be a confidence function such that $\sum_{t \ge 1}\delta(t) < \infty$ and $\log (\frac1{\delta(t)}) = \Theta(\log(t))$.
    When run with span proxy $\spanproxy \equiv D = 2 \max_{\arm \in \arms}\max\braces*{1 + \vecspan(\bias_\arm), 3 \vecspan(\bias_\arm)}$, the regret of \algname{} satisfies
    \begin{equation*}
        \regret(T; \model, \algname(D)) = 
        \OH \parens*{
            \sum_{\arm \ne \arm^*} 
            \frac{D^2 \log(T)}{\gain_* - \gain_\arm}
        }.
    \end{equation*}
\end{theorem}
\begin{proof}
    The proof essentially reduces to the analysis performed in the proof of \Cref{theorem_upperbound_f_diverging}. 
    Details can be found in the \Cref{appendix_instance_dependent_bounded}. 
\end{proof}

The regret remains logarithmic even if the optimal arm is not unique, but the regret upper bound is slightly degraded. 
See \Cref{appendix_instance_dependent_bounded} for more details. 
A discussion on the relation between our results and similar ones from the stochastic bandits literature can be found in \Cref{section_comparison_stochastic_bandits}.

\subsection{Worst-case regret guarantees}
\label{section_worst_case_bound}

Furthermore, with the assumption $\model \in \models_D$, we can move beyond the instance-dependent logarithmic regret previously presented in \Cref{theorem_upperbound_log} and establish robust guarantees that hold for every instance uniformly.

\begin{theorem}[Worst-case regret]
    \label{theorem_worst_case_sqrt}
    Let $D \in \RR_+$ and $\delta(t) : \NN \to \RR_+^*$ be a confidence function such that $\sum_{t\ge1} t \cdot \delta(t) < \infty$ and $\log (\frac 1{\delta(t)}) = \Theta(\log(t))$.\footnote{
        For instance, any $\delta(t) = (\frac 1t)^{2 + \epsilon}$ with $\epsilon > 0$ is fine.
    }
    When running with span proxy $\spanproxy \equiv D$, $ \algname(D)$ has the following worst-case regret:
    \begin{equation*}
    \begin{gathered}
        \sup_{ \model \in \models_{D} }
        \regret(T; \model, \algname(D))
        = \OH \parens*{
            D
            \sqrt{\abs{\arms} T \log(T)}
        }.
    \end{gathered}
    \end{equation*}
\end{theorem}
\begin{proof}
    The proof is deferred to \Cref{appendix_worst_case_upper}. 
\end{proof}

The tightness of this result is discussed in more detail in \Cref{appendix_worst_case_upper}.

\section{Tuning \texorpdfstring{$\delta(t)$}{delta(t)}: Being under-optimistic is preferable in Markovian bandits}
\label{sec:confidence_tuning}

\Cref{section_bounded_bias} provides the natural choice $\spanproxy \equiv D $ for the span proxy---when bounds on maximal span of the bias function and the delays of the decision epochs are available---and removes the need for a diverging span proxy. 
In this last section, we are interested in what the best choice for the confidence threshold $\delta(t)$ is, given that the span-proxy is fixed to the value recommended by \Cref{section_bounded_bias}.
Indeed, while from an asymptotic viewpoint the regret is simply $\smash{\OH((\frac D{\Delta})^2 \log(T))}$ for all summable confidence function, in finite horizon, a trade-off appears.
\begin{itemize}
    \item 
        Be $\delta(t)$ too large, and the confidence intervals are too narrow to contain the true gains, leading to a failure of the optimistic property.
        As the optimistic bonus $\smash{\log(\frac 1{\delta(t)})}$ vanishes, the algorithm's behavior converges toward a greedy algorithm, that pulls the arm with the highest empirical estimator. 
        Then, \texttt{UCB-NOM} is very likely to pull bad up to $\approx T$ times, resulting in disastrous performance. 
    \item 
        Be $\delta(t)$ too small, and the algorithm is too conservative: The optimistic bonuses dictate everything that the algorithm does. 
        As the optimistic bonus $\smash{\log(\frac 1{\delta(t)})}$ explodes to $+\infty$, the algorithm's behavior converges towards a round-robin process. 
\end{itemize}
In \Cref{fig:confidence}, we display the regret of \texttt{UCB-NOM}$(D)$ for a reasonably short horizon, as a function of the confidence parameter and for different parameterizations of $\delta(t)$.
To avoid running into deceptive conclusions, the environment is specifically designed so that the greedy algorithm (\texttt{UCB-NOM} with $\delta = 1$) has no chance to work. 
But then, even though the environment is unfavorable to greedy algorithms, these experiments suggest that very high confidence parameters are preferred, such as $\smash{\delta(t) = \frac{\alpha(t)}{t}}$ or $\smash{\frac{\beta(t)}{t \log(t)}}$ with very high $\alpha(t)$ or $\beta(t)$.
In the remaining of this section, we explain why. 

\begin{figure}[ht]
    \centering
    \begin{subfigure}{0.48 \linewidth}
        \centering
        \begin{tikzpicture}
            \node at (0, 0) {\includegraphics[width=.97\linewidth]{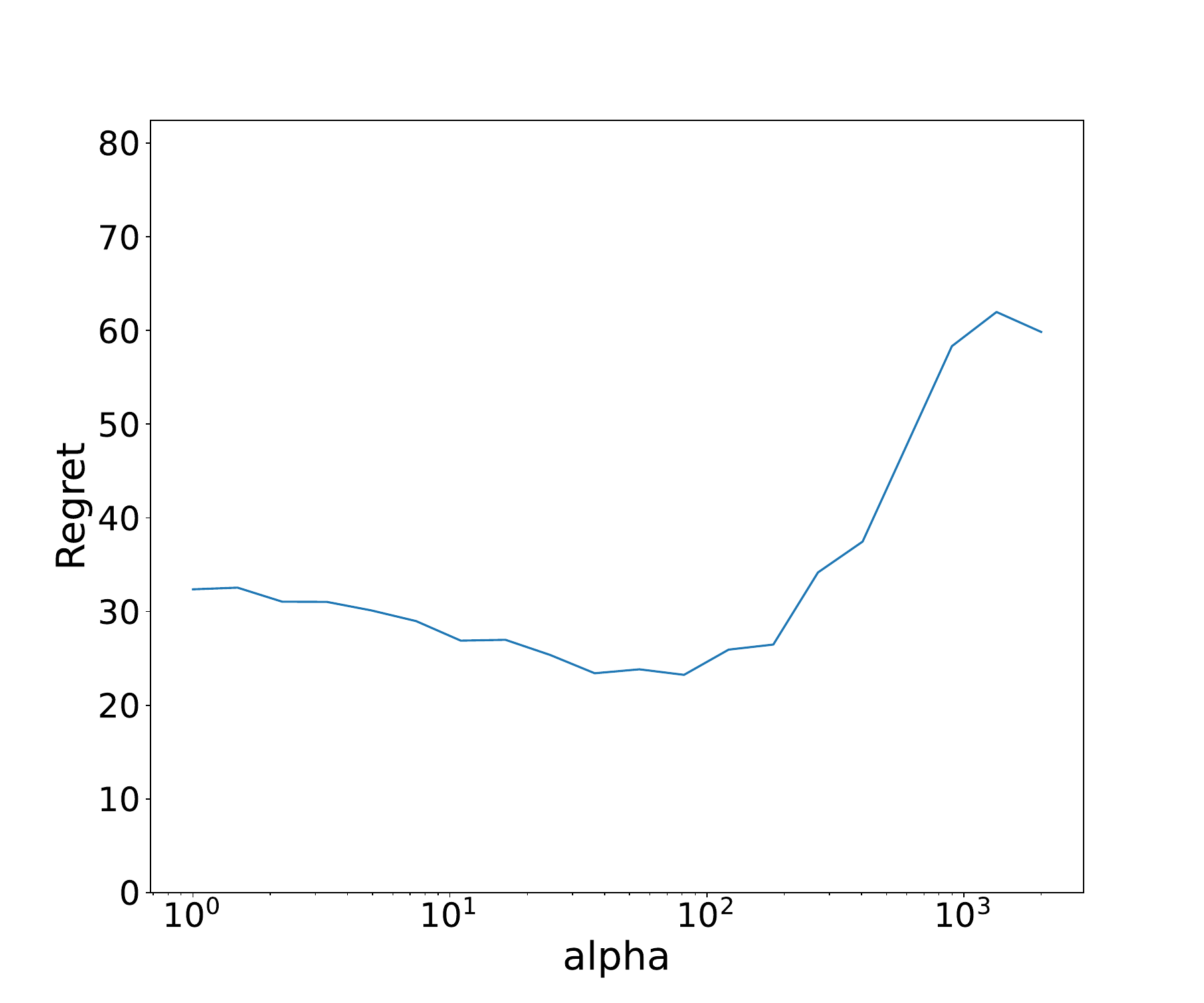}};
            \node[fill=white] at (.1, -2.9) {~~$\alpha$~~};
            \node[fill=white, rotate=90] at (-3.47, 0) {Regret};
        \end{tikzpicture}
        \caption{
            Regret of $\texttt{UCB-NOM}(D)$ with horizon $T=2000$ with confidence parameter $\delta(t) := \frac{\alpha}{t}$. 
            Each point is obtained by averaging over 1000 trajectories.
        }
        \label{fig:confidence_alpha}
    \end{subfigure}
    \hfill 
    \begin{subfigure}{0.48 \linewidth}
        \centering
        \begin{tikzpicture}
            \node at (0, 0) {\includegraphics[width=.97\linewidth]{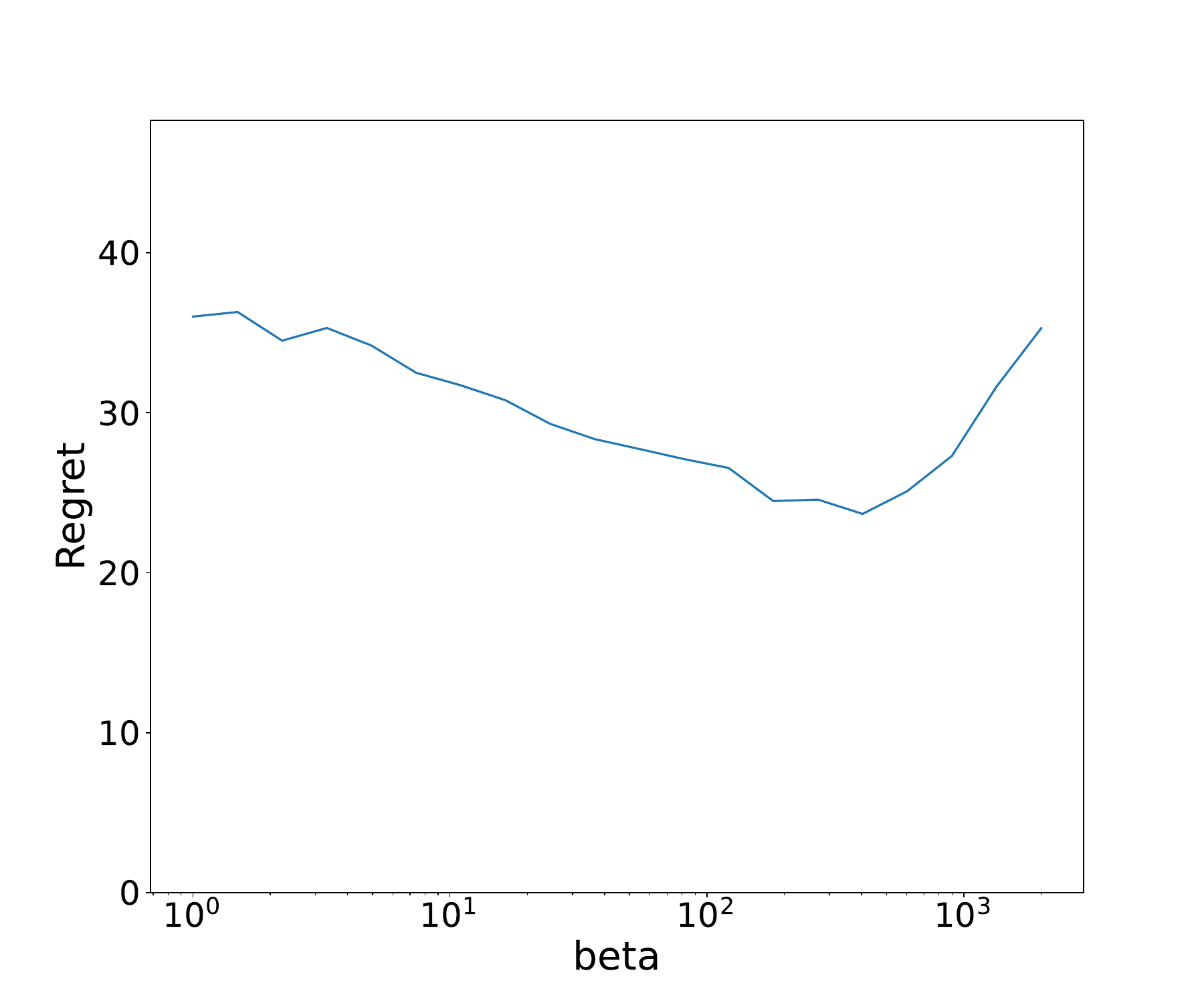}};
            \node[fill=white] at (.1, -2.9) {~~$\beta$~~};
            \node[fill=white, rotate=90] at (-3.48, 0) {Regret};
        \end{tikzpicture}
        \caption{
            Regret of $\texttt{UCB-NOM}(D)$ with horizon $T=2000$ with confidence parameter $\delta(t) := \frac{\beta}{t \log t}$. 
            Each point is obtained by averaging over 1000 trajectories.
        }
        \label{fig:confidence_beta}
    \end{subfigure}
    \caption{
    \label[figure]{fig:confidence}
        Regret of $\texttt{UCB-NOM}(D)$ on a Markovian bandits with two arms. 
        The first arm $\arm = 1$ is a stochastic bandit with reward $\rewardDistribution_1 = \mathrm{Bernoulli}(0.3)$.
        The second arm $\arm = 2$ is a two states Markovian arm whose transition are, when active, uniform over $\braces*{0,1}$, and when passive, deterministically directed to 0. Associated rewards are $\rewardDistribution_2 (0) = 0$ and $\rewardDistribution_2 (1) = \mathrm{Bernoulli}(0.8)$.
    }
\end{figure}
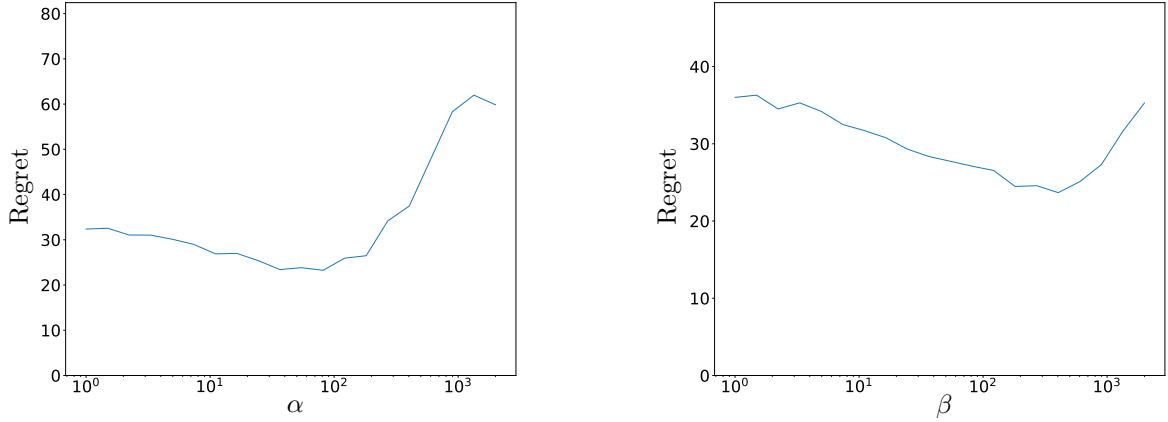

In the bonus, the term due to switching cost cannot be improved in general.
The environment in \Cref{fig:confidence} is chosen so that this term is chosen in a tight fashion: every time the arm is dropped, the arm goes back to the worst state instantly. 
For this environment, the switching bonus perfectly corrects the empirical estimator and we can fully focus on the study of the optimistic bonus. 
So, switching terms are casted aside by looking at bounds on the pseudo-regret directly.

From the proof of \Cref{theorem_upperbound_log}, keeping dominant terms, we get the upper-bound of the pseudo-regret is of the form:
\begin{equation}
\label{equation_regret_bound_approx}
    \EE \brackets*{\pseudoregret(T)}
    \lesssim
    C_1 \norm{\Delta}_1 \int_1^\infty \delta(t) ~ \dd t 
    \parens*{
        \sum_{\arm \notin \arms^*} \frac1{\Delta_\arm}   
    }
    \log \parens*{
        \frac{1}{\delta(T)}
    }
    \; .
\end{equation}
where $C_1, C_2 > 0$ are universal constants.\footnote{
    It is not easy to find a real meaning to the precise values of these constants $C_1, C_2$, because they are the results of computations that are pretty conservative. 
    Therefore, the discussion revolves around the dependency in $\Delta$ and $D$ exclusively. 
}

\subsection{Confidence functions of the form \texorpdfstring{$\delta(t) = \frac{\alpha}{t}$}{delta(t) = alpha/t}}

The bound \eqref{equation_regret_bound_approx} encourages to choose $\delta(t) \asymp \frac 1t$, so that the two terms are of the same order---we recover confidence parameters that are polynomial in $\frac 1t$, as suggested by standard stochastic bandits theory \citep{lattimore_bandit_2020}.
Therefore, as a first choice for the confidence function, we take $\delta(t) := \frac{\alpha}{t}$ for some $\alpha \in \RR$. 

Rewriting the previous upper-bound and trashing a few negligible terms, we obtain
\begin{equation*}
    \EE [\pseudoregret(T)]
    \lesssim
    \regretfunction(\alpha) 
    :=  
    C_1 \norm{\Delta}_1 \alpha \log(T)
    +
    C_2 D^2 
    \parens*{
        \sum_{\arm \notin \arms^*} \frac1{\Delta_\arm}   
    }
    \log \parens*{
        \frac{T}{\alpha}
    }
    \; .
\end{equation*}
Now, computing derivative of $\regretfunction$ with respect to $\alpha$ gives the optimal value:
\begin{equation}
\label{equation_optimal_alpha}
    \alpha^* =
    \frac{
        C_2 D^2 
        \sum_{\arm \notin \arms^*} \frac1{\Delta_\arm}   
    }{
        C_1 \norm{\Delta}_1 \log(T) 
    }
    \; .
\end{equation}
We should highlight the presence of the $\log(T)$ term which emphasis that our class of function for $\delta$ is not fully satisfactory.
However, this remains a suitable choice as long as the horizon is not too large, i.e., for horizon that are sub-exponential, typically $\smash{T \ll \exp \braces{D^2 \sum_{\arm} \frac 1{\Delta_\arm} / \norm{\Delta}_1}}$.
For such horizons and in the \emph{two-armed case}, we get an optimal optimism of order $(D / \Delta)^2$ up to numerical values: this is a quartic in the problem parameters, and therefore is large very easily. 
For the environments on which the experiments of \Cref{fig:confidence}, we have $(D / \Delta)^2 = 400$ already.

\subsection{Confidence functions of the form \texorpdfstring{$\delta(t) = \frac{\beta}{t \log(t)}$}{delta(t) = beta/(t log(t))}}

The logarithmic term in the denominator of $\alpha^*$ in \Cref{equation_optimal_alpha} is a bit unsatisfactory. 
To correct it, a simple approach is to take a confidence function of the form $\delta(t) := \frac{\beta}{t \log(t)}$ for $\beta \in \RR$. 
Repeating the same calculation, we find a pseudo-regret upper-bound of the form
\begin{equation*}
    \EE [\pseudoregret(T)]
    \lesssim
    \regretfunction(\beta)
    :=  
    C_1 \norm{\Delta}_1 \beta \underbrace{\int_\ee^T \frac1{t \log t} \; \dd t}_{\log \log(T)}
    + C_2 D^2 
    \parens*{
        \sum_{\arm \notin \arms^*} \frac1{\Delta_\arm}   
    }
    \log \parens*{
        \frac{T \log (T)}{\beta}
    }.
\end{equation*}
Optimizing $\regretfunction$ with respect to $\beta$, we obtain
\begin{align*}
    \beta^* =
    \frac{
        C_2 D^2 
        \sum_{\arm \notin \arms^*} \frac1{\Delta_\arm}   
    }{
        C_1 \norm{\Delta}_1 \log \log(T) 
    } \; .
\end{align*}
Again, $\beta^*$ depends on $T$, but the dependency in $T$ is mild, as $\log\log(T)$ is basically a numerical constant for any reasonable time horizon (say $T \le 10^{10}$).\footnote{
    The reason why we have remaining $\log(-)$ terms is that we are solving equations of the form $a x + b \log(x) = c$, of which the solution involves the \href{https://en.wikipedia.org/wiki/Lambert_W_function}{\underline{Lambert $W$ function}}, that cannot be expressed analytically.
}
We recover the quartic critical value of $(D/\Delta)^2$, as in the parameterization $\delta(t) = \frac{\alpha}{\log(t)}$.

\subsection{Comparison with stochastic bandits}
\label{section_comparison_stochastic_bandits}

The critical value $(D/\Delta)^2$ is reminiscent of what is known in stochastic bandits.
Indeed, for such learning problems, it is known that choosing a confidence parameter scaling as $\smash{\delta(t) = \frac{\sigma^2}{\Delta^2} \frac 1t}$ leads to better regret bounds, see for instance \cite[Chapter~9.2]{lattimore_bandit_2020}.\footnote{
    \cite{lattimore_bandit_2020} is focusing on the case $\sigma = 1$.
    We refer the reader to their discussion in Chapter 9.2, where they optimize a regret bound of order $n \Delta \delta + \frac 1\Delta \log(\frac 1\delta)$.
    For $\sigma^2$-subgaussian bandits with $\sigma \ne 1$, the analogue of their bound is $\smash{n \Delta \delta + \frac {\sigma^2}{\Delta} \log(\frac 1\delta)}$, that is optimized for $\delta \approx \sigma^2 / (n \Delta^2)$.
}
In our setting, the quantity $D$ plays a similar role to the variance proxy of subgaussian random variables: according to \Cref{lemma_estimation_error}, the noise on the empirical estimator of the gain scales as $D \sqrt{n \log(1/\delta)}$, while it is $\sigma^2 \sqrt{n \log(1/\delta)}$ for $\sigma^2$-subgaussian bandits. 
However---and this is the crucial difference with stochastic bandits---the noise can have high magnitude even though rewards remain $[0, 1]$.
This is something that is exclusive to Markovian bandits, because by Hoeffding's Lemma, every random variable supported in $[0, 1]$ is $\frac 12$-subgaussian. 
Accordingly, Markovian bandits can ``simulate'' reward distributions with higher noise \emph{yet} values in $[0, 1]$, and the optimistic bonuses must be inflated accordingly.

\section{Conclusion}

In this paper, we considered a class of Markovian bandit learning problems that allow for non-observable states and constrained decision epochs.
We argued that pure policies provide an appropriate benchmark for Markovian bandits by introducing a subclass, self-degrading bandits, for which they are optimal.
We first provided a lower bound for the regret of rarely-switching algorithms, revealing that strict logarithmic regret remains unreachable---in opposition to the stochastic bandit case.
We then introduced \algname{} and showed that, despite the non-observability of the environment's inner state, non-trivial regret guarantees can be achieved.
We have further discussed a structural assumption---a prior knowledge of the bias span---under which the classical $\OH(\log(T))$ instance-dependent and $\OH(\sqrt{T \log(T)})$ worst-case regrets can be recovered, as soon as the delay between decision epochs remain reasonable. 

We conclude by mentioning several possible directions for future work. A central assumption in our theory is that the state space $\states_\arm$ of every arm can be arbitrarily large.  It would be interesting to study what happens when the state space $\states_\arm$ is known in advance. Our preliminary intuition is that in such a setting, the theory would  become considerably more intricate. Another avenue is to explore alternative information settings, for instance when states are organized into clusters, and rewards depend on the cluster of the current state. Finally, in line with the well-studied restless bandit framework, it would be interesting to consider scenarios where both activated and non-activated arms yield rewards, possibly dependent on the action.

\section*{Acknowledgements}

The authors acknowledge the funding of the French National Research Agency under the PEPR IA FOUNDRY project (ANR-23-PEIA-0003), the ANR LabEx CIMI (grant ANR-11-LABX-0040) within the French State Programme Investissements d'Avenir, the AI Interdisciplinary Institute ANITI, and the PEPR NAI project funded by the France 2030 program under the Grant agreements n°ANR-23-IACL-0002 and n°ANR-22-PEFT-0003. 

\clearpage
\bibliography{biblio}
\bibliographystyle{Style/tmlr}

\clearpage 

\appendix
\numberwithin{equation}{section}
\numberwithin{theorem}{section}
\numberwithin{definition}{section}

\section{General results}

    \ifSubfilesClassLoaded{
        \allowdisplaybreaks
        \onecolumn
        \section{GENERAL RESULTS ON OUR MODEL}
    }{}
    
    In this Appendix, we provide general results on our class of Markovian bandit problems.  
    In \Cref{appendix_score_of_pure_policies}, we prove a simple expression for the score of pure policies, that we use to provide a simple expression for the regret.
    In \Cref{appendix_rarely_switching}, we deep-dive into the theory of rarely switching algorithms (\Cref{definition_rarely_switching_arms}), in which the regret and the pseudo-regret can be related to one another (\Cref{proposition_regret_and_pseudo_regret}).

    \subsection{The score of pure policies and a simple formula for the regret}
    \label[appendix]{appendix_score_of_pure_policies}

    In the main text, we have stated without proof that the score of pure policies could be linked to their gain and bias functions, with $V^{\policy_\arm}(t) = (t-1) \gain_\arm + \EE^{\policy_\arm}[\bias_\arm (\initstate_\arm) - \bias_\arm (\State_\arm (t))]$.
    This formula is a well-known result in Markov Reward Processes. Below  we restate it as a lemma and include its proof for completeness.

    \begin{lemma}[Score of pure policies]
    \label[lemma]{lemma_pure_policies_score}
        Let $\model \in \models$ be an instance.
        Let $\policy_\arm$ be the pure policy pulling arm $\arm \in \arms$ exclusively. 
        For every horizon $T \ge 1$, we have
        \begin{equation}
            \EE^{\model, \policy_\arm} \brackets*{
                \sum_{t=1}^{T-1}
                \Reward(t+1)
            }
            =
            (T - 1) \gain_\arm
            + \EE^{\model, \policy_\arm} \brackets*{
                \bias_\arm (\state_0) 
                - \bias_\arm (\State_\arm (T))
            }
            .
        \end{equation}
    \end{lemma}

    Note that in particular, this implies that $\abs{V^{\policy_\arm}(t) - (t-1)\gain_\arm} \le \vecspan(\bias_\arm)$ for all $t \ge 1$ and $\arm \in \arms$.

    \begin{proof}
        The result follows using the Poisson equation of $\policy \equiv \policy_\arm$ that states that for all $\state_\arm \in \states_\arm$, we have $\gain_\arm + \bias_\arm (\state_\arm) = \reward_\arm (\state_\arm) + \kernel_\arm^1 (\state_\arm) \bias_\arm$.
        More precisely, we have
        \begin{align*}
            \EE^{\model, \policy} \brackets*{
                \sum_{t=1}^{T-1}
                \Reward(t+1)
            }
            - (T - 1) \gain_\arm
            & =
            \EE^{\model, \policy} \brackets*{
                \sum_{t=1}^{T-1}
                \parens*{\Reward(t+1) - \gain_\arm}
            }
            \\
            & \overset{\eqnum{1}}=
            \begin{cases}
                \phantom{+}
                \EE^{\model, \policy} \brackets*{
                    \sum_{t=1}^{T-1}
                    \parens*{
                        \Reward(t+1) - \reward_\arm (\State_\arm (t))
                    }
                }
                \\
                +
                \EE^{\model, \policy} \brackets*{
                    \sum_{t=1}^{T-1}
                    \parens*{
                        \bias_\arm (\State_\arm (t))
                        - \bias_\arm (\State_\arm (t+1))
                    }
                }
                \\
                +
                \EE^{\model, \policy} \brackets*{
                    \sum_{t=1}^{T-1}
                    \parens*{
                        \bias_\arm (\State_\arm (t+1))
                        - \kernel_\arm^1 (\State_\arm (t)) \bias_\arm
                    }
                }
            \end{cases}
            \\
            & \overset{\eqnum{2}}=
            \EE^{\model, \policy} \brackets*{
                \bias_\arm (\State_\arm (1)) 
                - \bias_\arm (\State_\arm (T))
            },
        \end{align*}
        where
        \texteqnum{1} follows from the Poisson equation of $\policy \equiv \policy_\arm$; 
        and
        \texteqnum{2} uses that the first and third terms in \texteqnum{1} are martingale difference sequences, and that the second term is a telescopic sum. 
    \end{proof}

    As a direct consequence of \Cref{lemma_pure_policies_score}, the regret can be approximated by $\sum_{t=1}^{T-1} (\gain_* - \Reward(t))$, where $\gain_* = \max_{\arm \in \arms} \gain_\arm$.
    By \Cref{lemma_pure_policies_score} indeed, by letting $\policy$ be the optimal pure policy, we have $V^{\policy}(T) = (T-1) \gain^* \pm \max_{\arm \in \arms} \vecspan(\bias_\arm)$.
    This approximation can be traced back at least to \citep{auer_near_optimal_2009}, and will be the starting point for our upper and lower bounds on the regret. 
    
    \begin{corollary}[Simple form for the regret]
    \label[corollary]{corollary_regret_simple_expression}
        Let $\model \in \models$ be an instance.
        For every learning algorithm $\learner$ and every horizon $T \ge 1$, the regret can be approximated by $\sum_{t=1}^{T-1} (\gain_* - \Reward(t))$, with an error term of at most $\max_{\arm \in \arms} 
            \vecspan(\bias_\arm)$, that is,
        \begin{equation}
        \label{equation_regret_simple_expression}
            \abs*{
                \regret(T; \model, \learner)
                - 
                \EE^{\model, \learner} \brackets*{
                    \sum_{t=1}^{T-1}
                    \parens*{
                        \gain_* - \Reward(t+1)
                    }
                }
            }
            \le
            \max_{\arm \in \arms} 
            \vecspan(\bias_\arm)
            .
        \end{equation}
    \end{corollary}

    The error $\max_{\arm \in \arms} \vecspan(\bias_\arm)$ in \Cref{equation_regret_simple_expression} does not grow with $T \ge 1$.
    Since the regret is eventually growing as $\log(T)$ in the instance dependent setting and as $\sqrt{T}$ in the worst-case setting, the latter constant can safely be ignored. 
    
    \subsection{Proof of \texorpdfstring{\Cref{proposition_regret_and_pseudo_regret}}{Proposition 1}: Rarely switching algorithms}
    \label[appendix]{appendix_rarely_switching}

    The link between the regret and the pseudo-regret for rarely switching algorithms plays a crucial role in our analysis, as it connects the regret to the number of times suboptimal arms are pulled. For an arbitrary learner, such a link does not generally exist, since pulling suboptimal arms does \emph{not} necessarily imply that the regret increases accordingly: a smart controller could deliberately pull suboptimal arms when the optimal arm is in a bad state. However, $\switchingfunction$-rarely switching algorithms preclude such behavior, which allows us to obtain a direct connection between the regret and the pseudo-regret, via the switching function $\switchingfunction$.  This result was stated in \Cref{proposition_regret_and_pseudo_regret} and we prove a more precise version of it below in \Cref{proposition_regret_and_pseudo_regret_appendix}.

    \begin{proposition}[Regret and pseudo-regret]
    \label[proposition]{proposition_regret_and_pseudo_regret_appendix}
        Let $\model \in \models$ be an arbitrary model with a unique optimal arm $\arm^* \in \arms$.
        Let $\learner$ be a learning algorithm that satisfies a $\psi$-rarely switching property.
        Denote $\Delta_0 := \min \braces*{\gain_* - \gain_\arm: \gain_\arm < \gain_*}$ the minimum gain-gap. 
        For every horizon $T \ge 1$, we have 
        \begin{equation}
            \abs*{
                \regret(T) 
                - \EE [\pseudoregret(T)]
            }
            \le
            \max_{\arm \in \arms} \parens[\big]{\vecspan(\bias_\arm)}
            \cdot
            \parens*{
                \frac{2 \abs{\arms}}{\Delta_0}
                \cdot \switchingfunction \parens*{
                    \EE [\pseudoregret(T)]
                } 
                + \parens*{
                    \frac{2 \abs{\arms}}{\Delta_0}
                    - 1
                } C_\switchingfunction 
                + 1
            },
        \end{equation}
        where $\pseudoregret$ is the pseudo-regret \eqref{equation_pseudo_regret} and $C_\switchingfunction \in \RR_+$ is the constant given by \Cref{lemma_homogeneity_concave} below.
    \end{proposition}

    To prove \Cref{proposition_regret_and_pseudo_regret_appendix}, we first state \Cref{lemma_regret_and_switches}, which shows that, for any learner, the expected regret and pseudo-regret can be linked through the number of switches.
    The result then follows by bounding that number of switches, for which we rely on the  technical \Cref{lemma_bound_switches,lemma_homogeneity_concave}.

    \begin{lemma}[Regret and number of switches]
    \label[lemma]{lemma_regret_and_switches}
        Let $\model \in \models$ be an arbitrary model.
        Let $\learner$ be a learning algorithm.
        For every horizon $T \ge 1$, we have
        \begin{equation}
        \label{equation_regret_and_switches}
            \abs*{
                \regret(T; \model, \learner)
                - \EE^{\model, \learner} [\pseudoregret(T)]
            }
            \le
            \max_{\arm \in \arms} \vecspan(\bias_\arm)
            \cdot
            \parens*{
                1 
                + \EE^{\model, \learner} \brackets*{
                    \sum_{t=1}^{T-1}
                    \indicator{\Arm(t-1) \ne \Arm(t)}
                }
            }
            .
        \end{equation}
    \end{lemma}

    \begin{proof}
        Let $(\tau_k)$ be the sequence of decision epochs.
        Since $\learner$ is non-preemptive, we know that $\Arm(t) = \Arm(\tau_k)$ for all $t \in \braces*{\tau_k, \ldots, \tau_{k+1}-1}$.
        We clip decision epochs to $\braces{1, \ldots, T-1}$ by introducing $\tau'_k := \min \braces{\tau_k, T-1}$.
        We have
        \begin{align}
        \notag
            & \regret (T; \model, \learner)
            \overset{\eqnum{1}}\le
            \EE^{\model, \learner} \brackets*{
                \sum_{k=1}^\infty
                \sum_{t=\tau'_k}^{\tau'_{k+1}-1} \parens*{
                    \gain_* - \Reward(t+1)
                }
            }
            + \max_{\arm \in \arms} \vecspan(\bias_\arm)
            \\
        \notag
            & \overset{\eqnum{2}}=
            \EE^{\model, \learner} \brackets*{
                \sum_{k=1}^\infty
                \sum_{t=\tau'_k}^{\tau'_{k+1}-1} \parens*{
                    \gain_* - \gain_{\Arm(\tau'_k)}
                    - \bias_{\Arm(\tau'_k)} (\State_{\Arm(\tau'_k)}(t))
                    + \bias_{\Arm(\tau'_k)} (\State_{\Arm(\tau'_k)}(t+1))
                }
            }
            + \max_{\arm \in \arms} \vecspan(\bias_\arm)
            \\
        \notag
            & =
            \EE^{\model, \learner} \brackets*{
                \sum_{k=1}^\infty
                \parens*{\tau'_{k+1} - \tau'_k}
                \parens*{\gain_* - \gain_{\Arm(\tau'_k)}}
            }
            + 
            \EE^{\model, \learner} \brackets*{
                \sum_{k=1}^\infty
                \parens*{
                    \unit_{\State(\tau'_{k+1})} - \unit_{\State(\tau'_{k})}
                }
                \bias_{\Arm(\tau'_k)}
            }
            + \max_{\arm \in \arms} \vecspan(\bias_\arm)
            \\
        \notag
            & \overset{\eqnum{3}}\le
            \EE^{\model, \learner} \brackets*{
                \sum_{k=1}^\infty
                \parens*{\tau'_{k+1} - \tau'_k}
                \parens*{\gain_* - \gain_{\Arm(\tau'_k)}}
            }
            + 
            \max_{\arm \in \arms} \vecspan(\bias_\arm)
            \cdot
            \parens*{
                1 
                + \EE^{\model, \learner} \brackets*{
                    \sum_{k=1}^\infty
                    \indicator{\Arm(\tau'_{k-1}) \ne \Arm(\tau'_{k})}
                }
            }
            \\
        \label{equation_regret_bad_arm_pulls_1}
            & \overset{\eqnum{4}}\le
            \EE^{\model, \learner} \brackets*{
                \sum_{\arm \in \arms}
                \parens*{\gain_* - \gain_\arm}
                \visits_{\arm}(T)
            }
            + 
            \max_{\arm \in \arms} \vecspan(\bias_\arm)
            \cdot
            \parens*{
                1
                + \EE^{\model, \learner} \brackets*{
                    \sum_{t=1}^{T-1}
                    \indicator{\Arm(t-1) \ne \Arm(t)}
                }            
            },
        \end{align}
        where 
        \texteqnum{1} invokes \Cref{corollary_regret_simple_expression};
        \texteqnum{2} uses that at time $t$ with $\tau_k' \leq t < \tau_{k+1}'$ the arm $A(\tau_k')$ is activated and hence, $R(t+1)$ can be replaced using the Poisson equation of the Markov reward process  under the pure policy $\pi_{A(\tau_k')}$;
        \texteqnum{3} uses that the right-most term is nearly telescopic (see \Cref{lemma_bound_nearly_telescopic_bias});
        and
        \texteqnum{4} follows by definition of visit counts. 
        In \Cref{equation_regret_bad_arm_pulls_1}, we recognize the expected pseudo-regret in the first term, providing the desired upper bound of the regret with respect to the pseudo-regret.
        The lower bound is obtained similarly. 
    \end{proof}

    \begin{lemma}[Bound on the number of switches]
    \label[lemma]{lemma_bound_switches}
        Let $\model \in \models$ be an arbitrary model.
        Consider a learning algorithm that satisfies a $\psi$-rarely switching property.
        Then, for all horizon $T \ge 1$ and arm $\arm \in \arms$, we have
        \begin{equation}
        \label{equation_bound_switches}
            \sum_{t=1}^{T-1}
            \indicator{
                \Arm(t-1) \ne \arm,
                \Arm(t) = \arm
            }
            \le
            \sum_{\arm' \ne \arm}
            \psi(\visits_{\arm'}(T))
            .
        \end{equation}
    \end{lemma}

    \begin{proof}
        Because the algorithm is rarely switching, there is a non-decreasing sublinear function $\psi : \RR_+ \to \RR$ such that the number of switches off is bounded as $\sum_{t=1}^{T-1} \indicator{\Arm(t-1) = \arm, \Arm(t) \ne \arm} \le \psi (\visits_\arm (t))$ for all arm $\arm \in \arms$ and $T \ge 1$.
        Then, we have
        \begin{align}
        \nonumber
            \sum_{t=1}^{T-1}
            \indicator{\Arm(t-1) \ne \arm, \Arm(t) = \arm}
            & \overset{\eqnum{1}}\le
            \sum_{\arm' \ne \arm}
            \sum_{t=1}^{T-1}
            \indicator{\Arm(t-1) = \arm', \Arm(t) \ne \arm'}
            \\
        \nonumber
             & \overset{\eqnum{2}}\le
            \sum_{\arm' \ne \arm}
            \psi(\visits_{\arm'}(T)),
        \end{align}
        where 
        \texteqnum{1} holds because switching onto arm $\arm$ means that we are switching off some arm $\arm' \ne \arm$;
        \texteqnum{2} follows from the rarely switching property.
        This concludes the proof.
    \end{proof}

    \begin{lemma}[Homogeneity of concave functions]
    \label[lemma]{lemma_homogeneity_concave}
        Let $f : \RR_+ \to \RR_+$ be a non-negative smooth concave function and let $C_f := \sup_{x \ne 1} (f(x) - x f(1))/(x-1)$.
        Then $C_f < \infty$ and
        \begin{equation}
        \label{equation_homogeneity_concave}
            \forall \lambda \ge 0,
            \forall x \ge 0,
            \quad
            f((1 + \lambda) x)
            \le
            (1 + \lambda) f(x)
            + \lambda \cdot C_f.
        \end{equation}
    \end{lemma}
    
    \begin{proof}
        First, we check that $C_f < \infty$.
        Let $\psi(x) := (f(x) - x f(1))/(x - 1)$ for $x \ne 1$, which is a smooth function on its domain. 
        As $f$ is concave, it is bounded by a linear function $\alpha x + \beta$ and we have $\limsup_{x \to \infty} \psi(x) \le \alpha - f(1)$.
        We are left to study the convergence properties of $\psi(x)$ as $x \to 1$.
        Because $f$ is smooth, we have the Taylor approximation
        \begin{equation*}
            \psi(1 + u)
            =
            \frac{
                f(1 + u) - (1 + u) f(1)
            }{
                u
            }
            = \frac{
                u (f'(1) - f(1))
                + \OH (u^2)
            }{
                u
            }
            =
            f'(1) - f(1) + \OH (u)
        \end{equation*} 
        and $\psi$ is therefore continuous at $1$.
        So $\psi$ is bounded and $C_f < \infty$.

        Now, we establish \Cref{equation_homogeneity_concave}. Since $f$ is smooth and concave, we have
        \begin{equation}
        \label{equation_homogeneity_concave_1}
            f((1 + \lambda) x)
            \le
            f(x)
            + \lambda x f'(x).
        \end{equation}
        Now, we bound $x f'(x)$ with respect to $f(x)$ and $C_f$. 
        By definition of $C_f$, we have $(x - 1)C_f \ge f(x) - x f(1)$ for all $x \ge 0$. 
        Reordering terms, we get $f(x) \le x (f(1) + C_f \cdot  (1 - \frac 1x))$.
        From a straight-forward computation, we check that the function $g(x) := f(1) + C_f \cdot (1 - \frac 1x)$ is smooth and satisfies $x^2 g'(x) = C_f$, while $f'(x) = x g'(x) + g(x)$. 
        Combined, we get
        \begin{equation}
        \label{equation_homogeneity_concave_2}
            x f'(x) 
            = x^2 g'(x) + x g(x)
            = C_f + f(x)
            .
        \end{equation}
        Combining \Cref{equation_homogeneity_concave_1,equation_homogeneity_concave_2}, this completes the proof. 
    \end{proof}

    We can now prove \Cref{proposition_regret_and_pseudo_regret_appendix}. 

    \begin{proof}[Proof of \Cref{proposition_regret_and_pseudo_regret_appendix}]
 By \Cref{lemma_regret_and_switches},
 controlling the difference between the expected regret and the pseudo-regret reduces to bounding the number of switches in \Cref{equation_regret_and_switches}.
        Let $\arm^* = \arg \max_{\arm \in \arms} \gain_\arm$ be the unique optimal arm. 
        Let $\gap_0 := \min \braces*{\gain_* - \gain_\arm: \arm \ne \arm^*}$ be the gain-gap between the optimal and the best suboptimal arms. 
        We have
        \begin{align*}
            & \EE^{\model, \learner} \brackets*{
                \sum_{t=1}^{T-1}
                \indicator{\Arm(t-1) \ne \Arm(t)}
            }
            \\
            & \qquad =
            \EE^{\model, \learner} \brackets*{
                \sum_{\arm \ne \arm^*}
                \sum_{t=1}^{T-1}
                \indicator{
                    \Arm(t-1) \ne \arm,
                    \Arm(t) = \arm
                }
            }
            +
            \EE^{\model, \learner} \brackets*{                          
                \sum_{t=1}^{T-1}
                \indicator{
                    \Arm(t-1) \ne \arm^*,
                    \Arm(t) = \arm^*
                }
            }
            \\
            & \qquad \overset{\eqnum{1}}\le
            \EE^{\model, \learner} \brackets*{
                \sum_{\arm \ne \arm^*}
                \psi(\visits_{\arm}(T))
            }
            +
            \EE^{\model, \learner} \brackets*{
                \sum_{\arm \ne \arm^*}
                \psi (\visits_{\arm}(T))
            }
            \\
            & \qquad \overset{\eqnum{2}}\le
            2 \abs{\arms} 
            \cdot \EE^{\model, \learner} \brackets*{
                \psi \parens*{
                    \frac 1{\gap_0}
                    \sum_{\arm \in \arms}
                    (\gain_* - \gain_\arm)
                    \visits_{\arm}(T)
                }
            }
            \\
            & \qquad \overset{\eqnum{3}}\le
            2 \abs{\arms} 
            \cdot \psi \parens*{
                \EE^{\model, \learner} 
                \brackets*{
                    \frac 1{\gap_0}
                    \sum_{\arm \in \arms}
                    (\gain_* - \gain_\arm)
                    \visits_{\arm}(T)
                }
            }
            \\
            & \qquad \overset{\eqnum{4}}\le
            \frac {2\abs{\arms}}{\gap_0} 
            \cdot \psi \parens*{
                \EE^{\model, \learner} 
                \brackets*{
                    \sum_{\arm \in \arms}
                    (\gain_* - \gain_\arm)
                    \visits_{\arm}(T)
                }
            } 
            + \parens*{
                \frac {2\abs{\arms}}{\gap_0}
                - 1 
            } C_\psi,
        \end{align*}
        where
        \texteqnum{1} is obtained by bounding the first term using the $\psi$-rarely switching property and the second using \Cref{lemma_bound_switches};
        \texteqnum{2} follows by monotonicity of $\psi$;
        \texteqnum{3} follows by Jensen's inequality;
        and
        \texteqnum{4} holds following a homogeneity property of smooth concave functions, where $C_\psi < \infty$ is a constant, see \Cref{lemma_homogeneity_concave}.
    \end{proof}

    \ifSubfilesClassLoaded{
        
        \bibliographystyle{plainnat}
        \bibliography{biblio}
    }{}
    
\end{document}

    \ifSubfilesClassLoaded{
        \allowdisplaybreaks
        \onecolumn
    }{}

    \section{Self-degrading restless bandits}
    \label[appendix]{appendix_self_degrading_restless_bandits}
    
    In this appendix, we aim to provide proofs for the claims of \Cref{section:pure_policies}. 
    This includes the proof of asymptotic optimality of pure policies within self-degrading bandits.
    In particular, we prove \Cref{proposition_self_degrading_regret,proposition_self_degrading}, and show that the models described in \Cref{example:GE_channels} and \Cref{example:NA_toy_model} are instances of \emph{self-degrading} restless bandits.
    
    Recall the definition of a \emph{self-degrading} restless bandit from the main body: 
    
    \begin{definition}[Self-degrading restless bandits]
    \label{appendix_definition_self_degrading}
        Given an instance of restless bandit $\model \in \models$ together with an activation set $\activationset$, an arm $\arm \in \arms$ is said \emph{self-degrading} if, whatever the controller $\learner$ respecting $\activationset$,
        \begin{equation*}
            \Arm(t) \neq \arm 
            \implies
            \EE^{\model, \learner} \brackets*{
                \bias_\arm(\State_\arm (t+1)) 
                \middle|
                \ObservedHistory(t)
            }
            \le
            \EE^{\model, \learner} \brackets*{
                \bias_\arm(\State_\arm (t))
                \middle|
                \ObservedHistory(t)
            }
            \; .
        \end{equation*}
        The instance $\model$ is said \emph{self-degrading} if all of its arms are. 
        The set of all self-degrading instances of Markovian bandits is written $\models_{\mathrm{sd}}$.
    \end{definition}

    \subsection{Proof of \texorpdfstring{\Cref{proposition_self_degrading_regret,proposition_self_degrading}}{Propositions 2 and 3}: Regret and pseudo-regret in self-degrading instances}
    \label[appendix]{appendix:self_degrading_regret}

    Here, we show that the regret is lower-bounded by the expected pseudo-regret in self-degrading instances, by proving \Cref{proposition_self_degrading_regret} from the main text.
    Coupled with \Cref{lemma_pure_policies_score} that expresses the score of the pure policies with respect to the gain and the bias, the optimality of pure policies in self-degrading instances (\Cref{proposition_self_degrading}) is immediate.
    
    \textbf{\Cref{proposition_self_degrading_regret}} (Regret and pseudo-regret in self-degrading instances)\textbf{.}
    {
        \itshape
        The regret is bounded below by the pseudo-regret on self-degrading instances.
        That is, for every self-degrading instance of restless bandits $\model \in \models_\mathrm{sd}$, there exists a constant $C_\model < \infty$ such that, whatever the learning algorithm $\learner$, 
        \begin{equation*}
            \forall T \ge 1,
            \qquad
            \regret(T; \model, \learner)
            \ge
            \EE^{\model,\learner} \brackets*{
                \pseudoregret(T)
            }
            - C_\model
            \; .
        \end{equation*}
    }

    \begin{proof}[Proof of \Cref{proposition_self_degrading_regret}]
        For all arm $\arm \in \arms$ consider the sequences of stopping times 
        \[
            \tau^\arm_{k}:= 
            \inf\braces{t> \tilde \tau^\arm_k : \Arm(t) = \arm} \land T,
            \qquad
            \tilde \tau^\arm_{k+1}:= 
            \inf\braces{t> \tau^\arm_k : \Arm(t) \neq \arm} \land T
        \]
        which are respectively the first 
        For any learning algorithm $\learner$ and $\state_0 \in \states$,
        \begin{align*}
            V^{\learner}(T; \model)
            &=
            \EE^{\model,\learner} \brackets*{
                \sum_{t=1}^T
                \Reward (t)
            }
            \\
            &=
            \EE^{\model,\learner} \brackets*{
                \sum_{\arm \in \arms}
                \sum_{k=1}^\infty
                \sum_{t=\tau^\arm_k}^{\tilde{\tau}^\arm_{k+1}-1}
                \Reward(t)
            }
            \\
            & \overset{\eqnum{1}}{=} 
            \EE^{\model,\learner} \brackets*{
                \sum_{\arm \in \arms}
                \sum_{k=1}^\infty
                \sum_{t=\tau^\arm_k}^{\tilde \tau^\arm_{k+1}-1}
                \parens*{
                    \gain_\arm 
                    + 
                    \bias(\State_\arm(t)) 
                    - 
                    \bias(\State_\arm(t+1))
                }
            }
            \\
            &\le
            \EE^{\model,\learner} \brackets*{
                \sum_{\arm \in \arms}
                \parens*{
                    \gain_\arm
                    \visits_\arm(T)
                    + 
                    \sum_{k=1}^\infty
                    \parens*{
                        \bias_\arm(\State_\arm(\tau^\arm_k))
                        -
                        \bias_\arm(\State_\arm(\tilde \tau^\arm_{k+1}))
                    }
                }
            }
            \\
            &=
            \EE^{\model,\learner} \brackets*{
                \sum_{\arm \in \arms}
                \gain_\arm
                \visits_\arm(T)
            }
            +
            \EE^{\model,\learner} \brackets*{
                \sum_{\arm \in \arms}
                \parens*{
                    \bias_\arm (\State_\arm(1)) 
                    +
                    \sum_{k=2}^\infty
                    \parens*{
                        \bias_\arm(\State_\arm( \tau^\arm_{k+1}))
                        -
                        \bias_\arm(\State_\arm(\tilde \tau^\arm_{k+1}))
                    }
                }
            }
            \\
            &\overset{\eqnum{2}}{\le} 
            (T - 1) \gain_* 
            - 
            \EE^{\model,\learner} \brackets*{
                \pseudoregret(T)
            }
            + \OH(1)  
        \end{align*}
        where \texteqnum{1} follows from the Poisson equation; \texteqnum{2} stems from both the definition of the pseudo-regret and the \emph{self-degrading} property.
        It is now sufficient to note that
        \begin{align*}
            V^{\policy_*}(T; \model)
            = 
            \EE^{\policy_*, M} \brackets*{
                \sum_{t=1}^T \Reward(t)
            }
            \overset{\eqnum{3}}{=}
            T \gain_* + \OH(1),
        \end{align*}
        where \texteqnum{3} follows directly from the Poisson equation with $\policy_*$ the policy playing only the optimal arm, to conclude:
        \begin{equation*}
            \regret(T)
            =
            V^{\policy_*}(T; \model)
            -
            V^{\learner}(T; \model)
            \ge 
            \EE^{\model,\learner} \brackets*{
                \pseudoregret(T)
            }
            -
            \OH(1).
            \qedhere
        \end{equation*}
    \end{proof}

    \textbf{\Cref{proposition_self_degrading}} (Pure policies are optimal in self-degrading instances)\textbf{.}
    {
        \itshape
        The regret is bounded below on self-degrading instances.
        That is, for every self-degrading instance of restless bandits $\model \in \models_\mathrm{sd}$, there exists a constant $C_\model < \infty$ such that, whatever the learning algorithm $\learner$, we have
        \begin{equation*}
            \EE^{\model, \learner}\brackets*{
                \sum_{t=1}^T
                \Reward(t)
            }
            \le
            \max_{\arm \in \arms}
            V^{\policy_\arm} (T; \model)
            + C_\model
            \; .
        \end{equation*}
    }

    \begin{proof}[Proof of \Cref{proposition_self_degrading}]
        Let $\policy$ be an optimal pure policy, i.e., maximizing $\gain_\arm$.
        Let $\learner$ be an arbitrary learning algorithm.
        We have
        \begin{align*}
            \EE^{\model, \learner} \brackets*{
                \sum_{t=2}^T
                \Reward(t)
            }
            & \overset{\eqnum{1}}=
            V^\policy (T)
            - \regret(T; \model, \learner)
            \\ & \overset{\eqnum{2}}\le
            V^\policy (T)
            - \EE^{\model, \learner} \brackets*{ \pseudoregret(T) }
            + C_\model
            \\ & \overset{\eqnum{3}}\le
            V^\policy (T)
            + \OH(1)
        \end{align*}
        where 
        \texteqnum{1} simply unfolds the definition of the regret;
        \texteqnum{2} bounds the right-term via \Cref{proposition_self_degrading_regret};
        and
        \texteqnum{3} follows by non-negativity of the pseudo-regret.
        So $\policy$ is asymptotically optimal. 
    \end{proof}

    \subsection{Examples of self-degrading restless bandits}

    In the following, we come back to the examples from \Cref{section_self-degrading_examples} and prove they belong to the self-degrading restless bandit family. As in the main text, we begin with the case of Gilbert-Elliott communication channels and follow up by a toy model inspired from the Nerlove-Arrow concept in marketing. 
    
    \subsubsection{Example: Non-preemptive scheduling with Gilbert-Elliott channels}
    \label[appendix]{appendix_GE_channels}
    
    \begin{proposition}
        \label[proposition]{proposition:Gilbert-Elliott_bias_improving}
        Consider a restless bandit model in which, every arm $\arm \in \arms$ corresponds to a two-states MDP with $\states_\arm = \braces*{\texttt{OFF}, \texttt{ON}}$, transition kernels 
        \[
            \kernel_\arm^1 = \kernel_\arm^0 =
            \begin{pmatrix}
                1-y_\arm & y_\arm \\
                x_\arm & 1-x_\arm
            \end{pmatrix},
            \qquad
            x_\arm, y_\arm \in (0,1)
        \]
         and reward distribution satisfying
        \[
            \rewardDistribution_\arm (\texttt{OFF}) = \mathrm{Dirac}(0), 
            \qquad
            \rewardDistribution_\arm (\texttt{ON}) = \mathrm{Bernoulli}(\mu_\arm), \quad \mu_\arm \in (0,1). 
        \]
        Moreover, if each arm is positively autocorrelated, i.e. $1-x_\arm-y_\arm>0$, the activation set is $\activationset = \braces{1}$ and the initial state is $\initstate_\arm = \texttt{ON}$. 
        Then, the model $\model_\arm \equiv (\states_\arm, \initstate_\arm, \braces{0, 1}, \rewardDistribution_\arm, \kernel_\arm^0, \kernel_\arm^1)$ is a self-degrading restless bandit.
    \end{proposition}

    Before proving \Cref{proposition:Gilbert-Elliott_bias_improving}, we first introduce a few lemmas. An important part of the proof is the reformulation of the problem: we replace the full history of observations by a \emph{belief state}, that is, a probability distribution over the hidden states representing the information gathered from past observations. Then, considering the evolution of said belief state, we obtain a \emph{belief MDP}---an MDP on the simplex. The belief state is known to be a \emph{sufficient statistic} of the original problem: an optimal policy for the belief MDP can always be translated back into an optimal policy for the original problem. This is a standard approach from the POMDP literature (see \citep{krishnamurthy2016partially} for a comprehensive introduction to the subject). In the context of Gilbert-Elliott channel, there are only two states. Therefore, we can characterize the whole belief MDP by considering the evolution of the probability of the current state being \texttt{ON} given past observations. It should be noted that to prove the \emph{self-degrading} property, we only need to consider the dynamic of the belief state under the passive action.

    Under the passive action, nothing is observed about the arm. Therefore, the belief state just tends to drift towards the stationary distribution of the Markov chain. The following lemma quantifies this drift.  
    \begin{lemma}[Dynamics of belief state under passive action]
        \label[lemma]{lemma:belief_one_step_decomposition}
        For all arm $\arm \in \arms$ and all $t\in \NN$,  
        \[
            \Arm(t) \neq \arm
            \implies
            \Belief_\arm(t+1) 
            =
            \Belief_\arm(t) (1-x_\arm-y_\arm) + y_\arm
        \]
        where 
        $
            \Belief_\arm(t) 
            :=
            \PP \brackets*{
                \State_\arm(t) = \texttt{ON}
                \mid 
                \ObservedHistory(t) 
            }
        $
        is the belief state, i.e., the probability to be in state \texttt{ON} under the current observations. 
    \end{lemma}
    \begin{proof}
        We write $\Pr \equiv \Pr^{\model, \learner}$ and $\EE \equiv \EE^{\model, \learner}$ all throughout. 
        From the definition of $\Belief_\arm$,
        \begin{align*}
            \Belief_\arm(t+1)
            &:=
            \PP \brackets*{
                \State_\arm(t+1) = \texttt{ON}
                \mid
                \History_o (t+1)
            }
            \\
            &\overset{\eqnum{1}}{=}
            \EE \brackets*{
                \PP \brackets*{
                    \State_\arm(t+1) = \texttt{ON}
                    \mid
                    \History_o (t)
                }
                \mid
                \History_o (t+1)
            }
            \\ 
            &\overset{\eqnum{2}}{=} 
            \PP \brackets*{
                \State_\arm(t) = \texttt{ON}
                \mid
                \History_o (t)
            }
            \kernel_\arm (\texttt{ON} \mid \texttt{ON})
            + 
            \PP \brackets*{
                \State_\arm(t) = \texttt{OFF}
                \mid
                \History_o (t)
            }
            \kernel_\arm (\texttt{ON} \mid \texttt{OFF})
            \\
            &=
            \Belief_\arm(t) (1-x_\arm)+ (1-\Belief_\arm(t))y_\arm
            \\
            &=
            (1-x_\arm-y_\arm)\Belief_\arm(t)+y_\arm
        \end{align*}
        where \texteqnum{1} stems from the tower property and \texteqnum{2} follows from the fact that, under $\Arm(t) \neq \arm$, nothing is observed for arm $\arm$. Hence, update is just the one-step transition of the Markov chain.
    \end{proof}

    \begin{lemma}[Lower bound on the belief state]  
        \label[lemma]{lemma:lower_bound_belief}
        Fix $\arm \in \arms$, and let $x_\arm, y_\arm \in (0,1)$ such that $1-x_\arm-y_\arm>0$. Then, 
        \[
            \Arm(\tau_k) \neq \arm
            \implies
            \forall t \in \braces{\tau_k, \dots, \tau_{k+1}},
            \quad
            \Belief_\arm(t)> \frac{y_\arm}{x_\arm+y_\arm}.
        \]
        In particular, due to the non-preemptive constraint, 
        \[
            \Arm(t) \neq \arm \implies \Belief_\arm(t)> \frac{y_\arm}{x_\arm+y_\arm}.
        \]
        
    \end{lemma}
    \begin{proof}
        We write $\Pr \equiv \Pr^{\model, \learner}$ and $\EE \equiv \EE^{\model, \learner}$ all throughout. 
        We first note that, due to the constraint induced by the activation set $\activationset = {1}$, if arm $\arm$ was activated at $\tau_k$, on the next decision epoch $\tau_{k+1}$, the belief state of that arm is 
        \[
            \Belief_\arm(\tau_{k+1}) =
            \PP \brackets*{
                \State_\arm(t) = \texttt{ON}
                \middle| 
                \ObservedHistory(\tau_{k+1})
            }
            \overset{\eqnum{1}}{=}
            \PP \brackets*{
                \State_\arm(\tau_{k+1}) = \texttt{ON}
                \middle| 
                \State_\arm(\tau_{k+1}-1) = \texttt{ON}
            }
            \overset{\eqnum{2}}{=} 
            1- x_\arm
            \overset{3}{>}
            \frac{y_\arm}{x_\arm +y_\arm} 
        \]
        where \texteqnum{1} follows from the Markov property; \texteqnum{2} is due to the impossibility a reward if $\State_\arm(\tau_{k+1}-1) = \texttt{OFF}$; and \texteqnum{3} stems from the positive autocorrelation assumption. 
        
        Therefore, we only need to cover what happens when the belief state follows a passive dynamic. 
        We proceed by induction. 
        Due to the fixed initial state $\initstate_\arm = \texttt{ON}$, we have that $\Belief(0) = 1> \frac{y_\arm}{x_\arm + y_\arm}$.
        Suppose now that for $t>0$ we have $\Belief_\arm(t) > \frac{y_\arm}{x_\arm+y_\arm}$, then 
        \begin{align*}
            \Belief_\arm(t+1) 
            &\overset{\eqnum{1}}{=} (1-x_\arm-y_\arm)\Belief_\arm(t)+y_\arm 
            \overset{\eqnum{2}}{>} (1-x_\arm-y_\arm)\frac{y_\arm}{x_\arm+y_\arm}+y_\arm = \frac{y_\arm}{x_\arm+y_\arm}
        \end{align*}
        where \texteqnum{1} follows from \Cref{lemma:belief_one_step_decomposition}; and \texteqnum{2} stems from the induction hypothesis $\Belief_\arm(t)> \frac{y_\arm}{x_\arm+y_\arm}$ coupled with the positive autocorrelation assumption $1-x_\arm -y_\arm>0$. 
    \end{proof}

    \begin{lemma}[Monotony of belief state under passive action]
        \label[lemma]{lemma:belief_le_1-p}
        Let $\arm \in \arms$, $x_\arm$ and $y_\arm$ are such that $1-x_\arm-y_\arm>0$. Then for all $t>0$, 
        \[
             \Arm(t) \neq \arm  
             \implies 
             \Belief_\arm(t) \ge \Belief_\arm(t+1)
             \qquad
             \PP\text{-a.s.}
        \]
        where 
        $
            \Belief_\arm(t)
            :=
            \PP \brackets*{
                \State_\arm(t) = \texttt{ON}
                \mid
                \History_o (t)
            }
        $.
    \end{lemma}
    
    \begin{proof}
        Fix $\arm \in \arms$. Under $\Arm(t) \neq \arm$ we have,
        \begin{align*}
            \Belief_\arm(t+1) - \Belief_\arm(t)
            \overset{\eqnum{1}}{=}
            - \Belief_\arm (t) \parens*{
                x_\arm + y_\arm
            }
            + y_\arm
            \overset{\eqnum{2}}{\le}
            0
        \end{align*}
        where \texteqnum{1} follows from \Cref{lemma:belief_one_step_decomposition}; and \texteqnum{2} stems from \Cref{lemma:lower_bound_belief}.
    \end{proof}

    The last lemma allows us to compare the bias of the two hidden states \texttt{ON} and \texttt{OFF}. 
     \begin{lemma}[Bias of the Gilbert-Elliott channel]
        \label[lemma]{lemma:bias_ON_ge_bias_OFF}
        For all arm $\arm \in \arms$ we have
        \[
            \bias_\arm (\texttt{ON}) \ge \bias_\arm(\texttt{OFF}).
        \]
    \end{lemma}
    \begin{proof}
        Using Poisson equation, we obtain for all arm $\arm \in \arms$
        \[
            \frac{\mu_\arm y_\arm }{x_\arm +y_\arm} + \bias_\arm (\texttt{ON})
            =
            \mu_\arm + (1-x_\arm)\bias_\arm (\texttt{ON}) + x_\arm \bias_\arm (\texttt{OFF}).
        \]
        Which gives after a few algebraic manipulations 
        \begin{align*}
            \bias_\arm(\texttt{ON}) 
            &=  
            \frac{\mu_\arm x_\arm}{x_\arm +y_\arm}
            + (1-x_\arm)\bias_\arm (\texttt{ON}) 
            + x_\arm \bias_\arm (\texttt{OFF})
            \\
            &=
            \frac{\mu_\arm}{x_\arm + y_\arm} + \bias_\arm(\texttt{OFF})
            \ge 
            \bias_\arm(\texttt{OFF}).
            \qedhere
        \end{align*}
    \end{proof}

    We may now prove \Cref{proposition:Gilbert-Elliott_bias_improving}. 

    \begin{proof}[Proof of \Cref{proposition:Gilbert-Elliott_bias_improving}]
        We write $\Pr \equiv \Pr^{\model, \learner}$ and $\EE \equiv \EE^{\model, \learner}$ all throughout. 
        Fix arm $\arm \in \arms$. For all $t>0$ such that $\Arm(t) \neq \arm$, we have 
        \begin{align*}
            \EE \brackets*{
                \bias_\arm \parens*{\State(t+1)}
                \mid
                \ObservedHistory(t)
            }
            & \overset{\eqnum{1}}{=}
            \sum_{\state \in \braces*{\texttt{ON}, \texttt{OFF}}}  \hspace{-1em}
            \bias(\state)                                          \hspace{0.2em}
            \EE \brackets*{
                \indicator{
                    \State_\arm (t+1) = s
                }
                \mid
                \ObservedHistory(t)
            }
            \\
            & \overset{\eqnum{2}}{=}
            \EE \brackets*{
                \Belief_\arm(t+1)
                \mid
                \ObservedHistory(t)
            }
            \bias_\arm(\texttt{ON})
            +
            \parens*{
                1-
                \EE \brackets{
                    \Belief_\arm(t+1)
                    \mid
                    \ObservedHistory(t)
                }
            }
            \bias_\arm (\texttt{OFF})
            \\
            &\overset{\eqnum{3}}{\le}
            \Belief_\arm(t)
            \bias_\arm(\texttt{ON})
            +
            \parens*{
                1-\Belief_\arm(t)
            }
            \bias_\arm (\texttt{OFF})
            \\
            &=
            \EE \brackets*{
                \bias_\arm \parens*{\State(t)}
                \mid
                \History_o(t)
            }
        \end{align*}
        where \texteqnum{1} follows from \Cref{lemma:belief_le_1-p} and \Cref{lemma:bias_ON_ge_bias_OFF}.
    \end{proof}

    \subsubsection{Example: Nerlove-Arrow toy model}
    \label[appendix]{appendix:NA_toy_model}

    In this paragraph, we prove that the Nerlove-Arrow toy model from \Cref{example:NA_toy_model} is self-degrading. 

    \begin{proposition}
        \label[proposition]{proposition_NA_self-degrading}
        Consider a restless bandit model in which, for all arm $\arm \in \arms$,
        \[
            \forall \state \le \state', \; \; 
            \kernel^1_\arm (\state) \le_\mathrm{st} \kernel^1_\arm (\state') 
            \quad \text{and}\quad
            \kernel^0_\arm (\state) \le_\mathrm{st} \unit_\state,
        \]
        where $\le_\mathrm{st}$ is the inequality with respect to the first-order stochastic ordering; the mean reward satisfies
        \[
            \reward_\arm(\state) = \mu_\arm - f_\arm (\state) 
        \]
        where $\mu_\arm >0 $ and $f_\arm : \states_\arm \to \RR_+$ is an decreasing function such that
        \[
            \sum_{t=0}^\infty
            \parens{\kernel^1_\arm}^t(\state)
            f_\arm 
            < \infty
            \quad 
        \]
        for any initial point $\state \in \states_\arm$; and the activation set is $\activationset = [0,1]$ i.e. every step is a decision epoch. 
        
        Then, the induced restless bandit is \emph{self-degrading}.
    \end{proposition}

    \begin{lemma}[Bias representation]
        \label[lemma]{lemma_NA_bias_representation}
        The gain and the bias of the Nerlove-Arrow toy model are respectively
        \[
            \gain_\arm := \mu_\arm,
            \quad \text{and} \quad
            \bias_\arm :
            \state \mapsto
            -
            \sum_{t=0}^\infty
            \parens{\kernel^1_\arm}^t (\state)
            f_\arm.
        \]
    \end{lemma}
    
    \begin{proof}
        For all arm $\arm \in \arms$, let us consider the function 
        \begin{equation}
        \label{eq:NA_def_bias}
            \ell_\arm
            :
            \state \mapsto
            -
            \sum_{t=0}^\infty
            \parens{\kernel^1_\arm}^t (\state)
            f_\arm.
        \end{equation}
        Then, we can show the couple $\parens{\mu_\arm,\ell_\arm}$ satisfy the Poisson equation,
        \begin{align*}
            \mu_\arm + \ell_\arm(\state) 
            = 
            \mu_\arm 
            -
            \sum_{t=0}^\infty
            \parens{\kernel^1_\arm}^t(\state)
            f_\arm 
            =
            \mu_\arm 
            -
            \sum_{t=1}^\infty
            \parens{\kernel^1_\arm}^t(\state)
            f_\arm 
            -
            f_\arm (\state)
            =
            \reward_\arm(\state) + \kernel^1_\arm \ell (\state) 
        \end{align*}
        hence, $\mu_\arm$ and $\ell_\arm$ are respectively the gain and the bias of the arm.
    \end{proof}
        
    \begin{lemma}[Monotonicity of the bias function]
        \label[lemma]{lemma_NA_monotonicity_of_the_bias}
        For the Nerlove-Arrow toy model, the bias function $\state \mapsto \bias_\arm(\state)$ is non-increasing.
    \end{lemma}
    
    \begin{proof}
        Fix an arm $\arm$. Let $\state \le \state' \in \NN$. Then,
        \[
            \bias_\arm (\state)
            \overset{\eqnum{1}}{=}
            -
            \sum_{t=0}^\infty
            \parens{
                \kernel_\arm^1
            }^t    
            (\state)
            f_\arm 
            \overset{\eqnum{2}}{\le}
            -
            \sum_{t=0}^\infty
            \parens{
                \kernel_\arm^1
            }^t
            (\state')
            f_\arm
            \overset{\eqnum{1}}{=}
            \bias_\arm(\state')
        \]
        where \texteqnum{1} stems from \Cref{lemma_NA_bias_representation}; and \texteqnum{2} follows from the stochastic dominance $\kernel^1_\arm (\state) \le_\mathrm{st} \kernel^1_\arm (\state')$ and the fact that $f_\arm$ is decreasing. 
    \end{proof}

    \begin{proof}[Proof of \Cref{proposition_NA_self-degrading}]
        We write $\Pr \equiv \Pr^{\model, \learner}$ and $\EE \equiv \EE^{\model, \learner}$ all throughout. 
        Fix arm $\arm \in \arms$. For all $t>0$ such that $\Arm(t) \neq \arm$ we have
        \begin{align*}
            \EE \brackets*{
                \bias_\arm (\State_\arm(t+1))
                \mid 
                \ObservedHistory(t)
            }
            &\overset{\eqnum{1}}{=}
            \EE \brackets*{
                \EE \brackets*{
                    \bias_\arm(\State_\arm(t+1))
                    \mid
                    \State_\arm(t) 
                }
                \mid
                \ObservedHistory(t)
            }
            \\ 
            &\overset{\eqnum{2}}{=}
            \EE \brackets*{
                \kernel_\arm^0 (\State_\arm(t))
                \bias_\arm
                \mid
                \ObservedHistory(t)
            }
            \\
            &\overset{\eqnum{3}}{\le}
            \EE \brackets*{
                \bias_\arm(\State_\arm(t))
                \mid
                \ObservedHistory(t)
            }
        \end{align*}
        where \texteqnum{1} follows from the tower property; \texteqnum{2} stems from $\Arm(t) \neq \arm$; and \texteqnum{3} is a consequence of both \Cref{lemma_NA_monotonicity_of_the_bias} and the stochastic dominance property $\kernel^0_\arm (\state) \le_\mathrm{st} \unit_\state$.
    \end{proof}
    
    \ifSubfilesClassLoaded{
        
        \bibliographystyle{plainnat}
        \bibliography{biblio}
    }{}
    
\end{document}

\section{Regret lower bounds (\texorpdfstring{\Cref{theorem_lowerbound_dependent_universal,theorem_linear_worst_case}}{Theorems 4 and 7})}

    \ifSubfilesClassLoaded{
        \allowdisplaybreaks
        \onecolumn
        
        In this file, we detail the lower bounds.
    }{}

    In this Appendix, we provide the proofs of the two regret lower bounds of the main text (\Cref{theorem_lowerbound_dependent_universal,theorem_linear_worst_case}). 
    Our proof technique follows a well-established line that we have  presented in the proof sketch of \Cref{theorem_lowerbound_dependent_universal}.
    The core of the argument consists in finding two instances $\model_1$ and $\model_2$ that (1) are difficult to distinguish from a statistical view points and (2) do not have the same optimal arm.
    Using a change of measure (\Cref{appendix_change_of_measure}, see \Cref{theorem_bretagnolle_huber,theorem_change_of_measure}), we can bind the behavior of the learning algorithm in $\model_1$ to its behavior on $\model_2$.
    Then, the consistency of the algorithm forces it to play mostly the optimal arm of $\model_2$ when running on $\model_2$; However, because the behaviors on $\model_1$ and $\model_2$ are bounded, this gives information on what the algorithm is doing on $\model_1$.
    For a well-chosen $\model_2$ (see \Cref{lemma_model_perturbation_absorbing}), we derive a lower bound on the number of pulls of suboptimal arms.
    In turn, this provides a lower bound on the pseudo-regret that we convert to a lower bound on the regret using the $\switchingfunction$-rarely switching property (\Cref{proposition_regret_and_pseudo_regret_appendix}).
    
    Articulating this argument around two changes of measure of different style, we obtain the worst-case regret lower bound (\Cref{appendix_lowerbound_worst_case}) and the instance dependent regret lower bound (\Cref{appendix_lowerbound_instance_dependent}). 

    \subsection{Changes of measure and an idealistic state transform}
    \label[appendix]{appendix_change_of_measure}

    We start by providing a refresher on the change of measure machinery. 

    Given two probability measures $\Pr$ and $\Pr'$ over the same probability space satisfying $\Pr \ll \Pr'$, recall that the \emph{relative entropy} from $\Pr'$ to $\Pr$ is given by $\KL(\Pr||\Pr') := \integral \log (\dd \Pr/\dd \Pr') \, \dd \Pr$, where $\dd \Pr/\dd \Pr'$ is the Radon-Nikodym derivative of $\Pr$ with respect to $\Pr'$. 

    In this work, we will use a different style of change of measure, depending on the kind of regret lower bound of interest.
    The first is known as Bretagnolle-Huber inequality (\Cref{theorem_bretagnolle_huber}) and will be used for the worst-case regret lower bound (\Cref{theorem_linear_worst_case}). 

    \begin{theorem}[Bretagnolle-Huber inequality, \cite{bretagnolle_estimation_1979}]
    \label{theorem_bretagnolle_huber}
        Let $\Pr \ll \Pr'$ be two probability measures over the same probability space, and let $\event$ be an event.
        Then
        \begin{equation}
            \Pr \parens{\event}
            + \Pr' \parens{\event^c}
            \ge
            \frac 12
            \exp \braces*{
                - \KL(\Pr||\Pr')
            }.
        \end{equation}
    \end{theorem}

    The second has no name unfortunately (\Cref{theorem_change_of_measure}), and is closer to the lower bound of \cite[Lemma~1]{kaufmann_complexity_2016}.
    It will be used for the instance-dependent regret lower bound (\Cref{theorem_lowerbound_dependent_universal}). 

    \begin{theorem}[\cite{boone_regret_2025}]
    \label{theorem_change_of_measure}
        Let $\Pr \ll \Pr'$ be two probability measures over the same probability space with respective expectation operators $\EE$ and $\EE'$. 
        For every $X > 0$ positive random variable, we have
        \begin{equation}
            \log \EE'[X] + \KL(\Pr||\Pr')
            \ge 
            \EE [\log(X)].
        \end{equation}
    \end{theorem}

    To make use of \Cref{theorem_bretagnolle_huber,theorem_change_of_measure}, we provide an expression of the relative entropy $\KL(\Pr\|\Pr')$ between the two probability measures $\Pr, \Pr'$ of the stochastic processes induced by the models $\model, \model'$, see \Cref{lemma_relative_entropy}.

    Given two instances $\model \equiv (\states_\arm, \rewardDistribution_\arm, \kernel_\arm)_{\arm \in \arms}$ and $\model' \equiv (\states'_\arm, \rewardDistribution'_\arm, \kernel'_\arm)_{\arm \in \arms}$, we write $\model \ll \model'$ if for every arm $\arm \in \arms$,
    (1) $\states_\arm \subseteq \states'_\arm$;
    (2) $\rewardDistribution_\arm (\state_\arm) \ll \rewardDistribution'_\arm (\state_\arm)$ for all $\state_\arm \in \states_\arm$; and
    (3) $\support(\kernel_\arm (\state_\arm)) \subseteq \support(\kernel'_\arm (\state_\arm))$ for all $\state_\arm \in \states_\arm$.

    \begin{lemma}[Relative entropy]
    \label[lemma]{lemma_relative_entropy}
        Let $\model \equiv (\states_\arm, \rewardDistribution_\arm, \kernel_\arm)_{\arm \in \arms} \ll \model' \equiv (\states'_\arm, \rewardDistribution'_\arm, \kernel'_\arm)_{\arm \in \arms}$ be two instances.
        Let $\learner$ be a learning algorithm, $T \ge 1$ a horizon and $\state_0 \in \states$ an initial state.
        Upon running $\learner$ initialized at $\state_0$ for $T$ steps on $\model$ (respectively $\model'$), we obtain a probability measure $\Pr$ (respectively $\Pr'$) on the possible histories up to time $T$. 
        We have $\Pr \ll \Pr'$ and by denoting $\dd \Pr/\dd \Pr'$ the Radon-Nikodym derivative of $\Pr$ with respect to $\Pr'$, the relative entropy from $\Pr'$ to $\Pr$ satisfies
        \begin{equation}
        \label{equation_relative_entropy}
        \begin{aligned}
            \KL (\Pr||\Pr')
            =
            \begin{cases}
            \displaystyle
                \phantom{+}
                \EE_{\state_0}^{\model, \learner} 
                \brackets*{
                    \sum_{t=1}^{T-1}
                    \KL \parens*{
                        \rewardDistribution_{\Arm(t)}(\State_{\Arm(t)}(t))
                        \middle\|
                        \rewardDistribution'_{\Arm(t)}(\State_{\Arm(t)}(t))
                    }
                }
                \\
                \displaystyle
                +
                \EE_{\state_0}^{\model, \learner} 
                \brackets*{
                    \sum_{t=1}^{T-1}
                    \KL \parens*{
                        \kernel^1_{\Arm(t)}(\State_{\Arm(t)}(t))
                        \middle\|
                        \kernel'^1_{\Arm(t)}(\State_{\Arm(t)}(t))
                    }
                }
                \\
                \displaystyle
                + 
                \EE_{\state_0}^{\model, \learner} 
                \brackets*{
                    \sum_{t=1}^{T-1}
                    \sum_{\arm \ne \Arm(t)}
                    \KL \parens*{
                        \kernel^0_{\arm}(\State_{\arm}(t))
                        \middle\|
                        \kernel'^0_{\arm}(\State_{\arm}(t))
                    }
                }
            \end{cases}.
        \end{aligned}
        \end{equation}
    \end{lemma}

    \begin{proof}
        The fact that $\Pr \ll \Pr'$ comes from $\model \ll \model'$. 
        \\
        Denote $\history(T) \equiv (\state(1), \arm(1), \reward(2), \state(2), \ldots, \reward(T), \state(T))$ an instance of history at time $T$. 
        By Markov's property, the Radon-Nikodym derivative can be written as
        \begin{align*}
            & \frac{\dd \Pr}{\dd \Pr'} (\history(T))
            =
            \frac{\Pr (\History(T) = \history(T))}{\Pr' (\History(T) = \history(T))}
            \\
            & =
            \product_{t=1}^{T-1}
            \parens*{
                \frac{
                    \rewardDistribution_{\arm (t)}(\reward(t+1)|\state_{\arm(t)}(t))
                }{
                    \rewardDistribution'_{\arm (t)}(\reward(t+1)|\state_{\arm(t)}(t))
                }
                \frac{
                    \kernel^1_{\arm (t)}(\state_{\arm(t)}(t+1)|\state_{\arm(t)}(t))
                }{
                    \kernel'^1_{\arm (t)}(\state_{\arm(t)}(t+1)|\state_{\arm(t)}(t))
                }
                \cdot
                \product_{\arm \ne \arm(t)}
                \frac{
                    \kernel^0_{\arm}(\state_{\arm}(t+1)|\state_{\arm}(t))
                }{
                    \kernel'^0_{\arm}(\state_{\arm}(t+1)|\state_{\arm}(t))
                }
            }
            \\
            & =
            \exp \braces*{
                \sum_{t=1}^{T-1}
                \log \parens*{
                    \frac{
                        \rewardDistribution_{\arm (t)}(\reward(t+1)|\state_{\arm(t)}(t))
                    }{
                        \rewardDistribution'_{\arm (t)}(\reward(t+1)|\state_{\arm(t)}(t))
                    }
                    \frac{
                        \kernel^1_{\arm (t)}(\state_{\arm(t)}(t+1)|\state_{\arm(t)}(t))
                    }{
                        \kernel'^1_{\arm (t)}(\state_{\arm(t)}(t+1)|\state_{\arm(t)}(t))
                    }
                    \cdot
                    \product_{\arm \ne \arm(t)}
                    \frac{
                        \kernel^0_{\arm}(\state_{\arm}(t+1)|\state_{\arm}(t))
                    }{
                        \kernel'^0_{\arm}(\state_{\arm}(t+1)|\state_{\arm}(t))
                    }
                }
            }
        \end{align*}
        with the convention that $\log(0) = - \infty$.
        Taking the expectation of the $\log(-)$, we obtain
        \begin{equation}
        \label{equation_relative_entropy_1}
            \KL(\Pr\|\Pr')
            =
            \begin{cases}
                \displaystyle
                \phantom{{}+{}}
                \EE \brackets*{
                    \sum_{t=1}^{T-1}
                    \log \parens*{
                        \frac{
                            \rewardDistribution_{\Arm(t)}(\Reward(t+1)|\State_{\Arm(t)}(t))
                        }{
                            \rewardDistribution'_{\Arm(t)}(\Reward(t+1)|\State_{\Arm(t)}(t))
                        }
                    }
                }
                \\[1.33em]
                \displaystyle
                {} + \EE \brackets*{
                    \sum_{t=1}^{T-1}
                    \log \parens*{
                        \frac{
                            \kernel^1_{\Arm(t)}(\State_{\Arm(t)}(t+1)|\State_{\Arm(t)}(t))
                        }{
                            \kernel'^1_{\Arm(t)}(\State_{\Arm(t)}(t+1)|\State_{\Arm(t)}(t))
                        }
                    }
                }
                \\[1.33em]
                \displaystyle
                {} + \EE \brackets*{
                    \sum_{t=1}^{T-1}
                    \sum_{\arm \ne \Arm(t)}
                    \log \parens*{
                        \frac{
                            \kernel^0_\arm (\State_\arm (t+1)|\State_\arm (t))
                        }{
                            \kernel'^0_\arm (\State_\arm (t+1)|\State_\arm (t))
                        }
                    }
                }
            \end{cases}
        \end{equation}
        To explain how to go from \Cref{equation_relative_entropy_1} to sums of Kullback-Leibler divergences, we focus on the first term of the above, as all the others are handled similarly. 
        We have
        \begin{align*}
            \EE \brackets*{
                \sum_{t=1}^{T-1}
                \log \parens*{
                    \frac{
                        \rewardDistribution_{\Arm(t)}(\Reward(t+1)|\State_{\Arm(t)}(t))
                    }{
                        \rewardDistribution'_{\Arm(t)}(\Reward(t+1)|\State_{\Arm(t)}(t))
                    }
                }
            }
            & \overset{\eqnum{1}}=
            \EE \brackets*{
                \sum_{t=1}^{T-1}
                \EE \brackets*{
                    \log \parens*{
                        \frac{
                            \rewardDistribution_{\Arm(t)}(\Reward(t+1)|\State_{\Arm(t)}(t))
                        }{
                            \rewardDistribution'_{\Arm(t)}(\Reward(t+1)|\State_{\Arm(t)}(t))
                        }
                    }
                    \middle|
                    \State_{\Arm(t)} (t), \Arm(t)
                }
            }
            \\
            & \overset{\eqnum{2}}=
            \EE \brackets*{
                \sum_{t=1}^{T-1}
                \KL\parens*{
                    \rewardDistribution_{\Arm(t)}(\State_{\Arm(t)}(t))
                    \middle\|
                    \rewardDistribution'_{\Arm(t)}(\State_{\Arm(t)}(t))
                }
            }
        \end{align*}
        where
        \texteqnum{1} applies the Tower Property;
        and
        \texteqnum{2} follows by definition of $\KL(\rewardDistribution_\arm (
        \state)\|\rewardDistribution'_\arm(\state))$.
        Rewriting the other two terms similarly and injecting it back to \Cref{equation_relative_entropy_1}, we obtain \Cref{equation_relative_entropy}.
    \end{proof}

    Lastly, for lower bounds (both instance-dependent and worst-case), the proof strategy revolves around an \emph{idealistic state transform}, that adds an idealistic absorbing state of full reward which is difficult to access (\Cref{lemma_model_perturbation_absorbing}).  
    The idea of the construction was sketched in the main body with \Cref{figure_transformation}, that we recall in \Cref{figure_transformation_appendix}.
    
    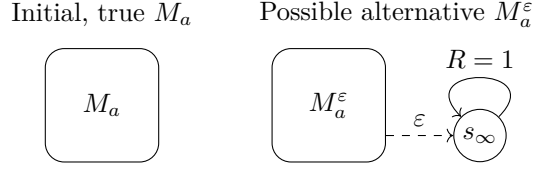
\begin{figure}[ht]
        \centering
        \begin{tikzpicture}
            \begin{scope}[shift={(0, 0)}]
                \node at (0, 1.2) {Initial, true $\model_\arm$}; 
                \node[draw, rectangle, minimum width=1.5cm, minimum height=1.5cm, rounded corners=0.25cm] (Ma) at (0, 0) {$\model_\arm$};
            \end{scope}
            \begin{scope}[shift={(3, 0)}]
                \node at (0.9, 1.2) {Possible alternative $\model_\arm^\epsilon$}; 
                \node[draw, rectangle, minimum width=1.5cm, minimum height=1.5cm, rounded corners=0.25cm] (Ma) at (0, 0) {$\model_\arm^\epsilon$};
                \node[draw, circle, inner sep=0.2em] (soo) at (2, -.4) {$\state_\infty$};
                \draw[->, dashed] (Ma.east) ++ (0,-0.4) to node[midway, above] {$\epsilon$} (soo);
                \draw[->, loop, looseness=5] (soo) to node[midway, above] {$R = 1$} (soo);
            \end{scope}
        \end{tikzpicture}
        \caption{
        \label{figure_transformation_appendix}
            An illustration of the transformation $\model_\arm^\epsilon$.
        }
    \end{figure}

    The precise statement is given below. 

    \begin{lemma}[Idealistic state transform]
    \label[lemma]{lemma_model_perturbation_absorbing}
        Let $\model \equiv (\states, \rewardDistribution, \kernel)$ a Markov reward process\footnote{i.e., a single-action Markov decision process.} with $[0, 1]$-rewards. 
        For all $\epsilon > 0$, there exists a Markov reward process $\model^\epsilon \equiv (\states + \braces{\state_\infty}, \rewardDistribution^\epsilon, \kernel^\epsilon)$ such that 
        \begin{equation*}
            \max_{\state \in \states}
            \braces*{
                \KL(\rewardDistribution(\state)||\rewardDistribution^\epsilon (\state))
                +
                \KL(\kernel(\state)||\kernel^\epsilon (\state))
            }
            \le
            \epsilon
        \end{equation*}
        with $[0, 1]$-rewards, $\gain_*(\model^\epsilon) = 1$ and $\vecspan(\bias^*(\model^\epsilon)) \le 2/\epsilon$.
    \end{lemma}

    \begin{proof}
        Let $\epsilon > 0$.
        Add an extra state $\state_\infty$ to $\model$.
        In $\model^\epsilon \equiv (\states + \braces{\state_\infty}, \rewardDistribution^\epsilon, \kernel^\epsilon)$, the new state $\state_\infty$ is absorbing with maximal reward $\reward^\epsilon (\state_\infty) = 1$; Then, for $\state \in \states$, we set $\rewardDistribution^\epsilon(\state) = \rewardDistribution(\state)$ and
        \begin{equation*}
            \kernel^\epsilon (\state'|\state)
            :=
            (1 - \epsilon) \kernel(\state'|\state)
            + \epsilon \indicator{\state' = \state_\infty}
            .
        \end{equation*}
        We have $\gain_*(\model^\epsilon) = 1$ and $\vecspan(\bias^*(\model^\epsilon)) \le 1/\epsilon$. 
        Continuing, we have
        \begin{align*}
            \max_{\state \in \states}
            \braces*{
                \KL(\rewardDistribution(\state)\|\rewardDistribution^\epsilon (\state))
                +
                \KL(\kernel(\state)\|\kernel^\epsilon (\state))
            }
            & =
            \max_{\state \in \states} 
            \KL \parens*{
                \kernel(\state)
                \|
                (1 - \epsilon) \kernel(\state)
                + \epsilon \cdot \mathrm{Dirac}(\state_\infty)
            }
            \\ & \le
            \max_{\state \in \states} 
            \sum_{\state' \in \states}
            \kernel(\state'|\state)
            \log \parens*{
                \frac{\kernel(\state'|\state)}{(1 - \epsilon) \kernel(\state'|\state)}
            }
            \\ & = \log \parens*{\frac 1{1 - \epsilon}}
            \overset{\eqnum{1}}\le 2 \epsilon \; ,
        \end{align*}
        where \texteqnum{1} holds when $\epsilon \le 1/2$.  
        Accordingly, the family $(\model_{\epsilon/2})_{\epsilon > 0}$ satisfies the desired properties.
    \end{proof}

    \subsection{Proof of \texorpdfstring{\Cref{theorem_lowerbound_dependent_universal}}{Theorem 4}: Instance dependent lower bound}
    \label[appendix]{appendix_lowerbound_instance_dependent}

    In this section, we provide the complete proof of \Cref{theorem_lowerbound_dependent_universal}, restated below for convenience. 

    \textbf{\Cref{theorem_lowerbound_dependent_universal}.}
    {
        \itshape
        Let $\learner$ be a learning algorithm that 
        (1) is consistent on $\models$, and
        (2) is $\switchingfunction$-rarely switching.
        For every instance $\model \in \models$ with  $\gain_* < 1$ and unique optimal pure policy, it holds that  
        \begin{equation}
            \forall \state_0 \in \states,
            \quad
            \liminf_{T \to \infty}
            \frac{
                \regret(T; \model, \learner)
            }{
                \log(T)
            }
            =
            +\infty
            .
        \end{equation}
    }

    \begin{proof}
        Let $\learner$ be a consistent learning algorithm on the universal ambient space $\models$ and assume that it is $\psi$-rarely switching. 
        Let $\model \in \models$ be a model such that $\gain_* < 1$ and fix $\arm \in \arms$ a sub-optimal arm in $\model$.
        For $\epsilon > 0$, consider $\model^\epsilon$ the model in which $\model_\arm^1$ (the Markov reward process of arm $\arm$ induced by $\model_\arm$ upon playing the activating action $1$ exclusively) is changed according to \Cref{lemma_model_perturbation_absorbing}, i.e., upon activating arm $\arm$, the associated Markov chain eventually falls into an absorbing state of maximal reward.
        By definition of $\model^\epsilon$ (\Cref{lemma_model_perturbation_absorbing}), we see that 
        \begin{equation}
        \label{equation_lowerbound_dependent_universal_1}
        \begin{gathered}
            \EE^{\model, \learner} \brackets*{
                \sum_{t=1}^T 
                \KL\parens*{
                    \parens*{
                        \rewardDistribution_{\Arm(t)} 
                        \otimes
                        \kernel_{\Arm(t)} 
                    }
                    \parens*{\State_{\Arm(t)}(t)}
                    \middle \|
                    \parens*{
                        \rewardDistribution^\epsilon_{\Arm(t)} 
                        \otimes
                        \kernel^\epsilon_{\Arm(t)} 
                    }
                    \parens*{\State_{\Arm(t)}(t)}
                }
            }
            \\
            \le
            \epsilon \EE^{\model, \learner}[\visits_\arm (T)],
        \end{gathered}
        \end{equation}
        where $\visits_\arm (T) := \sum_{t=1}^{T-1} \indicator{\Arm(t) = \arm}$ is the number of visits of arm $\arm$ up to time $T$.
        Consider the random variable 
        \begin{equation*}
            X_T := T - \visits_\arm (T)
        \end{equation*}
        tracking the number times $\arm$ isn't played.
        Combining \Cref{equation_lowerbound_dependent_universal_1} with standard information theory \cite[Corollary~C.3]{boone_regret_2025}, we have
        \begin{equation}
        \label{equation_lowerbound_dependent_universal_3}
            \log \EE^{\model^\epsilon, \learner}[X_T]
            + 
            \epsilon \EE^{\model, \learner}[\visits_\arm (T)]
            \ge
            \EE^{\model, \learner} [\log(X_T)]
            .
        \end{equation}
        Our objective is to rework \Cref{equation_lowerbound_dependent_universal_3} to obtain a lower bound on $\visits_{\arm}(T)$.
        Using that $X_t$ tracks the number of times $\arm$ isn't played, and that $\arm$ is optimal (respectively sub-optimal) in $\model^\epsilon$ (respectively $\model$), we rely on the consistency of the algorithm $\learner$ to show that $\EE^{\model, \learner}[\log(X_T)] \sim \log(T)$, while $\log \EE^{\model^\epsilon, \learner}[X_T] = \oh(\log(T))$. 

        On the one hand, $\arm$ is the unique optimal arm in $\model^\epsilon$ by construction.
        As $\regret(T; \model^\epsilon, \learner) = \oh(T^\eta)$ when $T \to \infty$, we deduce $\log \EE^{\model^\epsilon, \learner}[X_T] \le \eta \log(T) + \oh(\log(T))$, where $\eta > 0$ is arbitrary.
        This is proved as follows.
        Binding the regret with the number of suboptimal pulls, we know from \Cref{proposition_regret_and_pseudo_regret_appendix} that there is a positive smooth concave function $\varphi$ such that 
        \begin{equation*}
            \EE^{\model^\epsilon, \learner}
            \brackets*{
                \sum_{\arm' \in \arms} 
                \gap_{\arm'}(\model^\epsilon) 
                \visits_{\arm'}(T)
            }
            \le 
            \regret(T; \model^\epsilon, \learner)
            + 
            \varphi \parens*{
                \EE^{\model^\epsilon, \learner}
                \brackets*{
                    \sum_{\arm' \in \arms} 
                    \gap_{\arm'}(\model^\epsilon) 
                    \visits_{\arm'}(T)
                }
            }.
        \end{equation*}
        Now, letting $\gap_0 (\model^\epsilon) := \min \braces*{\gap_{\arm'} (\model^\epsilon) : \arm' \ne \arm}$ the smallest sub-optimal gap of $\model^\epsilon$, we have 
        \begin{equation*}
            \gap_0 (\model^\epsilon)
            X_T
            \le
            \sum_{\arm' \in \arms} 
            \gap_{\arm'}(\model^\epsilon) 
            \visits_{\arm'}(T)
            \le
            X_T
            .
        \end{equation*}
        Combining the two previous equation and using that $\varphi$ is non-increasing, we find that for all $\eta > 0$, we have
        \begin{equation*}
            \EE^{\model^\epsilon, \learner}[X_T] - 
            \frac 1{\gap_0 (\model^\epsilon)}
            \varphi \parens*{
                \EE^{\model^\epsilon, \learner}[X_T] 
            }
            \le
            \regret(T; \model^\epsilon, \learner)
            = \oh(T^\eta).
        \end{equation*}
        As $\varphi$ is a sublinear function, it follows that $\EE^{\model^\epsilon, \learner}[X_T] = \oh(T^\eta)$.
        So $\log \EE^{\model^\epsilon, \learner}[X_T] \le \eta \log(T) + \oh(\log(T))$.
        This holds for all $\eta > 0$ and $X_t$ is non-negative, so that $\log \EE^{\model^\epsilon, \learner}[X_T] = \oh(\log(T))$.

        Continuing by proving that $\EE^{\model, \learner} [\log (X_T)] = \log(T) + \oh(\log(T))$.
        Using the same technique as above, we show that because the arm $\arm$ is sub-optimal in $\model$, we have $\EE^{\model, \learner}[T - X_T] = \EE^{\model, \learner} [\visits_\arm (T)] = \oh(T^\eta)$.
        So, we have
        \begin{align*}
            \EE^{\model, \learner} [\log(X_T)]
            & \overset{\eqnum{1}}\ge
            \sup_{\lambda > 0} \braces*{
                \lambda 
                \Pr^{\model, \learner}(\log(X_T) \ge \lambda)
            }
            \\
            & \ge
            \sup_{\delta > 0} \braces*{
                \log\parens*{\delta T}
                \parens*{1 - \Pr^{\model, \learner}\parens*{X_T \le \delta T}}
            }
            \\
            & \overset{\eqnum{2}}\ge
            \sup_{\delta > 0} \braces*{
                \log\parens*{\delta T}
                \parens*{
                    1 - \frac{
                        \EE^{\model, \learner}[T - X_T]
                    }{(1 - \delta) T}
                }
            }
            \\
            & \overset{\eqnum{3}}\ge
            \log\parens*{\frac 12 T} \parens*{
                1 - \oh\parens*{T^{- \frac 12}}
            }
            = \log(T) + \oh(\log(T)),
        \end{align*}
        where 
        \texteqnum{1} and \texteqnum{2} follow from Markov's inequality; and
        \texteqnum{3} sets $\delta \equiv \frac 12$ and uses $\EE^{\model, \learner}[T - X_T] = \oh(T^{1/2})$.
        
        Going back to \Cref{equation_lowerbound_dependent_universal_3} and injecting everything, we finally obtain
        \begin{equation}
        \label{equation_lowerbound_dependent_universal_4}
            \EE^{\model, \learner} \brackets*{
                \visits_\arm (T)
            }
            \ge
            \frac{\log(T) + \oh(\log(T))}{\epsilon}
            .
        \end{equation}
        As $\epsilon > 0$ is arbitrary, we conclude accordingly that $\EE^{\model, \learner}[\visits_\arm (T)] = \omega(\log(T))$.
        Binding the regret to the number of visits of suboptimal arms again with \Cref{proposition_regret_and_pseudo_regret_appendix}, we conclude that $\regret(T; \model, \learner) = \omega(\log(T))$.        
    \end{proof}

    \subsection{Proof of \texorpdfstring{\Cref{theorem_linear_worst_case}}{Theorem 7}: Worst-case lower bound}
    \label[appendix]{appendix_lowerbound_worst_case}

    In this section, we provide a complete proof of \Cref{theorem_linear_worst_case}, restated below for convenience. 

    \textbf{\Cref{theorem_linear_worst_case}.}
    {
        \itshape
        Let $\learner$ be a learning algorithm. 
        Let $\activationset \subseteq [0,1]$ be arbitrary. 
        Then
        \begin{equation}
            \forall T \ge 1,
            \quad
            \sup_{\model \in \models}
            \regret(T; \model, \learner)
            \ge 
            \frac {11}{4800} \cdot T.
        \end{equation}
    }

    \begin{proof}
        We introduce the following constants:
        \begin{equation}
        \label{equation_linear_worst_case_1}
            \alpha := 12,
            \quad
            \lambda := \frac 1{10},
            \quad
            \delta := \frac 1{12} \log \parens*{\frac{12}{11}},
            \quad
            c := \frac{11}{4800}
            .
        \end{equation}
        Consider the arm space $\arms := \braces{1, 2}$ and let $\model \equiv (\states, \arms, \rewardDistribution, \kernel^0, \kernel^1)$ be the Bernoulli multi-armed bandits with arms $\arms$ and mean reward $\reward_1 = \frac 12$ and $\reward_2 = 0$.
        That is, $\states_\arm = \braces{1}$ for all arm, with $\rewardDistribution_1 (1) = \mathrm{Ber}(\frac 12)$ and $\rewardDistribution_2 (1) = \mathrm{Ber}(0)$.
        Write $\Pr,\EE$ the probability (resp.~the expectation) under $\model$.
        We see that $\arm = 1$ is the optimal arm of $\model$ and that, in $\model$, both the expected regret and pseudo-regret are equal to $\frac 12 \EE[\visits_2 (T)]$. 
        
        Let $\epsilon \equiv \alpha / T$ with $\alpha := 12$ as given by \Cref{equation_linear_worst_case_1}.
        Consider the alternative model $\model^\epsilon$ where the model of arm $2$ undergoes the idealistic transform of \Cref{lemma_model_perturbation_absorbing}, that is, an extra state is added and $\kernel^{\epsilon, 1}_2$ goes to that state with probability $\epsilon/2$. 
        Write $\Pr^\epsilon,\EE^\epsilon$ the probability (resp.~the expectation) under $\model^\epsilon$.
        We see that $\gain_2 (\model^\epsilon) = 1$.
        Moreover, the reaching time to the idealistic state is $1/(2\epsilon)$, while one scores $0$ instead of $1$, so $\vecspan(\bias_2^\epsilon) \le 2/\epsilon$.
        Also, we have 
        \begin{equation*}
            \KL(\Pr||\Pr^\epsilon) \le \epsilon \EE[\visits_2(T)].
        \end{equation*}
        The proof is split on whether $\EE[\visits_2(T)] \le \delta T$ or not, and proves that in both cases, we have
        \begin{equation}
        \label{equation_linear_worst_case_2}
            \max \braces*{
                \regret(T; \model),
                \regret(T; \model^\epsilon)
            }
            \ge 
            c T.
        \end{equation}
        First, if $\EE[\visits_2 (T)] > \delta T$ then $\regret(T; \model) \ge T/(2\delta) \ge c T$, hence \Cref{equation_linear_worst_case_2}.
        
We now focus on the case $\EE[\visits_2(T)] \le \delta T$.
        By Bretagnolle-Huber inequality (\Cref{theorem_bretagnolle_huber}), we have
        \begin{equation}
            \Pr \parens*{
                \visits_2 (T) \ge \lambda T
            }
            + 
            \Pr^\epsilon \parens*{
                \visits_1 (T) \ge (1 - \lambda) T
            }
            \ge
            \frac 12 \exp \parens*{
                - \epsilon \delta T
            }
            =
            \frac 1{2\ee^{\delta \alpha}}.
        \end{equation}
        We have one out of two cases.
        \begin{itemize}
            \item Case 1: $\Pr (\visits_2 (T) \ge \lambda T) \ge \lambda /(2 \ee^{\delta \alpha})$, in which case we trivially obtain 
            \begin{equation}
                \regret(T; \model)
                = 
                \frac 12 \EE[\visits_2 (T)]
                \ge
                \frac 12 \cdot \lambda T \cdot \frac {\lambda}{2 \ee^{\delta \alpha}}
                =
                c T.
            \end{equation}

            \item Case 2: $\Pr (\visits_1 (T) \ge (1 - \lambda) T) \ge (1 - \lambda)/(2 \ee^{\delta \alpha})$.
            In $\model^\epsilon$, the gain gap between arm $1$ and arm $2$ is $1/2$.
            So $\EE^\epsilon[\pseudoregret^\epsilon(T)] \ge (1 - \lambda) / (4 \ee^{\delta \alpha}) \cdot T$ for the same reason as in Case 1.
            To link this to the regret, we use $\regret(T; \model^\epsilon) \ge \EE^\epsilon[\pseudoregret^\epsilon(T)] - 2/\epsilon$, that holds by \Cref{corollary_regret_simple_expression}.
            We have
            \begin{equation}
                \regret(T; \model^\epsilon) 
                \ge 
                \parens*{
                    \frac{(1-\lambda)^2}{4 \ee^{\delta \alpha}}
                    - \frac 2{\alpha}
                } T
                \ge 
                \parens*{
                    \frac{1-2\lambda}{4 \ee^{\delta \alpha}}
                    - \frac 2{\alpha}
                } T
                = \frac 1{60} \cdot T
                \ge c T.
            \end{equation}
        \end{itemize}
    This proves the result.
\end{proof}

    \ifSubfilesClassLoaded{
        
        \bibliographystyle{plainnat}
        \bibliography{biblio}
    }{}

\end{document}

\section{Regret upper bounds (\texorpdfstring{\Cref{theorem_upperbound_f_diverging,theorem_upperbound_log,theorem_worst_case_sqrt}}{Theorems 6, 8 and 9})}

    \ifSubfilesClassLoaded{
        \allowdisplaybreaks
        \onecolumn
        
        In this file, we provide upper bounds for the instance-dependent and worst case regret of UCB-NOM. 
    }{}

    In this Appendix, we provide the proofs of the regret upper bounds of the main text.
    In \Cref{appendix:upper_bound_instance_dependent}, we prove \Cref{theorem_upperbound_f_diverging}, which gives the  instance dependent bound.
We improve on this bound in \Cref{appendix_instance_dependent_bounded} when restricting to bounded span proxy, that is,  \Cref{theorem_upperbound_log}. 
    Lastly, in \Cref{appendix_worst_case_upper}, we establish the worst-case regret bound of \Cref{theorem_worst_case_sqrt}.

    \paragraph{Additional notations and terminology.}
    Throughout the section, we introduce a supplementary set of notations.
    Given an arm $\arm$, we write $(\tau^\arm_k)$ the enumeration of pulls of arm $\arm$, given by $\tau^\arm_1 := \inf \braces{t \ge 1: \Arm(t) = \arm}$ and $\tau^\arm_{k+1} := \inf \braces{t > \tau^\arm_k: \Arm(t) = \arm}$.
    We further write $\hat{\gain}_{\arm, n}$ the empirical estimate $\hat{\gain}_\arm (t)$ conditionally on $\visits_\arm (t) = n$.
    Formally, 
    \begin{equation}
        \hat{\gain}_{\arm, n}
        :=
        \frac 1n 
        \sum_{k=1}^n 
        \Reward(\tau^\arm_k+1)
        .
    \end{equation}
    Also, we will call \emph{episode $k$} the time segment $\braces*{t_k, \ldots, t_{k+1}-1}$ over which \algname{} plays the fixed arm $\Arm(t_k)$.
    The \emph{number of episodes} up to time $T$ is $K(T) := \sum_{k=1}^\infty \indicator{t_k \le T}$.

    \subsection{Proof of \texorpdfstring{\Cref{theorem_upperbound_f_diverging}}{Theorem 6}: Instance-dependent regret upper bound}
    \label[appendix]{appendix:upper_bound_instance_dependent}
    
    In this section, we provide a general upper bound for the instance dependent regret of \algname{}, when no information on the bias span of the underlying instance is provided.
    In that scenario, \algname{} is fed a span proxy function $\spanproxy$ that diverges to $+\infty$.
    The regret bound is sketched in   \Cref{theorem_upperbound_f_diverging} of the main text, which essentially claims that the performance scale as $\OH(\spanproxy(T)^2 \log(T))$. 
    We give a more detailed version of \Cref{theorem_upperbound_f_diverging} in \Cref{theorem_complete_upperbound_f_diverging}, incorporating second-order terms in the bound.
    
    \begin{theorem}[Instance-dependent regret of \algname]
    \label{theorem_complete_upperbound_f_diverging}
        Let $\spanproxy : \RR_+ \to \RR_+$ an increasing function diverging to infinity such that $\spanproxy(T) = \oh(\log(T)^{\frac 14})$. 
        Let $\delta : \NN \to \RR_+$ be a confidence function such that $\sum_{t=1}^\infty \delta(t) < \infty$ and $\log (\frac 1{\delta(t)}) = \Theta(\log(t))$.
        For all instance $\model \in \models$, $\regret(T) \equiv \regret(T; \model, \algname(\spanproxy))$ is bounded as
        \begin{align*}
            \regret(T)
            & = \OH \parens*{
                \sum_{\arm \notin \arms^*} \parens*{
                    \frac{
                        \spanproxy(T)^2
                        \log(T)
                    }{
                        \Delta_\arm
                    }
                    + \frac{\spanproxy(T)^4}{\Delta_\arm^3}
                }
                + \abs{\arms}
                \max_{\arm \in \arms}
                \vecspan(\bias_\arm) \log(T) 
            }
            & \parens*{\hspace{-.5em}\begin{tabular}{c} \scshape diverging \\ \scshape terms \end{tabular}\hspace{-.5em}}
            \\ 
            & \phantom{{} = {}}
            + \OH \parens*{
                \sum_{\arm \notin \arms^*} 
                \Delta_\arm \parens*{
                    (\beta_f)^2
                    + \strongdelayconstant \parens*{
                        \beta_f 
                        + \frac 1{\Delta_\arm^4}
                        + \parens*{
                            \frac{1 + \vecspan(\bias_\arm)}{\Delta_\arm}
                        }^2
                    }
                    + \parens*{
                        \frac{1 + \vecspan(\bias_\arm)}{\Delta_\arm}
                    }^4
                }
            }
            \; ,
            & \parens*{\hspace{-.5em}\begin{tabular}{c} \scshape constant \\ \scshape terms \end{tabular}\hspace{-.5em}}
        \end{align*}
        where $\arms^* := \arg \max_{\arm \in \arms} \gain_\arm$ is the set of optimal arms, $\Delta_\arm := \gain_* - \gain_\arm$ is the gain-gap and $\beta_f$ is the same as in \Cref{theorem_upperbound_f_diverging}.
        Moreover, if there is a single optimal arm, $\abs{\arms^*} = 1$, the bound of $\regret(T)$ is improved to
        \begin{align*}
            \regret(T)
            & = \OH \parens*{
                \sum_{\arm \notin \arms^*} \parens*{
                    \frac{
                        \spanproxy(T)^2
                        \log(T)
                    }{
                        \Delta_\arm
                    }
                    + \frac{\spanproxy(T)^4}{\Delta_\arm^3}
                }
                + \abs{\arms}
                \max_{\arm \in \arms}
                \vecspan(\bias_\arm) \log \parens*{
                    \spanproxy(T)^2 \log(T)
                } 
            }
            & \parens*{\hspace{-.5em}\begin{tabular}{c} \scshape diverging \\ \scshape terms \end{tabular}\hspace{-.5em}}
            \\ 
            & \phantom{{} = {}}
            + \OH \parens*{
                \sum_{\arm \notin \arms^*} 
                \Delta_\arm \parens*{
                    (\beta_f)^2
                    + \strongdelayconstant \parens*{
                        \beta_f 
                        + \frac 1{\Delta_\arm^4}
                        + \parens*{
                            \frac{1 + \vecspan(\bias_\arm)}{\Delta_\arm}
                        }^2
                    }
                    + \parens*{
                        \frac{1 + \vecspan(\bias_\arm)}{\Delta_\arm}
                    }^4
                }
            }
            \; .
            & \parens*{\hspace{-.5em}\begin{tabular}{c} \scshape constant \\ \scshape terms \end{tabular}\hspace{-.5em}}
        \end{align*}
    \end{theorem}

    \paragraph{About \Cref{theorem_upperbound_f_diverging}.}
    \Cref{theorem_upperbound_f_diverging} follows immediately by extracting the dominant diverging term and selecting the constant term $\sum_{\arm \notin \arms^*} \Delta_\arm (\beta_f)^2$.
    This term is independent of $\Delta_\arm$, and grows larger the slower is $\spanproxy$ diverging to infinity.
    In particular, $(\beta_f)^2$ can be significantly large if $\spanproxy$ is poorly initialized with respect to $\model$.

    \paragraph{Discussion.}
    Remarkably, we can provide better regret guarantees when the optimal arm is unique. 
    In cases where there are multiple optimal arms, the third order term of the regret upper bound scales with the number of arm switches (which is bounded by $\OH(\abs{\arms} \log(T))$, see \Cref{lemma_number_of_episodes}). 
    However, the dependency in the number of arm switches is pushed to the second order when the optimal arm is unique --- this is because when $\abs{\arms^*}=1$, the total number of switches is provably $\OH(\log(\spanproxy(T) \log(T)))$.

    \paragraph{About notations.}
    In the remaining of the section, we fix an instance $\model \in \models$, a span-proxy function $\spanproxy$ and the algorithm $\learner \equiv \algname(\spanproxy)$.
    Now that they are all fixed, we drop the dependency in all these quantities in notations.
    In particular, we write $\regret(T), \EE[-]$ and $\Pr(-)$ instead of $\regret(T; \model, \learner), \EE^{\model, \learner}[-]$ and $\Pr^{\model, \learner}(-)$. 

    \paragraph{Idea of the proof.}
    To bound the regret, we bound the pseudo-regret (\Cref{equation_pseudo_regret}), then convert back the pseudo-regret bound to a regret bound using \Cref{proposition_regret_and_pseudo_regret_appendix}.
    To bound the pseudo-regret, we upper-bound the number of pulls to every suboptimal arms via \Cref{appendix:lemma_order_visits_f_diverging} down-below, that incarnates most of the difficulties behind \Cref{theorem_complete_upperbound_f_diverging}.

    Let $\epsilon < \frac 12 \min_{\arm \notin \arms^*} (\gain_* - \gain_\arm)$ be an arbitrary (but small enough) constant. 

    \begin{lemma}[Number of visits, $\spanproxy$ diverging]
        \label[lemma]{appendix:lemma_order_visits_f_diverging}
        Consider a suboptimal arm $\arm \notin \arms^*$.
        Then $\EE[\visits_\arm (T)]$ is of order:
        \begin{equation}
            \frac{
                \spanproxy(T)^2
                \log(T)
            }{
                \parens{\gain_* - \gain_\arm - 2 \epsilon}^2
            }
            + \parens*{
                \frac{\spanproxy(T)}{\gain_* - \gain_\arm - 2 \epsilon}
            }^4
            + (\beta_f)^2
            + \strongdelayconstant \parens*{
                \beta_f 
                + \frac 1{\epsilon^4}
                + \parens*{
                    \frac{1 + \vecspan(\bias_\arm)}{\epsilon}
                }^2
            }
            + \parens*{
                \frac{1 + \vecspan(\bias_\arm)}{\epsilon}
            }^4
            .
        \end{equation}
    \end{lemma}

    \begin{proof}
        Denote $\spanbound := 1 + \max_{\arm \in \arms} \vecspan(\bias_\arm)$. 
        Fix $\epsilon > 0$ and $T \ge 1$, and introduce the function
        \begin{equation}
        \label{equation_visit_threshold}
            \chi_\epsilon (T)
            :=
            1 +
            \sup \braces*{
                n \ge 1
                :
                \sqrt{
                    \frac{\log((1+n)/\delta(T))}{n}
                }
                + \frac{\log_2 (1+n)}n
                \ge
                \frac{\gain_* - \gain_\arm - 2 \epsilon}{f(T)}
            }
            .
        \end{equation}
        We will upper bound this function in time. 
        
        We define the sequence of episodes \emph{relative to arm $\arm$} $(t_\ell^\arm)_{\ell \ge 1}$ as the episodes such that $\Arm(t_k) = \arm$, further clipped to $\braces{1, \ldots, T-1}$, i.e., 
        $
            t^\arm_1 
            := 
            \min \braces{\inf \braces{t_k : \Arm(t_k) = \arm}, T-1} 
            ~\mathrm{and}~
            t^\arm_{\ell+1} 
            := 
            \min \braces{\inf \braces*{t_k : \Arm(t_k) = \arm, t_k > t_\ell^\arm}, T-1}
            .
        $
        The number of visits to arm $\arm$ is decomposed as such.
        To determine why the algorithm pulls $\Arm(t) = \arm$, we look at what happens at the start of the current episode, i.e., at the unique time $t_\ell^\arm$ such that $t_\ell^\arm \le t \le t_\ell^\arm - 1$.
        We write,
        \begin{align} 
        \nonumber
            & \visits_\arm(T)
            =
            \sum_{t = 1}^{T-1}
            \indicator{\Arm(t)=\arm}
            \\ 
        \nonumber
            &  =
            \sum_{\ell = 1}^\infty
            \sum_{t = 1}^{\infty}
            \indicator{
                \substack{
                    \displaystyle
                    \Arm(t) = \arm
                    \\
                    \displaystyle
                    t_{\ell+1}^\arm \le t < t^\arm_{\ell+1}   
                }
            }
            \\ & = 
        \label{equation_visit_decomposition_0}
            \sum_{\ell = 1}^\infty
            \sum_{t = 1}^{\infty}
            \parens*{
                \smash[b]{
                    \underbrace{
                        \indicator{
                            \substack{
                                \displaystyle
                                \Arm(t) = \arm
                                \\
                                \displaystyle
                                t^\arm_{\ell} \le t < t^\arm_{\ell+1}
                                \\
                                \displaystyle
                                \hat \gain_{\arm}(t_\ell^\arm) 
                                > 
                                \gain_{\arm} + \epsilon
                            }
                        }
                    }_{
                        \substack{
                            \textsc{Empirical} \\ \textsc{Errors}
                            \\[.5em]
                            \displaystyle (\mathrm{I})
                        }
                    }
                }
                + \smash[b]{
                    \underbrace{
                        \indicator{
                            \substack{
                                \displaystyle
                                \Arm(t) = \arm
                                \\
                                \displaystyle
                                t^\arm_{\ell} \le t < t^\arm_{\ell+1}
                                \\
                                \displaystyle
                                \gain_{\arm^*} 
                                > 
                                \tilde{\gain}_{\arm^*}(t^\arm_\ell) + \epsilon
                            }
                        }
                    }_{
                        \substack{
                            \textsc{Optimism} \\ \textsc{Failures}
                            \\[.5em]
                            \displaystyle (\mathrm{II})
                        }
                    }
                }
                + \smash[b]{
                    \underbrace{
                        \indicator{
                            \substack{
                                \displaystyle
                                \Arm(t) = \arm
                                \\
                                \displaystyle
                                t^\arm_{\ell} \le t < t^\arm_{\ell+1}
                                \\
                                \displaystyle
                                \visits_\arm (t^\arm_\ell) \le \chi_\epsilon(T)
                            }
                        }
                    }_{
                        \substack{
                            \textsc{Optimistic} \\ \textsc{Pulls}
                            \\[.5em]
                            \displaystyle (\mathrm{III})
                        }
                    }
                }
                + \smash[b]{
                    \underbrace{
                        \indicator{
                            \substack{
                                \displaystyle
                                \Arm(t) = \arm
                                \\
                                \displaystyle
                                t^\arm_{\ell} \le t < t^\arm_{\ell+1}
                                \\
                                \displaystyle
                                \visits_\arm (t^\arm_\ell) > \chi_\epsilon(T)
                                \\
                                \displaystyle
                                \hat \gain_{\arm}(t_\ell^\arm)  
                                \le
                                \gain_{\arm} + \epsilon
                                \\
                                \displaystyle
                                \gain_{\arm^*} 
                                \le
                                \tilde{\gain}_{\arm^*}(t^\arm_\ell) + \epsilon
                            }
                        }
                    }_{
                        \substack{
                            \textsc{Late Bloomers}
                            \\[.5em]
                            \displaystyle (\mathrm{IV})
                        }
                    }
                }
            }
        \end{align}
        \vspace{1em} 
        
        In the above, we state that arm $\Arm(t) = \arm$ may be pulled for four reasons. 
        Either (I) the empirical estimate of arm~$\arm$ is far away from its true value $\gain_\arm$;
        Or (II) the suboptimal arm is pulled because the index of the optimal arm is not an upper bound of $\gain^*$, meaning that optimism failed;
        Or (III) the number of pulls is still small, so that the optimism bonus $\xi_\arm (t, \delta(t))$ (see \Cref{lemma_estimation_error}) is still large;
        Or finally (IV) optimism didn't fail, the empirical estimate of arm $\arm$ is well concentrated and the optimism bonus is small. 
        
        In what follows, we show that (I) accounts for $\OH(1)$ pulls in expectation, (II) for $\OH(1)$, (III) is of order $\OH(\spanproxy(T) \log(T))$ and dominates all the other terms, and (IV) is $0$ by design of \algname{} and of $\chi_\epsilon (T)$, see \Cref{equation_visit_threshold}.

        \paragraph{Empirical estimator failure (I).}
        We show that this term is bounded.
        \begin{align}
        \nonumber
            \EE \brackets*{
                \sum_{\ell=1}^{\infty}
                \sum_{t=1}^\infty
                (\mathrm{I})
            }
            & = 
            \EE \brackets*{
                \sum_{\ell=1}^{\infty}
                \sum_{t=1}^\infty
                \indicator{
                    \substack{
                        \displaystyle
                        \Arm(t) = \arm
                        \\
                        \displaystyle
                        t_{\ell}^\arm \le t \le t_{\ell+1}^\arm
                        \\
                        \displaystyle
                        \hat{\gain}_\arm (t_\ell^\arm)
                        > 
                        \gain_\arm + \epsilon
                    }
                }
            }
            \\
        \nonumber
            & \overset{\eqnum{1}}=
            \EE \brackets*{
                \sum_{\ell=1}^{\infty}
                \parens*{
                    \visits_{\arm}(t_{\ell+1}^\arm) 
                    - \visits_{\arm}(t_\ell^\arm)
                }
                \indicator{
                    \hat{\gain}_\arm (t_\ell^\arm)
                    > 
                    \gain_\arm + \epsilon
                }
            }
            \\
        \label{equation_upperbound_f_diverging_I_1}
            & \overset{\eqnum{2}}=
            \EE \brackets*{
                \sum_{\ell=1}^{\infty}
                \visits_{\arm}(t_\ell^\arm)
                \indicator{
                    \hat{\gain}_\arm (t_\ell^\arm)
                    > 
                    \gain_\arm + \epsilon
                }
            }
            +
            \EE \brackets*{
                \sum_{\ell=1}^{\infty}
                \parens*{
                    \visits_{\arm}(t_{\ell+1}^\arm)
                    - 2 \visits_{\arm}(t_{\ell}^\arm)
                }
                \indicator{
                    \hat{\gain}_\arm (t_\ell^\arm)
                    > 
                    \gain_\arm + \epsilon
                }
            },
        \end{align}
        where 
        \texteqnum{1} uses that only arm $\arm$ is pulled during the episode $\braces{t_\ell^\arm, \ldots, t_{\ell+1}^\arm-1}$;
        and
        \texteqnum{2} splits the sum of \texteqnum{1} according to the doubling trick.
        Both terms in \Cref{equation_upperbound_f_diverging_I_1} are bounded separately. 

        For the left term of \Cref{equation_upperbound_f_diverging_I_1}, we proceed as follows:
        \begin{align}
        \nonumber
            \EE \brackets*{
                \sum_{\ell=1}^{\infty}
                \visits_{\arm}(t_\ell^\arm)
                \indicator{
                    \hat{\gain}_\arm (t_\ell^\arm)
                    > 
                    \gain_\arm + \epsilon
                }
            }
            & \overset{\eqnum{1}}\le
            \EE \brackets*{
                \sum_{n=1}^\infty
                n 
                \cdot \indicator{
                    \hat{\gain}_{\arm, n}
                    > 
                    \gain_\arm + \epsilon
                }
            }
            \\
        \nonumber
            & =
            \sum_{n=1}^\infty
            n
            \cdot \Pr \parens*{
                \hat{\gain}_{\arm, n}
                > 
                \gain_\arm 
                + \epsilon
            }
            \\
            & \overset{\eqnum{2}}\le
            \sum_{n=1}^\infty
            n
            \cdot \exp \parens*{
                - 
                \frac{
                    n \epsilon^2
                    - 2 \epsilon \vecspan(\bias_\arm) (1 + \log_2 (1 + n))
                }{
                    (1 + \vecspan(\bias_\arm))^2
                }
            }
            = \OH \parens*{
                \frac 1{\epsilon^4}
            },
        \end{align}
        where
        \texteqnum{1} uses that the LHS is a subset of the RHS; 
        and
        \texteqnum{2} invokes \Cref{lemma_gain_concentration_per_sample}.

        For the right term of \Cref{equation_upperbound_f_diverging_I_1}, we have  
        \begin{align*}
            \EE \brackets*{
                \sum_{\ell=1}^{\infty}
                \parens*{
                    \visits_{\arm}(t_{\ell+1}^\arm)
                    - 2 \visits_{\arm}(t_{\ell}^\arm)
                }
                \indicator{
                    \hat{\gain}_\arm (t_\ell^\arm)
                    > 
                    \gain_\arm + \epsilon
                }
            }
            & \overset{\eqnum{1}}{\le}
            \EE \brackets*{
                \sum_{\ell=1}^{\infty}
                \inf \braces*{
                    \substack{
                        \displaystyle
                        \tau_i 
                        - (t_\ell^\arm + \visits_\arm (t_\ell^\arm))
                        \\
                        : ~
                        \tau_i 
                        \ge t_\ell^\arm + \visits_\arm (t_\ell^\arm)
                    }
                }
                \indicator{
                    \hat{\gain}_\arm (t_\ell^\arm)
                    > 
                    \gain_\arm + \epsilon
                }
            }
            \\
            & \overset{\eqnum{2}}\le \strongdelayconstant
            \EE \brackets*{
                \sum_{\ell=1}^{\infty}
                \indicator{
                    \hat{\gain}_\arm (t_\ell^\arm)
                    > 
                    \gain_\arm + \epsilon
                }
            }
            \\
            & \le
            \strongdelayconstant
            \EE \brackets*{
                \sum_{n=1}^\infty
                \indicator{
                    \hat{\gain}_{\arm, n}
                    > 
                    \gain_\arm + \epsilon
                }
            }
            \\
            & =
            \strongdelayconstant
            \sum_{n=1}^\infty
            \Pr \parens*{
                \hat{\gain}_{\arm, n}
                > 
                \gain_\arm 
                + \epsilon
            }
            \\
            & \overset{\eqnum{3}}\le
            \strongdelayconstant
            \sum_{n=1}^\infty
            \exp \parens*{
                - 
                \frac{
                    n \epsilon^2
                    - 2 \epsilon \vecspan(\bias_\arm) (1 + \log_2 (1 + n))
                }{
                    (1 + \vecspan(\bias_\arm))^2
                }
            }
            \\
            &= \OH \parens*{
                \frac {\strongdelayconstant}{\epsilon^2}
            },
        \end{align*}
        where 
        \texteqnum{1} uses that $t_{\ell+1}^\arm$ corresponds to the first decision epoch after which the Doubling Trick does not force the continuation of the episode anymore, which is at time $t_\ell^\arm + \visits_\arm (t_\ell^\arm)$;
        \texteqnum{2} conditions on that time $t_\ell^\arm + \visits_\arm (t_\ell^\arm)$, which is a stopping-time, and invokes the Tower Property;
        and finally,
        \texteqnum{3} follows from \Cref{lemma_gain_concentration_per_sample}. 

        We conclude that 
        \begin{equation}
            \mathrm{(I)}
            = 
            \OH \parens*{
                \frac{\strongdelayconstant}{\epsilon^2}
                + \frac 1{\epsilon^4}
            }
            \: .
        \end{equation}
            
        \paragraph{Optimism failure (II).}
        We now aim to bound the term (II) which correspond to the case where the optimistic estimator of the optimal arm $\tilde \gain_{\arm^*}$ turns out to be smaller than the true gain $\gain_{\arm^*} $.   
        \begin{align}
        \nonumber
            \EE &\brackets*{
                \sum_{\ell = 1}^\infty
                \sum_{t=1}^{\infty}
                (\mathrm{II})
            }
            \\
            & \qquad = 
            \EE \brackets*{
                \sum_{\ell \ge 1}
                \parens*{
                    \visits_\arm(t^\arm_{\ell+1}) - \visits_\arm(t^\arm_{\ell}) 
                }
                \indicator{
                    \displaystyle
                    \gain_{\arm^*} 
                    > 
                    \tilde \gain_{\arm^*}(t^\arm_\ell)
                    + \epsilon
                }
            }
            \\
        \label{equation_upperbound_f_diverging_II}
            & \qquad =
            \underbrace{
                \EE \brackets*{
                    \sum_{\ell \ge 1}
                    \visits_\arm(t^\arm_{\ell})
                    \indicator{
                        \displaystyle
                        \gain_{\arm^*} 
                        > 
                        \tilde \gain_{\arm^*} (t^\arm_\ell)
                        + \epsilon
                    }
                }
            }_{\displaystyle (\mathrm{II.1})}
            + 
            \underbrace{
                \EE \brackets*{
                    \sum_{\ell \ge 1}
                    \parens*{
                        \visits_\arm(t^\arm_{\ell+1})
                        - 
                        2 \visits_\arm(t^\arm_{\ell})
                    }
                    \indicator{
                        \displaystyle
                        \gain_{\arm^*} 
                        > 
                        \tilde \gain_{\arm^*} (t^\arm_\ell)
                        + \epsilon
                    }
                }  
            }_{\displaystyle (\mathrm{II.2})}
            .
        \end{align}
        We bound those two terms (II.1) and (II.2) separately. 
        We will introduce the following magical constant:
        \begin{equation}
        \label{equation_beta_spanproxy}
            \beta_\spanproxy
            :=
            1 +
            \sup \braces*{
                t \ge 1
                :
                \spanproxy(t) 
                \le 2 \max_{\arm \in \arms} \braces*{
                    1+\vecspan(\bias_\arm),
                    3 \vecspan(\bias_\arm)
                }
            }
            .
        \end{equation}
        Note that $\beta_\spanproxy < \infty$, since $\spanproxy \to \infty$.
        
        For the first term (II.1) in \Cref{equation_upperbound_f_diverging_II}, we write
        \begin{align*}
            (\mathrm{II.1}) 
            & =
            \EE \brackets*{
                \sum_{\ell=1}^\infty
                \sum_{n=1}^{+\infty}
                n \cdot 
                \indicator{
                    \gain_{\arm^*} 
                    > 
                    \tilde \gain_{\arm^*} (t^\arm_\ell)
                    + \epsilon,
                    \visits_{\arm}(t_\ell^\arm) = n
                }
            }
            \\ & \overset{\eqnum{1}}\le
            \EE \brackets*{
                \sum_{t=1}^\infty
                \sum_{n=1}^{+\infty}
                n \cdot 
                \indicator{
                    \gain_{\arm^*} 
                    > 
                    \hat{\gain}_{\arm^*}
                    + \epsilon
                    + \spanproxy(t) \sqrt{
                        \frac 1n
                        \log \parens*{\frac{1 + n}{\delta(t)}}
                    }
                    ,
                    \visits_{\arm}(t) = n
                }
            }
            \\ & \overset{\eqnum{2}}\le
            \EE \brackets*{
                \sum_{t=1}^\infty
                \sum_{n=1}^{t}
                n \cdot 
                \indicator{
                    \gain_{\arm^*} 
                    > 
                    \hat{\gain}_{\arm^*}
                    + \epsilon
                    + \spanproxy(t) \sqrt{
                        \frac 1n
                        \log \parens*{\frac{1 + n}{\delta(t)}}
                    }
                }
            }
        \end{align*}
        where 
        \texteqnum{1} follows by union bound over all the values that $t_\ell^\arm$ can take;
        and
        \texteqnum{2} uses that $\visits_\arm(t) \le n$.
        
        We split the term in two depending on the value of $\spanproxy(t)$. 
        The regret contribution before time $\beta_\spanproxy$ is trivially bounded by $\beta_\spanproxy$ and is thus $\OH(1)$. Consequently, we can restrict ourselves to evaluating the regret past $\beta_\spanproxy$.
        \begin{align}
        \nonumber
            (\mathrm{II.1})
            & \overset{\eqnum{3}}{\le} 
            \sum_{t=1}^{\beta_f}
            n
            + 
            \sum_{t = \beta_f}^\infty
            \sum_{n = 1}^{t}
            n
            \cdot
            \exp
                \braces*{
                    - 
                    \frac {
                        n 
                        \cdot
                        \parens*{
                            \parens*{
                                \epsilon
                                +
                                f(t)
                                \sqrt{
                                    \frac {\log((1+n)/\delta(t))}n 
                                }
                                - 
                                \frac{
                                    3\vecspan(\bias_\arm)  
                                }{\sqrt{n}}
                            }^2
                            -9\vecspan(\bias_\arm)^2
                        }
                    }{
                        (1 + \vecspan(\bias_\arm))^2
                    }
                }
            \\
        \nonumber
            & \overset{\eqnum{4}}{\le} 
            (\beta_f)^2
            +
            \sum_{t = \beta_f}^\infty
            \sum_{n = 1}^{t}
            n \cdot
            e^9
            \exp \braces*{
                - 
                \frac n{
                    (1 + \vecspan(\bias_\arm))^2
                }
                \parens*{
                    \epsilon
                    +
                    \frac12
                    f(t)
                    \sqrt{
                        \frac {\log((1+n)/\delta(t))}n 
                    }
                }^2
            } 
            \\ 
        \nonumber
            & \le  
            (\beta_f)^2
            +
            \sum_{t = \beta_f}^\infty
            \cdot 
            \sum_{n = 1}^{t}
            n \cdot
            e^9
            \exp \braces*{
                - 
                \frac{
                    n \epsilon^2
                    + 
                    \frac14 f(t)^2 \log \parens*{\frac{1}{\delta(t)}}
                }{
                    (1 + \vecspan(\bias_\arm))^2
                }
            } 
            \\ 
        \nonumber
            & \overset{\eqnum{5}}{\le}
            (\beta_f)^2
            +
            \sum_{t = \beta_f}^\infty
            \sum_{n = 1}^{t}
            n \cdot 
            e^9
            \delta(t)
            \exp \braces*{
                -\frac{n \epsilon^2}{(1+\vecspan(\bias_\arm))^2} 
            }
            \\
        \label{equation_upperbound_f_diverging_II_1}
            & \le 
            (\beta_f)^2
            + 
            \OH \parens*{
                \parens*{\frac{1 + \vecspan(\bias_\arm)}{\epsilon}}^4
                \sum_{t=1}^\infty
                \delta(t)
            },
        \end{align}
        where 
        \texteqnum{3} follows from \Cref{lemma_gain_concentration_per_sample}
        inequalities;
        and 
        \texteqnum{4} and \texteqnum{5} stem from the definition of $\beta_\spanproxy$ in \Cref{equation_beta_spanproxy}, as for $t \ge \beta_\spanproxy$, we have $\frac 12 \spanproxy(t) \sqrt{\frac 1n\log((1+n)/\delta(t))} \ge \frac{3 \vecspan(\bias_\arm)}{\sqrt{n}}$ and $\frac 14 \spanproxy(t)^2 \ge (1 + \vecspan(\bias_\arm))^2$.

        We continue with the second term (II.2) from \Cref{equation_upperbound_f_diverging_II}.
        We have
        \begin{align*}
            (\mathrm{II.2})
            &\overset{\eqnum{1}}{\le}
            \strongdelayconstant
            \EE \brackets*{
                \sum_{\ell = 1 }^\infty
                \indicator{
                    \displaystyle
                    \gain_{\arm^*} 
                    > 
                    \tilde \gain_{\arm^*} (t^\arm_\ell)
                    + \epsilon
                }
            }
            \\
            &\overset{\eqnum{2}}{\le}
            \strongdelayconstant
            \EE \brackets*{
                \sum_{t = 1 }^\infty
                \sum_{n = 1 }^t
                \indicator{
                    \displaystyle
                    \gain_{\arm^*} 
                    > 
                    \hat \gain_{\arm^*,n}
                    + \epsilon
                    +
                    f(t)
                    \sqrt{
                        \frac 1n \log 
                        \parens*{
                            \frac{1+n}{\delta(t)}
                        }
                    }
                }
            }
            .
        \end{align*}
        Again, we decompose depending on the value of $f(t)$. This derivation is similar to the one of the previous paragraph.
        \begin{align}
        \nonumber
            (\mathrm{II.2})
            & \overset{\eqnum{3}}{\le} 
            \strongdelayconstant
            \beta_f
            +
            \strongdelayconstant
            \sum_{t = \beta_f}^\infty
            \sum_{n = 1}^{t}
            \exp
            \braces*{
                - 
                \frac {
                    n
                    \cdot
                    \parens*{
                        \parens*{
                            \epsilon
                            +
                            f(t)
                            \sqrt{
                                \frac {\log((1+n)/\delta(t))}n 
                            }
                            - 
                            \frac{
                                3\vecspan(\bias_\arm)  
                            }{\sqrt{n}}
                        }^2
                        -9\vecspan(\bias_\arm)^2
                    }
                }{
                    (1 + \vecspan(\bias_\arm))^2
                }
            }
            \\
        \nonumber
            & \overset{\eqnum{4}}{\le} 
            \strongdelayconstant
            \beta_f 
            +
            \strongdelayconstant
            \sum_{t = \beta_f}^\infty
            \sum_{n = 1}^{t}
            \ee^9
            \exp \braces*{
                -
                \frac n{
                    (1 + \vecspan(\bias_\arm))^2
                }
                \parens*{
                    \epsilon
                    +
                    \frac12
                    f(t)
                    \sqrt{
                        \frac {\log((1+n)/\delta(t))}n 
                    }
                }^2
            } 
            \\ 
        \nonumber
            & \overset{\eqnum{5}}{\le}  
            \strongdelayconstant
            \beta_f 
            +
            \strongdelayconstant
            \sum_{t = \beta_f}^\infty
            \sum_{n = 1}^{t}
            \ee^9
            \exp \braces*{
                - 
                \frac{
                    n \epsilon^2
                    + 
                    \frac14 f(t)^2 \log \parens*{\frac{1}{\delta(t)}}
                }{
                    (1 + \vecspan(\bias_\arm))^2
                }
            } 
            \\ 
        \nonumber
            & \overset{\eqnum{6}}{\le}
            \strongdelayconstant
            \beta_f 
            +
            \strongdelayconstant
            \sum_{t = \beta_f}^\infty
            \sum_{n = 1}^{t}
            \ee^9
            \delta(t)
            \exp \braces*{
                -\frac{n \epsilon^2}{(1+\vecspan(\bias_\arm))^2} 
            }
            \\
        \label{equation_upperbound_f_diverging_II_2}
            & \le 
            \strongdelayconstant
            \beta_f 
            + 
            \OH \parens*{
                \strongdelayconstant
                \parens*{\frac{1 + \vecspan(\bias_\arm)}{\epsilon}}^2
                \sum_{t=1}^\infty
                \delta(t)
            },
        \end{align}
        where, likewise, \texteqnum{3} follows from \Cref{lemma_gain_concentration_per_sample} and \texteqnum{4}, \texteqnum{5} stem from the definition of $\beta_f$.
        We put the bound of (II) together by combining \Cref{equation_upperbound_f_diverging_II_1,equation_upperbound_f_diverging_II_2} into \Cref{equation_upperbound_f_diverging_II}.
        We conclude that
        \begin{equation}
            \mathrm{(II)}
            = 
            \OH \parens*{
                (\strongdelayconstant + \beta_f) \beta_f 
                + \parens*{
                    \strongdelayconstant \parens*{
                        \frac{1 + \vecspan(\bias_\arm)}{\epsilon}
                    }^2
                    + \parens*{\frac{1 + \vecspan(\bias_\arm)}{\epsilon}}^4
                }
                \sum_{t=1} \delta(t)
            }
            \; .
        \end{equation}
    
        \paragraph{Optimistic pulls (III).}
        We show that this term is dominant, and accounts for a total of $2\chi_\epsilon (T) + \strongdelayconstant$ pulls in expectation.
        Given $t \ge 1$ such that $\Arm(t) = \arm$, we write $\ell(t)$ the current episodes $t$ falls in, i.e., the unique $\ell \ge 1$ such that $t_\ell^\arm \le t < t_{\ell+1}^\arm$.
        Introduce the stopping time 
        \begin{equation}
        \label{equation_upperbound_f_diverging_III_1}
            \tau_\chi 
            := 
            \inf \braces*{
                t \ge 1
                :
                \substack{
                    \displaystyle
                    \Arm(t) = \arm,
                    \\
                    \displaystyle
                    \visits_\arm (t^\arm_{\ell(t)}) \le \chi_\epsilon (T),
                    \quad\visits_\arm (t) > \chi_\epsilon (T),
                    \\
                    \displaystyle
                    \visits_\arm (t) \ge 2 \visits_\arm (t_{\ell(t)}^\arm)
                }
            }.
        \end{equation}
        That is, $\tau_\chi$ is the first time at which $\Arm(t) = \arm$ is pulled, and such that
        (1) the current started below the threshold with $\visits_\arm (t_{\ell(t)}^\arm) \le \chi_\epsilon (T)$,
        (2) the current visited count has exceeded the threshold with $\visits_\arm (t) \ge \chi_\epsilon (T)$, and
        (3) the continuation of the current episode $\ell(t)$ is not triggered by the Doubling Trick since $\visits_\arm (t) \ge 2 \visits_\arm (t_{\ell(t)}^\arm)$. 

        We then bound (III) as follows.
        \begin{align*}
            \EE \brackets*{
                \sum_{\ell = 1}^\infty
                \sum_{t=1}^\infty
                \indicator{
                    \substack{
                        \displaystyle
                        \Arm(t) = \arm
                        \\
                        \displaystyle
                        t^\arm_{\ell} \le t < t^\arm_{\ell+1}
                        \\
                        \displaystyle
                        \visits_\arm (t^\arm_\ell) \le \chi_\epsilon(T)
                    }
                }
            }
            & \overset{\eqnum{1}}=
            \EE \brackets*{
                \sum_{\ell = 1}^\infty
                \sum_{t=1}^\infty
                \indicator{
                    \substack{
                        \displaystyle
                        \Arm(t) = \arm
                        \\
                        \displaystyle
                        t^\arm_{\ell} \le t < t^\arm_{\ell+1}
                        \\
                        \displaystyle
                        \visits_\arm (t^\arm_\ell) \le \chi_\epsilon(T)
                        \\
                        \displaystyle
                        t < \tau_\chi
                    }
                }
            }
            +
            \EE \brackets*{
                \sum_{\ell = 1}^\infty
                \sum_{t=1}^\infty
                \indicator{
                    \substack{
                        \displaystyle
                        \Arm(t) = \arm
                        \\
                        \displaystyle
                        t^\arm_{\ell} \le t < t^\arm_{\ell+1}
                        \\
                        \displaystyle
                        \visits_\arm (t^\arm_\ell) \le \chi_\epsilon(T)
                        \\
                        \displaystyle
                        t \ge \tau_\chi
                    }
                }
            } 
            \\
            & \overset{\eqnum{2}}\le
            \EE [\tau_\chi]
            + \EE \brackets*{
                \sum_{t=1}^\infty
                \indicator{
                    \substack{
                        \displaystyle
                        \Arm(t) = \arm
                        \\
                        \displaystyle
                        t^\arm_{\ell(\tau_\chi)} \le t < t^\arm_{\ell(\tau_\chi)+1}
                        \\
                        \displaystyle
                        t \ge \tau_\chi
                    }
                }
            }
            \\
            & \overset{\eqnum{3}}\le 
            2 \chi_\epsilon (T)
            + \EE \brackets*{
                \inf \braces*{
                    \tau_i - \tau_\chi
                    :
                    \tau_i \ge \tau_\chi 
                }
            }
            \overset{\eqnum{4}}\le
            2 \chi_\epsilon (T) + \strongdelayconstant,
        \end{align*}
        where 
        \texteqnum{1} splits the sum according to the stopping time $\tau_\chi$ introduces by \Cref{equation_upperbound_f_diverging_III_1};
        \texteqnum{2} bounds the first term by $\tau_\chi$ by using that every pull of $\arm$ makes its visit counts $\visits_\arm$ increase by $1$;
        \texteqnum{2} bounds the second term using the observation that, by definition of $\tau_\chi$, there is at most one episode $\ell \ge 1$ such that $\visits_\arm (t_\ell^\arm) \le \chi_\epsilon (T)$ that contains some $t \ge \tau_\chi$;
        \texteqnum{3} bounds the first term using that $\visits_\arm (\tau_\chi) \le 2 x $ a.s., because if $\tau_\chi \ge 2 x + 1$, one checks that $\tau_\chi - 1$ also satisfies \Cref{equation_upperbound_f_diverging_III_1};
        \texteqnum{3} bounds the second term by observing that, by definition of $\tau_\chi$, the current episode $\ell(\tau_\chi)$ has passed the doubling trick rule at time $\tau_\chi$, so that the episode will stop at the next decision epoch $\tau_i$;
        and finally, 
        \texteqnum{4} bounds the second term by invoking \Cref{assumption_strong_decision_epochs}.

        \paragraph{Late bloomers (IV).}
        This term is almost surely $0$.
        \\
        Indeed, when term $\mathrm{(IV)} = 1$ we have 
        \begin{equation*}
            \begin{aligned}
                \tilde \gain_{\arm^*} (t^\arm_\ell) - \tilde \gain_\arm (t^\arm_l) 
                &\overset{\eqnum{1}}{\ge}
                \gain_{\arm^*} - \epsilon - \parens*{\hat \gain_\arm (t^\arm_\ell) + \xi(t^\arm_\ell, \delta(t^\arm_\ell) }
                \\
                &\overset{\eqnum{2}}{\ge}
                \gain_{\arm^*}-\gain_\arm -2 \epsilon - \xi\parens*{t^\arm_\ell,\delta(t^\arm_\ell)}
                \overset{\eqnum{3}}{>}0,
            \end{aligned}
        \end{equation*}
        where \texteqnum{1} follow from
        $ 
            \gain_{\arm^*} 
            \le
            \tilde{\gain}_{\arm^*}(t^\arm_\ell) + \epsilon
        $; 
        \texteqnum{2} from 
        $
            \hat \gain_{\arm}(t_\ell^\arm)  
            \le
            \gain_{\arm} + \epsilon
        $;
        and \texteqnum{3} from $\visits_\arm (t^\arm_\ell) > \chi_\epsilon(T)$.
        Thus, by definition, the algorithm will not select arm $\arm$ at time $t^\arm_\ell$. 

        \paragraph{Bounding the visit threshold \texorpdfstring{$\chi_\epsilon$}{chi}.}
        Last but not least, we must upper-bound $\chi_\epsilon (T)$ from \Cref{equation_visit_threshold}, that appears in the bound of (III).
        Denote $n \equiv \chi_\epsilon (T) - 1$.
        Then $n$ is such that 
        \begin{equation*}
            \sqrt{
                \frac{\log((1+n)/\delta(T))}{n}
            }
            + \frac{\log_2 (1+n)}n
            \ge
            \frac{\gain_* - \gain_\arm - 2 \epsilon}{f(T)}
            .
        \end{equation*}
        From straight-forward algebra, we find that $n$ must satisfy one of the three conditions below:
        \begin{equation*}
            \frac{\log(1+n)}n
            \ge 
            \frac18
            \parens*{
                \frac{\gain_* - \gain_\arm - 2 \epsilon}{\spanproxy(T)}
            }^2,
            \text{~or~~}
            \frac{\log\parens*{\frac 1{\delta(T)}}}n
            \ge 
            \frac18
            \parens*{
                \frac{\gain_* - \gain_\arm - 2 \epsilon}{\spanproxy(T)}
            }^2,
            \text{~or~~}
            \frac{\log_2(1+n)}n
            \ge 
            \frac12
            \frac{\gain_* - \gain_\arm - 2 \epsilon}{\spanproxy(T)}.
        \end{equation*}
        Using $\log(1+x) \le \sqrt{x}$, we find that 
        \begin{align}
        \notag
            \chi_\epsilon (T)
            & \le 
            1 + \max \braces*{
                \parens*{
                    \frac{
                        2 \sqrt{2} \spanproxy(T)
                    }{
                        \gain_* - \gain_\arm - 2 \epsilon
                    }
                }^4,
                \frac{
                    8 \spanproxy(T)^2
                    \log\parens*{\frac 1{\delta(T)}}
                }{
                    \parens{\gain_* - \gain_\arm - 2 \epsilon}^2
                },
                \parens*{
                    \frac{
                        2 \spanproxy(T)
                    }{
                        \gain_* - \gain_\arm - 2 \epsilon
                    }
                }^2
            } 
            \\ & =
            \OH \parens*{
                \frac{
                    \spanproxy(T)^2
                    \log(T)
                }{
                    \parens{\gain_* - \gain_\arm - 2 \epsilon}^2
                }
                + \parens*{
                    \frac{\spanproxy(T)}{\gain_* - \gain_\arm - 2 \epsilon}
                }^4
            }
            \; .
        \end{align}
        Note that because $\log(\frac 1{\delta(t)}) = \Theta(\log(t))$ and $\spanproxy(T)^4 = \oh(\log(T))$, the term in $\spanproxy(T) \log(\frac 1{\delta(T)})$ is dominant.

        \paragraph{Conclusion.}
        Combining everything, we find that
        \begin{gather*}
            \EE \brackets*{
                \visits_\arm (T)
            }
            \le 
            2 \chi_\epsilon(T)
            + \OH \parens*{
                (\beta_f)^2 
                + \strongdelayconstant \parens*{
                    \beta_f + \frac 1{\epsilon^4}
                }
                + \parens*{
                    \strongdelayconstant \parens*{
                        \frac{1 + \vecspan(\bias_\arm)}{\epsilon}
                    }^2
                    + \parens*{
                        \frac{1 + \vecspan(\bias_\arm)}{\epsilon}
                    }^4
                }
                \sum_{t=1}^\infty \delta(t)
            }
            \\ \overset{\eqnum{1}}=
            \OH \parens*{
                \frac{
                    \spanproxy(T)^2
                    \log(T)
                }{
                    \parens{\gain_* - \gain_\arm - 2 \epsilon}^2
                }
                + \parens*{
                    \frac{\spanproxy(T)}{\gain_* - \gain_\arm - 2 \epsilon}
                }^4
                + (\beta_f)^2
                + \strongdelayconstant \parens*{
                    \beta_f 
                    + \frac 1{\epsilon^4}
                    + \parens*{
                        \frac{1 + \vecspan(\bias_\arm)}{\epsilon}
                    }^2
                }
                + \parens*{
                    \frac{1 + \vecspan(\bias_\arm)}{\epsilon}
                }^4
            }
        \end{gather*}
        Hence the result.
    \end{proof}

    \begin{proof}[Proof of \Cref{theorem_complete_upperbound_f_diverging}]
        For $\arm \in \arms$ such that $\gain_\arm < \gain_*$, denote $\Delta_\arm := \gain_* - \gain_\arm$ the gain-gap.
        
        We apply \Cref{appendix:lemma_order_visits_f_diverging} for every sub-optimal arm independently, with $\epsilon_\arm := \frac 14 \Delta_\arm$ to bound $\EE[\visits_\arm (T)]$. 
        By invoking that $\EE[\pseudoregret(T)] = \sum_{\arm \notin \arms^*} \EE[\visits_\arm (T)] \Delta_\arm$,  we obtain that the expected pseudo-regret is bounded as:
        \begin{gather}
            \sum_{\arm \notin \arms^*} \parens*{
                \frac{
                    \spanproxy(T)^2
                    \log(T)
                }{
                    \Delta_\arm
                }
                + \frac{\spanproxy(T)^4}{\Delta_\arm^3}
                + (\beta_f)^2 \Delta_\arm
                + \strongdelayconstant \Delta_\arm \parens*{
                    \beta_f 
                    + \frac 1{\Delta_\arm^4}
                    + \parens*{
                        \frac{1 + \vecspan(\bias_\arm)}{\Delta_\arm}
                    }^2
                }
                + \parens*{
                    \frac{1 + \vecspan(\bias_\arm)}{\Delta_\arm}
                }^4
            }
            \; .
        \end{gather}

        We are left to relate the regret to the pseudo-regret.
        Because \algname{} relies on the Doubling Trick to manage episodes, we see that it is $\log_2$-rarely switching. 
        Therefore, we can relate the regret to the pseudo-regret up to the number of switches via \Cref{lemma_regret_and_switches}.
        Then, we bound the number of switches via the number of episodes with \Cref{lemma_number_of_episodes}, and obtain that, in general,
        \begin{align*}
            \regret(T)
            & = \OH \parens*{
                \sum_{\arm \notin \arms^*} \parens*{
                    \frac{
                        \spanproxy(T)^2
                        \log(T)
                    }{
                        \Delta_\arm
                    }
                    + \frac{\spanproxy(T)^4}{\Delta_\arm^3}
                }
                + \abs{\arms}
                \max_{\arm \in \arms}
                \vecspan(\bias_\arm) \log(T) 
            }
            & \parens*{\hspace{-.5em}\begin{tabular}{c} \scshape diverging \\ \scshape terms \end{tabular}\hspace{-.5em}}
            \\ 
            & \phantom{{} = {}}
            + \OH \parens*{
                \sum_{\arm \notin \arms^*} 
                \Delta_\arm \parens*{
                    (\beta_f)^2
                    + \strongdelayconstant \parens*{
                        \beta_f 
                        + \frac 1{\Delta_\arm^4}
                        + \parens*{
                            \frac{1 + \vecspan(\bias_\arm)}{\Delta_\arm}
                        }^2
                    }
                    + \parens*{
                        \frac{1 + \vecspan(\bias_\arm)}{\Delta_\arm}
                    }^4
                }
            }
            \; .
            & \parens*{\hspace{-.5em}\begin{tabular}{c} \scshape constant \\ \scshape terms \end{tabular}\hspace{-.5em}}
        \end{align*}
        Moreover, in case where there is a single optimal arm, i.e., $\abs{\arms^*} = 1$, the bound can be improved by using a better bound on the number of switches.
        Using \Cref{proposition_regret_and_pseudo_regret_appendix}, we obtain
        \begin{align*}
            \regret(T)
            & = \OH \parens*{
                \sum_{\arm \notin \arms^*} \parens*{
                    \frac{
                        \spanproxy(T)^2
                        \log(T)
                    }{
                        \Delta_\arm
                    }
                    + \frac{\spanproxy(T)^4}{\Delta_\arm^3}
                }
                + \abs{\arms}
                \max_{\arm \in \arms}
                \vecspan(\bias_\arm) \log \parens*{
                    \spanproxy(T)^2 \log(T)
                } 
            }
            & \parens*{\hspace{-.5em}\begin{tabular}{c} \scshape diverging \\ \scshape terms \end{tabular}\hspace{-.5em}}
            \\ 
            & \phantom{{} = {}}
            + \OH \parens*{
                \sum_{\arm \notin \arms^*} 
                \Delta_\arm \parens*{
                    (\beta_f)^2
                    + \strongdelayconstant \parens*{
                        \beta_f 
                        + \frac 1{\Delta_\arm^4}
                        + \parens*{
                            \frac{1 + \vecspan(\bias_\arm)}{\Delta_\arm}
                        }^2
                    }
                    + \parens*{
                        \frac{1 + \vecspan(\bias_\arm)}{\Delta_\arm}
                    }^4
                }
            }
            \; .
            & \parens*{\hspace{-.5em}\begin{tabular}{c} \scshape constant \\ \scshape terms \end{tabular}\hspace{-.5em}}
        \end{align*}
        This concludes the proof of \Cref{theorem_complete_upperbound_f_diverging}.
    \end{proof}

    \subsection{Proof of \texorpdfstring{\Cref{theorem_upperbound_log}}{Theorem 8}: Regret bound under bounded span proxy}
    \label[appendix]{appendix_instance_dependent_bounded}

    In this section, we show that when a uniform bound on the bias span is known, the diverging function $\spanproxy$ introduced in \Cref{appendix:upper_bound_instance_dependent} is no longer necessary. We can now run algorithm \algname{} with a \emph{constant} function instead. This additional knowledge allow us to improve on the regret obtained in \Cref{theorem_complete_upperbound_f_diverging} and to reach a logarithmic regret. The theorem below summarize the improvement on the regret allowed by this additional assumption.  

    \begin{theorem}[Instance-dependent regret of \algname{} when an upper bound on the bias is available.]
    \label{theorem_complete_upperbound_without_f_diverging}
        Let $\spanboundtotal:=1+2 \max \parens{1+\spanbound, 3 \spanbound}$.
        Let $\delta : \NN \to \RR_+$ be a confidence function such that $\sum_{t=1}^\infty \delta(t) < \infty$ and $\log (\frac 1{\delta(t)}) = \OH(\log(t))$.
        For all instance $\model \in \models$, $\regret(T; \model, \algname(\spanboundtotal))$ is of order
        \begin{equation}
            \regret(T; \model, \algname(\spanboundtotal))
            =
            \OH \parens*{
                \sum_{\arm \notin \arms^*}
                \frac{(\spanbound)^2 \log(T)}{\gain_* - \gain_\arm }
                + \abs{\arms} \max_{\arm \in \arms}
                \vecspan(\bias_\arm)
                \log(T)
            }
            .
        \end{equation}
        Moreover, in case where there is a single optimal arm $\abs{\arms^*} = 1$, the bound is improved to
        \begin{equation}
            \regret(T; \model, \algname(\spanboundtotal))
            =
            \OH \parens*{
                \sum_{\arm \notin \arms^*}
                \frac{(\spanbound)^2 \log(T)}{\gain_* - \gain_\arm }
            }
            .
        \end{equation}
    \end{theorem}

    Note that the number of episodes may be dominant this time around: If there are multiple optimal arms, \algname{} switches $\log(T)$ many times between them and may endure an extra cost of $\abs{\arms} \max_{\arm \in \arms} \vecspan(\bias_\arm) \log(T)$.
    This term is comparable to the main term in $(\spanbound)^2 \log(T)$.
    However, switching costs are negligible if there is a unique optimal arm, because one can show that there are $\OH(\log \log(T))$ of them. 
    Making the switching costs $\oh(\log(T))$ when $\abs{\arms^*} > 1$ is far from easy, and is the reason why in multi-armed bandits with switching costs \cite{ortner_online_2010,agrawal_asymptotically_2002,jun_survey_2004}, the assumption $\abs{\arms} = 1$ is often made to achieve asympotically   optimal regret \cite{agrawal_asymptotically_2002}.

    Regarding the proof of \Cref{theorem_complete_upperbound_without_f_diverging}, it is almost completely identical to the one of \Cref{theorem_complete_upperbound_f_diverging}.

    \begin{lemma}[Number of vistis, $\spanproxy$ bounded]
    \label[lemma]{appendix:lemma_order_visits_f_bounded}
        Consider a suboptimal arm $\arm \notin \arms^*$.
        We have
        \begin{equation}
        \label{equation_order_visits_f_bounded}
            \EE [\visits_\arm (T)]
            =
            \OH \parens*{
                \frac{(\spanbound)^2 \log\parens*{\frac 1{\delta(T)}}}{\parens*{\gain_* - \gain_\arm - 2 \epsilon}^2}
                + \frac 1{\epsilon^4}
                + \frac{(\strongdelayconstant)^2 + (1 + \spanbound)^2}{\epsilon^2}
            }
            .
        \end{equation}
    \end{lemma}
    
    \begin{proof}
        The proof of \Cref{appendix:lemma_order_visits_f_bounded} follows closely the line of \Cref{appendix:lemma_order_visits_f_diverging}.
        Except for the changes noted below, the reasoning carries over verbatim after replacing $\spanproxy$ with $\spanboundtotal$. 
        We therefore isolate  the steps that must be altered, specify the required modifications, and derive the corresponding implications for our algorithm's regret. 

        The analysis of $\EE[\visits_\arm (T)]$ splits the number of visits into for terms:
        empirical errors (I), optimism failures (II), optimistic pulls (III) and late bloomers (IV), see \Cref{equation_visit_decomposition_0}.

        The computations for the terms (I), (III) and (IV) carry on from \Cref{appendix:lemma_order_visits_f_diverging} to \Cref{appendix:lemma_order_visits_f_bounded}, albeit $\spanproxy$ must be changed to $\spanboundtotal$. 
        The only real change can be found in term (II): The original proof relied on the diverging property of $\spanproxy$, by introducing the ad-hoc constant $\beta_\spanproxy$ in \Cref{equation_beta_spanproxy}.
        The analysis of (II) then split the sums between what happens before and after the times at which $\spanproxy(t)$ is large enough to upper bound $\spanboundtotal$. 
        If a bound on the bias is already known, this is no longer necessary, because $\spanproxy(t) \ge \spanboundtotal$.
        This simplifies the calculation by half.
        
        Putting everything together, we obtain an equation of the form of \Cref{equation_order_visits_f_bounded}.
    \end{proof}
    
    \begin{proof}[Proof of \Cref{theorem_complete_upperbound_without_f_diverging}]
        Once \Cref{appendix:lemma_order_visits_f_bounded} is established, the proof is concluded just as the one of \Cref{theorem_complete_upperbound_f_diverging}, by picking $\epsilon = ((\min_{\arm \notin \arms^*} (\gain_* - \gain_\arm))^2/( (\spanboundtotal)^2\log(T)))^{\frac 15}$ and simplifying the bound. 
    \end{proof}

    \subsection{Proof of \texorpdfstring{\Cref{theorem_worst_case_sqrt}}{Theorem 9}: regret upper bound under bounded span proxy}
    \label[appendix]{appendix_worst_case_upper}

    The goal of this section is to derive worst-case regret of \algname{}   when a bound on the bias span and the decision epoch delays is known. 
    Just like for the instance dependent upper bound, we provide a finer result below (\Cref{theorem_worstcase_upperbound}) and explain how the result from the main body (\Cref{theorem_worst_case_sqrt}) is an instantiation of it.
    First, we refine the set $\models_D$ from the main body by differentiating between the constraints $1 + \max_{\arm \in \arms} \vecspan(\bias_\arm) \le D$ and $\strongdelayconstant \le D$.
    For $\spanbound, \delaybound \in \RR_+$, let
    \begin{equation}
    \label{equation_refined_worstcase_space}
        \models_{\spanbound, \delaybound}
        :=
        \braces*{
            \model \in \models
            :
            1 + \max_{\arm \in \arms} \vecspan(\bias_\arm)
            \le \spanbound
            \text{~and~}
            \strongdelayconstant \le \delaybound
        }.
    \end{equation}
    On $\models_{\spanbound, \delaybound}$, we derive the following worst-case regret guarantees.

    \begin{theorem}[Worst-case regret of \algname]
    \label{theorem_worstcase_upperbound}
        Let $\spanbound \in \RR_+$. 
        Let $\delta : \NN \to \RR_+$ be a confidence function such that $\sum_{t=1}^\infty t \cdot \delta(t) < \infty$ and $\log (\frac1{\delta(t)}) = \Theta(\log(t))$.
        For all instances $\model \in \models_{\spanbound, \delaybound}$, $\algname(\spanbound)$ has the following worst-case regret: 
        \begin{equation}
        \label{equation_worstcase_upperbound_f_diverging}
            \regret(T; \model, \algname(\spanbound))
            =
            \OH \parens*{
                \spanbound \sqrt{
                    \abs{\arms} T \log(T)
                }
                +
                \spanbound \delaybound \abs{\arms} \log(T)^2
            }
            .
        \end{equation}
    \end{theorem}

    Observe that, in opposition to the instance-dependent bound, the cost due to delays (\Cref{assumption_strong_decision_epochs}) scales with the number of episodes as a second order term.
    The dominant term $\spanbound \sqrt{\abs{\arms} T \log(T)}$ scales exclusively with the span bound.
    Because the span bound $\spanbound$ is to be analogous to a diameter-like quantity $D$, this term is to be compared with the famous $D \abs{\states} \sqrt{\abs{\arms} T \log(T)}$ regret bound of \texttt{UCRL2} (\cite{auer_near_optimal_2009}), for which the dependency in $D$ rather than $\sqrt{D}$ is well-known to be suboptimal. 

    \begin{proof}[Proof of \Cref{theorem_worstcase_upperbound}]
        Let $\model \in \models$ such that $1 + \max_{\arm \in \arms} \vecspan(\bias_\arm) \le \spanbound$, and such that $\strongdelayconstant \le \delaybound$.
        Fix a horizon $T \ge 1$ and an initial state $\state_0 \in \states$.
        Because the dependencies in $\model$, $\learner \equiv \algname{}$, $T$ and $\state_0$ are fixed from now on, these objects are dropped in notations. 

        We write $(t'_k)_{k \ge 1}$ the sequence of episodes clipped to $\braces{1,\dots, T-1}$. 
        That is, $t'_k := \min \braces*{t_k, T-1}$.

        We start by rewriting the regret relatively to the optimal gain and episodes.
        We have
        \begin{equation}
        \label{equation_worstcase_upperbound_f_diverging_1}
            \regret(T)
            \overset{\eqnum{1}}\le 
            \EE \brackets*{
                \sum_{t=1}^{T-1}
                \parens*{
                    \gain_* - \Reward(t+1)
                }
            }
            + \spanbound 
            \overset{\eqnum{2}}=
            \EE \brackets*{
                \sum_{k=1}^\infty
                \sum_{t=t'_k}^{t'_{k+1}-1}
                \parens*{
                    \gain_* - \Reward(t+1)
                }
            }
            + \spanbound 
        \end{equation}
        where 
        \texteqnum{1} follows by \Cref{corollary_regret_simple_expression};
        and
        \texteqnum{2} rewrites the sum according to episodes.

        Then, the summand $\gain_* - \Reward(t+1)$ is expanded into sub-terms, each one corresponding to an error of a different nature. 
        Fix any episode $k$ and time $t\in[t_k,t_{k+1}-1]$. 
        We can write
        \begin{gather*}
            \gain_* - \Reward(t+1) 
            \\
            =
            \underbrace{
                \gain_* - \tilde \gain_*(t'_k) 
            }_{
                \substack{
                    \textsc{Optimism} \\ \textsc{Failure}
                    \\[.5em]
                    \displaystyle (\mathrm{I})
                }
            }
            +
            \underbrace{
                \tilde \gain_*(t'_k)- \tilde \gain_{\Arm(t_k)}(t'_k)
            }_{
                \substack{
                    \textsc{Arm Selection} \\ \textsc{Slackness}
                    \\[.5em]
                    \displaystyle (\mathrm{II})
                }
            }
            +
            \underbrace{
                \tilde \gain_{\Arm(t'_k)}(t'_k) - \hat \gain_{\Arm(t'_k)}(t'_k)
            }_{
                \substack{
                    \textsc{Optimism} \\ \textsc{Overshot}
                    \\[.5em]
                    \displaystyle (\mathrm{III})
                }
            }
            +
            \underbrace{
                \hat \gain_{\Arm(t'_k)}(t'_k) - \gain_{\Arm(t'_k)}
            }_{
                \substack{
                    \textsc{Empirical} \\ \textsc{Error}
                    \\[.5em]
                    \displaystyle (\mathrm{IV})
                }
            }
            +
            \underbrace{
                \gain_{\Arm(t'_k)} - \Reward(t+1)
            }_{
                \substack{
                    \textsc{Navigation} \\ \textsc{Error}
                    \\[.5em]
                    \displaystyle (\mathrm{V})
                }
            }
        \end{gather*}
        We proceed to bound each term's contribution separately in expectation over all episodes and times.

        \paragraph{Optimism failure (I)}
        This term is shown to be $\OH(1)$.
        We first define the ``good event" at episode $k$ by: 
        \begin{equation}
        \label{equation_worst_regret_good_event}
            \event(t'_k) 
            := 
            \braces*{\tilde \gain_*(t'_k) > \gain_* }.
        \end{equation}
        The main idea of this first part of the proof is that, in the ``good event" $\event(t_k)$ where the gain of the optimal arm lie below its optimistic estimator --- and therefore below its upper-confidence bound ---  (Term I) is negative.
        We then use the concentration inequality (\Cref{lemma_estimation_error}) to show that the probability of the ``bad event" $\event(t'_k)^c$ decays sufficiently quickly, so that its total error contribution is small. 
        \begin{align*}
            &\EE \brackets*{
                \sum_{k =1}^\infty
                \sum_{t=t'_k}^{t'_{k+1}-1}
                \mathrm{(I)}
            }
            \\
            &\overset{\eqnum{1}}{\le}
            \EE \brackets*{
                \sum_{k =1}^\infty
                \sum_{t=t'_k}^{t'_{k+1}-1}
                \indicator{
                    \tilde \gain_*(t'_k) < \gain_*
                }
            }
            \\
            & \overset{\eqnum{2}}\le
            \EE \brackets*{
                \sum_{k =1}^\infty
                \indicator{t_k \le T-1}
                \sum_{t=t_k}^{t_{k+1}-1}
                \indicator{
                    \tilde \gain_*(t_k) < \gain_*
                }
            }
            \\
            &\overset{\eqnum{3}}{\le}
            \underbrace{\EE \brackets*{
                \sum_{k =1}^\infty
                \indicator{t_k \le T-1}
                \sum_{t=t_k}^{t_k^\mathrm{DT}-1}
                \indicator{
                    \tilde \gain_*(t_k) < \gain_*
                }
            }}_{\displaystyle \mathrm{(I.1)}} 
            +
            \underbrace{ \EE \brackets*{
                \sum_{k =1}^\infty
                \indicator{t_k \le T-1}
                \sum_{t=t_k^\mathrm{DT}}^{t_{k+1}-1}
                \indicator{
                    \tilde \gain_*(t_k) < \gain_*
                }
            }}_{\displaystyle \mathrm{(I.2)}}
        \end{align*}
        where 
        \texteqnum{1} holds since $\gain_* - \tilde \gain_*(t_k) \le 1$ a.s., and on the good event $\event(t_k)$ from \Cref{equation_worst_regret_good_event}, we have $\gain_* - \tilde \gain_*(t_k) \le 0$; 
        \texteqnum{2} just rewrites the sum of \texteqnum{1};
        and
        \texteqnum{3} splits every episode according to whether the doubling trick rule is passed or not, by setting $t_k^{\mathrm{DT}} := t_k + \visits_{\Arm(t_k)}(t_k)$.
        Note that $t_k^\mathrm{DT}$ is a stopping time. 

        We then bound the terms (I.1) and (I.2) separately. 
        For the first term (I.1), we have
        \begin{align*}
            (\mathrm{I.1})
            &\overset{\eqnum{1}}=
            \EE \brackets*{
                \sum_{k=1}^\infty
                t_k
                \cdot
                \indicator{
                    \substack{
                        \displaystyle
                        t_k < T
                        \\
                        \displaystyle
                        \tilde \gain_*(t_k) < \gain_*
                    }
                }
            }
            \overset{\eqnum{2}}=
            \EE \brackets*{
                \sum_{t=1}^{T-1}
                t
                \cdot
                \indicator{
                    \tilde \gain_*(t) < \gain_*
                }
            }
            \overset{\eqnum{3}}{\le}
            \sum_{t=1}^{\infty}
            t \cdot \delta(t)
            = \OH(1)
        \end{align*}
        where in
        \texteqnum{1} we state that $t_k^\mathrm{DT} \le 2t_k$, hence that $t_k^\mathrm{DT} - t_k \le t_k$;
        \texteqnum{2} we use that the left sum is a subset of the right one; 
        and
        \texteqnum{2} we invoke \Cref{lemma_estimation_error}. 
        
        For the second term (I.2), we have 
        \begin{align*}
            \mathrm{(I.2)}
            & =
            \EE \brackets*{
                \sum_{k=1}^\infty
                \parens*{
                    t_{k+1} - t_k^\mathrm{DT}
                }
                \indicator{
                    \substack{
                        \displaystyle
                        t_k < T 
                        \\
                        \displaystyle
                        \tilde \gain_*(t_k) < \gain_*
                    }
                }
            }
            \\
            & \overset{\eqnum{1}}=
            \EE \brackets*{
                \sum_{k=1}^\infty
                \indicator{
                    \substack{
                        \displaystyle
                        t_k < T 
                        \\
                        \displaystyle
                        \tilde \gain_*(t_k) < \gain_*
                    }
                }
                \EE \brackets*{
                    t_{k+1} - t_k^\mathrm{DT}
                    \middle|
                    \History(t_k^\mathrm{DT})
                }
            }
            \\
            & \overset{\eqnum{2}}=
            \EE \brackets*{
                \sum_{k=1}^\infty
                \indicator{
                    \substack{
                        \displaystyle
                        t_k < T 
                        \\
                        \displaystyle
                        \tilde \gain_*(t_k) < \gain_*
                    }
                }
                \EE \brackets*{
                    \inf \braces*{
                        t - t_k^\mathrm{DT}
                        :
                        t \ge t_k^\mathrm{DT},
                        t \in \braces{\tau_i: i \ge 1}
                    }
                    \middle|
                    \History(t_k^\mathrm{DT})
                }
            }
            \\
            &\overset{\eqnum{3}}{\le}
            \strongdelayconstant
            \EE \brackets*{
                \sum_{k =1}^\infty
                \indicator{
                    \substack{
                        \displaystyle
                        t'_k < T 
                        \\
                        \displaystyle
                        \tilde \gain_*(t'_k) < \gain_*
                    }
                }
            }
            \le
            \strongdelayconstant
            \EE \brackets*{
                \sum_{t =1}^{T-1}
                \indicator{
                    \tilde \gain_*(t) < \gain_*
                }
            }
            \overset{\eqnum{4}}{\le}
            \strongdelayconstant
            \sum_{t =1}^{\infty}
            \delta (t)
            =
            \OH(1)
        \end{align*}
        where 
        \texteqnum{1} stems from the tower property, and uses that $\indicator{t_k < T, \tilde{\gain}_* (t_k) < \gain_*}$ is $\sigma(\History(t_k^{\mathrm{DT}}))$-measurable;
        \texteqnum{2} uses that once $t_k^{\mathrm{DT}}$ is passed, episode $k$ ends as soon as a decision epoch $\tau_i$ triggers;
        \texteqnum{3} bounds the delay before the decision epoch via \Cref{assumption_strong_decision_epochs}; 
        and 
        \texteqnum{4} follows from \Cref{lemma_estimation_error}.

        Combining (I.1) and (I.2) we conclude that the total contribution to the regret from term I is $\OH(1)$.

        \paragraph{Slackness of arm selection (II).}
        By construction of our algorithm \algname{}, at the beginning of every episodes $k$, \algname{} picks the arm $\Arm(t_k)$ achieving $\tilde{\gain}_\arm (t_k) := \max_{\arm \in \arms} \tilde{\gain}_\arm (t_k)$.
        So 
        \begin{equation*}
            \EE \brackets*{
                \sum_{k=1}^\infty
                \sum_{t=t'_k}^{t'_{k+1}-1}
                (\mathrm{II})
            }
            = 0.
        \end{equation*}

        \paragraph{Optimism overshot (III).}
        This term is the dominant one.
        We start by writing it as follows:
        \begin{align}
        \nonumber
            &\EE \brackets*{
                \sum_{k=1}^\infty
                \sum_{t=t'_k}^{t'_{k+1}-1}
                (\mathrm{III})
            }
            \\
            & \overset{\eqnum{1}}=
            \EE \brackets*{
                \sum_{k=1}^\infty
                (t'_{k+1} - t'_k)
                \cdot \spanbound  \parens*{
                    \sqrt{
                        \frac 1{\visits_{\Arm(t'_k)}(t'_k)}
                        \log \parens*{
                            \frac{1+\visits_{\Arm(t'_k)}(t'_k)}{\delta(t'_k)}
                        }
                    }
                    + 
                    \frac{
                        \log \parens*{
                            1 + \visits_{\Arm(t'_k)}(t'_k)
                        }
                    }{
                        \visits_{\Arm(t'_k)}(t'_k)
                    }
                }
            }
            \\
        \nonumber
            & \overset{\eqnum{2}}\le
            \EE \brackets*{
                \sum_{k=1}^\infty
                \indicator{t_k \le T}
                (t_{k+1} - t_k)
                \cdot \spanbound  \parens*{
                    \sqrt{
                        \frac 1{\visits_{\Arm(t'_k)}(t'_k)}
                        \log \parens*{\frac{1+T}{\delta(T)}}
                    }
                    + 
                    \frac{\log(1+T)}{\visits_{\Arm(t'_k)}(t'_k)}
                }
            }
            \\
        \nonumber
            & \overset{\eqnum{3}}\le
            \left.
                \spanbound 
                \cdot
                \EE \brackets*{
                    \sum_{k=1}^\infty
                    \sum_{\arm \in \arms}
                    \indicator{
                        \substack{
                            \displaystyle
                            t_k \le T
                            \\
                            \displaystyle
                            \Arm(t_k) = \arm
                        }
                    }
                    \sum_{t=t_k}^{t_k^{\mathrm{DT}}-1}
                    \parens*{
                        \sqrt{
                            \frac {\log\parens*{\frac{1+T}{\delta(T)}}}{\visits_{\arm}(t_k)}
                        }
                        + \frac {\log(1+T)}{\visits_{\arm}(t_k)}
                    }
                }
            \qquad \right\} \mathrm{(III.1)}
            \\
        \label{equation_worstcase_upperbound_f_diverging_III}
            &  
            \left.
                \hspace{.4em}
                + \spanbound 
                \cdot
                \EE \brackets*{
                    \sum_{k=1}^\infty
                    \sum_{\arm \in \arms}
                    \indicator{
                        \substack{
                            \displaystyle
                            t_k \le T
                            \\
                            \displaystyle
                            \Arm(t_k) = \arm
                        }
                    }
                    \sum_{t=t_k^{\mathrm{DT}}}^{t_{k+1}-1}
                    \parens*{
                        \sqrt{
                            \frac {\log\parens*{\frac{1+T}{\delta(T)}}}{\visits_{\arm}(t_k)}
                        }
                        + \frac {\log(1+T)}{\visits_{\arm}(t_k)}
                    }
                }
            \hspace{1.75em} \right\} \mathrm{(III.2)}
        \end{align}
        where 
        \texteqnum{1} unfolds the definition of the index $\tilde{\gain}_\arm (t)$;
        \texteqnum{2} simplifies the expression using that $\delta(-)$ is non-increasing and the non-negativity of the summand;
        and
        \texteqnum{3} splits every episode according to whether the doubling trick rule is passed or not, by setting $t_k^{\mathrm{DT}} := t_k + \visits_{\Arm(t_k)}(t_k)$.
        We continue by bounding (III.1) and (III.2) in \Cref{equation_worstcase_upperbound_f_diverging_III}.

        We begin with (III.1), which is shown to dominate (III.2).
        We have
        \begin{align}
        \nonumber
            \mathrm{(III.1)}
            & \overset{\eqnum{1}}\le
            2 \spanbound \cdot \EE \brackets*{
                \sum_{k=1}^\infty
                \sum_{\arm \in \arms}
                \indicator{
                    \substack{
                        \displaystyle
                        t_k \le T
                        \\
                        \displaystyle
                        \Arm(t_k) = \arm
                    }
                }
                \sum_{t=t_k}^{t_k^{\mathrm{DT}}-1}
                \parens*{
                    \sqrt{
                        \frac {\log\parens*{\frac{1+T}{\delta(T)}}}{\visits_{\arm}(t)}
                    }
                    + \frac {\log(1+T)}{\visits_{\arm}(t)}
                }
            }
            \\
        \nonumber
            & \overset{\eqnum{2}}\le 
            2 \spanbound \sqrt{\log \parens*{\frac{1+T}{\delta(T)}}}
            \cdot \EE \brackets*{
                \sum_{\arm \in \arms}
                \sum_{n=1}^{\visits_\arm (t)}
                \sqrt{\frac 1n}
            }
            + 2 \spanbound \log(1+T)
            \cdot \EE \brackets*{
                \sum_{\arm \in \arms}
                \sum_{n=1}^{\visits_\arm (t)}
                {\frac 1n}
            }
            \\
        \nonumber
            & \overset{\eqnum{3}}\le
            4 \spanbound \sqrt{\log \parens*{\frac{1+T}{\delta(T)}}}
            \cdot \EE \brackets*{
                \sum_{\arm \in \arms}
                \sqrt{\visits_\arm (T)}
            }
            + 4 \spanbound \log(1+T)
            \cdot \EE \brackets*{
                \sum_{\arm \in \arms}
                \log(1 + \visits_\arm (T)) 
            }
            \\
        \label{equation_worstcase_upperbound_f_diverging_III_1}
            & \overset{\eqnum{4}}\le
            4 \spanbound \sqrt{
                \abs{\arms} T 
                \log \parens*{\frac{1+T}{\delta(T)}}
            }
            + 4 \spanbound \abs{\arms} \log(1+T)^2
        \end{align}
        where
        \texteqnum{1} changes $\visits_{\arm}(t_k)$ to $\visits_{\arm}(t)$ in (III.1) by making the observation that, for $t_k \le t \le t_k^{\mathrm{DT}}-1$ and writing $\arm = \Arm(t_k)$, we have $\visits_\arm (t) \le 2 \visits_\arm (t_k)$ by definition of the doubling trick;
        \texteqnum{2} rearranges both sums;
        \texteqnum{3} follows from standard analysis;
        and
        \texteqnum{4} uses that $\sum_{\arm \in \arms} \visits_\arm (T) = T$ and invokes Cauchy-Schwartz inequality. 

        Continuing with (III.2), we have
        \begin{align}
        \nonumber
            \mathrm{(III.2)}
            & \overset{\eqnum{1}}\le 
            \spanbound
            \cdot
            \EE \brackets*{
                \sum_{k=1}^\infty
                \indicator{t_k \le T}
                \sum_{t=t_k^{\mathrm{DT}}}^{t_{k+1}-1}
                \parens*{
                    \sqrt{
                        {\log\parens*{\frac{1+T}{\delta(T)}}}
                    }
                    + {\log(1+T)}
                }
            }
            \\
        \nonumber
            & \overset{\eqnum{2}}=
            \spanbound \cdot \parens*{
                \sqrt{
                    {\log\parens*{\frac{1+T}{\delta(T)}}}
                }
                + {\log(1+T)}
            }
            \cdot
            \EE \brackets*{
                \sum_{k=1}^\infty
                \indicator{t_k \le T}
                \inf \braces*{
                    \tau_i - t_k^{\mathrm{DT}}
                    :
                    \tau_i \ge t_k^{\mathrm{DT}}
                }
            }
            \\
        \nonumber
            & \overset{\eqnum{3}}\le
            \spanbound \strongdelayconstant 
            \cdot \parens*{
                \sqrt{
                    {\log\parens*{\frac{1+T}{\delta(T)}}}
                }
                + {\log(1+T)}
            }
            \cdot
            \EE \brackets*{
                K(T)
            }
            \\
        \label{equation_worstcase_upperbound_f_diverging_III_2}
            & \overset{\eqnum{4}}=
            \spanbound \strongdelayconstant \abs{\arms}
            \cdot \parens*{
                \sqrt{
                    {\log\parens*{\frac{1+T}{\delta(T)}}}
                }
                + {\log(1+T)}
            } \cdot \log(2T)
        \end{align}
        where
        \texteqnum{1} merely simplifies the expression by removing the visit counts;
        \texteqnum{2} uses that once $t_k^{\mathrm{DT}}$ is passed, episode $k$ ends as soon as a decision epoch $\tau_i$ triggers;
        \texteqnum{3} bounds the expected delay before the next decision epoch using \Cref{assumption_strong_decision_epochs} and the Tower Property --- $t_k$ is $\sigma(\History(t_k^{\mathrm{DT}}))$-measurable indeed;
        and
        \texteqnum{4} bounds the number of episodes using \Cref{lemma_number_of_episodes}.

        Combining \Cref{equation_worstcase_upperbound_f_diverging_III_1,equation_worstcase_upperbound_f_diverging_III_2} into \Cref{equation_worstcase_upperbound_f_diverging_III}, we obtain a bound on (III).
        
        \paragraph{Empirical error (IV).}
        For this term, introduce the good event
        \begin{equation}
        \label{equation_worstcase_upperbound_f_diverging_IV_event}
            \event(T)
            :=
            \parens*{
                \forall t \le T,
                \forall \arm \in \arms
                :
                \quad
                \abs*{
                    \gain_\arm (t) - \gain_\arm
                }
                \le
                \spanbound \cdot \parens*{
                    \sqrt{
                        \frac{2\log(1+T)}{\visits_\arm(t)}
                    }
                    + \frac{\log(1+T)}{\visits_\arm (t)}
                }
            }.
        \end{equation}
        By \Cref{lemma_estimation_error}, we have $\Pr\parens{\event(T)^c} \le T^{-1}$.
        Now, we bound (IV) as follows.
        \begin{align}
            \mathrm{(IV)}
            & \le 
            \EE \brackets*{
                \sum_{k=1}^\infty
                \indicator{t_k \le T}
                \sum_{t=t_k}^{t_{k+1}-1}
                \abs*{
                    \hat{\gain}_{\Arm(t_k)}(t_k) - \gain_{\Arm(t_k)}
                }
            }
            \\
            & \overset{\eqnum{1}}\le
            \EE \brackets*{
                T \cdot \indicator{\event(T)^c}
            }
            + \EE \brackets*{
                \sum_{k=1}^\infty
                \indicator{t_k \le T}
                \sum_{t=t_k}^{t_{k+1}-1}
                \spanbound \cdot \parens*{
                    \sqrt{
                        \frac{2\log(1+T)}{\visits_\arm(t)}
                    }
                    + \frac{\log(1+T)}{\visits_\arm (t)}
                }
            }
            \\
            & \overset{\eqnum{2}}\le
            1 + 4 \spanbound \cdot \parens*{
                \sqrt{2 \abs{\arms} T \log(1+T)}
                + \abs{\arms} \log(1+T)
            }
        \end{align}
        where 
        \texteqnum{1} cuts the summand through the event $\event(T)$ of \Cref{equation_worstcase_upperbound_f_diverging_IV_event} and uses that $\abs*{\hat{\gain}_{\Arm(t_k)}(t_k) - \gain_{\Arm(t_k)}} \le 1$ a.s.;
        and
        \texteqnum{2} is obtained via a similar line as the one (III.2), see \Cref{equation_worstcase_upperbound_f_diverging_III_1}.

        \paragraph{Navigation error (V).}
        This term is due to switches from one arm to another, and is bounded by the number of episodes. 
        We have
        \begin{align*}
            \EE \brackets*{
                \sum_{k=1}^\infty
                \sum_{t=t'_k}^{t'_{k+1}-1}
                (\mathrm{V})
            }
            & \overset{\eqnum{1}}= 
            \EE \brackets*{
                \sum_{k=1}^\infty
                \sum_{t=t'_k}^{t'_{k+1}-1}
                \parens*{
                    \bias_{\Arm(t'_k)} (\State_{\Arm(t'_k)}(t+1))
                    - 
                    \bias_{\Arm(t'_k)} (\State_{\Arm(t'_k)}(t))
                }
            }\\
            & \overset{\eqnum{2}}=
            \EE \brackets*{
                \sum_{k=1}^\infty
                \parens*{
                    \bias_{\Arm(t'_k)} (\State_{\Arm(t'_k)}(t'_{k+1}))
                    - 
                    \bias_{\Arm(t'_k)} (\State_{\Arm(t'_k)}(t'_k))
                }
            }
            \\
            & \le
            \max_{\arm \in \arms}
            \vecspan(\bias_\arm)
            \cdot
            \EE [K(T)]
            \overset{\eqnum{3}}\le
            2 \abs{\arms} \spanbound \log(T)
        \end{align*}
        where 
        \texteqnum{1} invokes the Poisson equation of the pure policy $\policy_{\Arm(t'_k)}$;
        \texteqnum{2} recognizes a telescopic sum;
        and
        \texteqnum{3} bounds the expected number of episodes via \Cref{lemma_number_of_episodes}.

        \paragraph{Conclusion.}
        We conclude by merging the bounds of (I), (II), (III), (IV) and (V) together. 
    \end{proof}

    \subsection{Auxiliary structural results on \texorpdfstring{\algname}{UCB-NOM}}

    In this section we provide a large range of results about \algname{} that are used throughout its regret analysis. 
    
    \subsubsection{Concentration results of the empirical gain estimator}

    In this paragraph, we provide a deviation inequality on the quality of the empirical gain estimator $\hat{\gain}_{\arm, n}$ with \Cref{lemma_gain_concentration_per_sample}.
    This result is a variant of \Cref{lemma_estimation_error}, which is expressed by putting the emphasis on the error probability of error rather than on the expected deviation from the mean. 
    
    \begin{lemma}
    \label[lemma]{lemma_gain_concentration_per_sample}
        The empirical estimator of arm $\arm$ satisfies, for all $n \ge 1$ and $x \ge 0$:
        \begin{align*}
            \Pr \parens*{
                \hat \gain_{\arm,n}- \gain_\arm \ge x }
                &\le 
                \exp
                \parens*{
                    - 
                    \frac{
                        n x^2
                        - 2 x \vecspan(\bias_\arm) (1 + \log_2 (1 + n))
                    }{
                        (1 + \vecspan(\bias_\arm))^2
                    }
                }
                \\
                &\le
                \exp
                \parens*{
                    - 
                    \frac{
                        n \parens*{
                            x 
                            - 
                            \frac{
                                3\vecspan(\bias_\arm) 
                            }{\sqrt{n}}
                        }^2
                        -9\vecspan(\bias)^2
                    }{
                        (1 + \vecspan(\bias_\arm))^2
                    }
                }
                .
        \end{align*}
    \end{lemma}
    \begin{proof}
        Let $(\tau_k^\arm)$ be the sequence of stopping times enumerating the pulls of arm $\arm$, that is, $\tau_1^\arm := \inf \braces{t \ge 1: \Arm(t) = \arm}$ and $\tau_{i+1}^\arm := \inf \braces{t > \tau_i^\arm: \Arm(t) = \arm}$. 
        To begin with, introduce the following martingale noise terms:
        \begin{equation*}
        \begin{aligned}
            \noise_{\arm,n}^\reward 
            & := 
            \sum_{k=1}^{n} 
            \parens[\Big]{
                \Reward(\tau_k^a+1) 
                - 
                \reward_\arm (\State_\arm (\tau_k^\arm))
            },
            \\
            \noise_{\arm,n}^\kernel
            & := 
            \sum_{k=1}^{n} 
            \parens[\Big]{
                \bias_\arm (\State_\arm (\tau_k^\arm+1)) 
                - 
                \kernel_\arm^1 (\State_\arm(\tau_k^\arm)) \bias_\arm
            }
            .
        \end{aligned}
        \end{equation*}
        The quantities $\noise_{\arm, n}^\reward$ and $\noise_{\arm, n}^\kernel$ respectively measure how much the immediate rewards (resp.~state transitions) obtained by activating arm $\arm$ deviate from their mean after $n$ pulls of arm $\arm$.
        The aggregate noise is denoted $\noise_{\arm,n} := \noise_{\arm,n}^\reward + \noise_{\arm,n}^\kernel$.
        Now, we bound the deviations of $\gain_{\arm, n}$ as follows.
        For $x \ge 0$, we have
        \begin{align*}
            \PP \parens*{
                \hat \gain_{\arm,n}- \gain_\arm \ge x
            }
            & \overset{\eqnum{1}}= 
            \PP \parens*{
                \sum_{k=1}^n 
                \Reward(\tau_k^\arm +1)
                -
                \gain_\arm
                \ge n x
            }
            \\
            &\overset{\eqnum{2}}{=}  
            \PP \parens*{
                \sum_{k=1}^n
                \parens*{
                    \noise_{\arm,n} 
                    +
                    \parens*{
                        \unit_{\State_\arm(\tau_k^\arm)} 
                        - 
                        \unit_{\State_\arm(\tau_k^\arm+1)}
                    }
                    \bias_\arm
                }
                \ge n x
            }
            \\
            & \overset{\eqnum{3}}{\le} 
            \PP \parens*{
                \sum_{k=1}^n
                \noise_{\arm,n} 
                +
                \vecspan(\bias_\arm)\parens*{1+\log_2(1+n)}
                \ge n x
            }
            \\
            & \overset{\eqnum{4}}{\le} 
            \exp
            \parens*{
                - 
                \frac{
                    n \parens*{
                        x 
                        - 
                        \frac{
                            \vecspan(\bias_\arm)  
                            (1 + \log_2 (1 + n))
                        }{n}
                    }^2
                }{
                    (1 + \vecspan(\bias_\arm))^2
                }
            }
            \\ &
            \overset{\eqnum{5}}\le
            \exp
            \parens*{
                - 
                \frac{
                    n x^2
                    - 2 x \vecspan(\bias_\arm) (1 + \log_2 (1 + n))
                }{
                    (1 + \vecspan(\bias_\arm))^2
                }
            }
            \\
            &\overset{\eqnum{6}}\le
            \exp
            \parens*{
                - 
                \frac{
                    n \parens*{
                        x 
                        - 
                        \frac{
                            3\vecspan(\bias_\arm) 
                        }{\sqrt{n}}
                    }^2
                    -9\vecspan(\bias)^2
                }{
                    (1 + \vecspan(\bias_\arm))^2
                }
            }
        \end{align*}
        where 
        \texteqnum{1} unfolds the definition of the empirical estimate $\gain_{\arm, n}$ (see \Cref{equation_empirical_estimate});
        \texteqnum{2} follows from the Poisson equation of $\policy_\arm$;
        \texteqnum{3} bounds the switching term using \Cref{lemma_bound_nearly_telescopic_bias};
        \texteqnum{4} stems from Hoeffding inequality (\Cref{lemma_hoeffding});
        and
        \texteqnum{5} uses the identity $(u - v)^2 \ge u^2 - 2 u v$.
        As for 
        \texteqnum{6}, it follows from simple algebraic manipulations, that we detail below:
        Using that $1+\log_2(1+n) \le 3 \log(1+n) \le 3 \sqrt{n}$, we have
        \begin{align*}
            nx^2 - 2 x \vecspan(\bias_\arm) (1+\log(1+n)) 
            & \ge 
            nx^2 -6x \vecspan(\bias_\arm) \log(1+n) 
            \\ & \ge n x^2 -6x \vecspan(\bias_\arm) \sqrt{n}
            \\ & = n \parens*{x - \frac{3 \vecspan(\bias_\arm)}{\sqrt{n}}}^2 - 9 \vecspan(\bias_\arm)^2
            .
        \end{align*}
        Hence \texteqnum{6} above.
        This concludes the proof.
    \end{proof}

    \subsubsection{A bound on the number of episodes}

    At multiple point in the regret analysis, we use the following upper bound of the number of episodes. 
    The number of episodes can be bounded by $\OH(\log(T))$ thanks to the doubling trick that \algname{} uses. 

    \begin{lemma}[Number of episodes]
    \label[lemma]{lemma_number_of_episodes}
        The number of episodes of \algname{} up to horizon $T \ge 1$ is bounded by
        \begin{equation}
            K(T)
            :=
            \abs*{
                \braces*{
                    k \ge 1:
                    t_k \le T-1
                }
            }
            \le 
            \abs{\arms} 
            \log (2 T)
        \end{equation}
    \end{lemma}
    \begin{proof}
        Given an arm $\arm \in \arms$, let $\mathcal{K}_\arm (T) := \braces{k \ge 1: t_k \le T-1, \Arm(t_k) = \arm}$ be the set of episodes prior to time $T$ where $\arm$ is pulled.
        Denote $(t_\ell^\arm)_{\ell \ge 1}$ the enumeration of episodes initial times where arm $\arm$ is pulled, i.e., $t^\arm_{\ell+1} := \inf \braces{t_k > t_\ell^\arm : \Arm(t_k) = \arm}$.
        Thanks to the Doubling Trick rule, we have $\visits_\arm (t_\ell^\arm) \ge 2^{\ell-1} - 1$.
        So,
        \begin{equation*}
            \abs{\mathcal{K}_\arm (T)}
            \le
            1 + 
            \log_2 \parens*{
                1 + \visits_\arm (T)
            }
            .
        \end{equation*}
        Summing over arms, we obtain:
        \begin{equation}
            K(T)
            \le 
            \abs{\arms}
            + 
            \sum_{\arm \in \arms}
            \log_2 \parens*{
                1 + \visits_\arm (T)
            }
            =: 
            \abs{\arms}
            +
            \psi(\visits (T)).
        \end{equation}
        Note that the above function $\psi : \RR_+^\arms \to \RR_+$ is the parallel sum of the concave increasing function $x \mapsto \log_2 (1+x)$.
        The maximum is reached when $\visits_\arm (T) = \visits_{\arm'}(T)$ for all $\arm, \arm' \in \arms$.
        We conclude using that $\sum_{\arm \in \arms} \visits_\arm (T) = T$ and that $\log_2(1+x) \le \log(2x)$ for $x \ge 1$.
    \end{proof}
    
    \subsubsection{About the number of switches}

    In the previous paragraph, we have shown that the number of episodes is $\OH(\log(T))$.
    Because the number of switches is bounded by the number of episodes, the number of switches is $\OH(\log(T))$.
    However, this bound proves to be insufficient to improve the instance dependent regret bounds under $\abs{\arms^*}=1$ in \Cref{theorem_complete_upperbound_without_f_diverging}.
    To that end, we will use the more precise result below.

    \begin{lemma}[Number of switches under Doubling Trick]
    \label[lemma]{lemma_switches_doubling_trick}
        After $n$ pulls of arm $\arm$, the number of switches of that arm is bounded by $\log_2 (1+n)$, i.e., for all $T \ge 1$,
        \begin{equation}
            \sum_{t=1}^{T-1}
            \indicator{\Arm(t) = \arm, \Arm(t+1) \ne \arm}
            =
            \sum_{i=1}^{\visits_\arm (T)}
            \indicator{\tau_i^\arm + 1 \ne \tau_{i+1}^a}
            \le
            \log_2 \parens*{1 + \visits_\arm (T)}
            .
        \end{equation}
    \end{lemma}
    
    \begin{proof}
        By design of the algorithm, arms are played in contiguous segments of doubling durations. 
        So, $\tau_{2^i+j}^\arm = \tau_{2^i}^\arm + j$ for all $i \ge 0$ and $j \in \braces{0, \ldots, 2^j-1}$.
        In particular, there are at most $i$ switches of arm $\arm$ prior to time $\tau_{2^i+j}^\arm$. 
        We conclude using $\visits_\arm (\tau_{k}^\arm) = k - 1$.
    \end{proof}

    \begin{lemma}[Almost-telescopic bias sum with non-consecutive pulls ]
    \label[lemma]{lemma_bound_nearly_telescopic_bias}
        Let $(e_{\state})_{\state \in \states}$ be the canonical basis of $\RR^\states$.
        Recall that $(\tau^\arm_k)$ denotes the enumeration of pulls of arm $\arm$, that is, $\tau^\arm_1 := \inf \braces{t \ge 1: \Arm(t) = \arm}$ and $\tau^\arm_{k+1} := \inf \braces{t > \tau^\arm_k: \Arm(t) = \arm}$.
        For every arm $\arm \in \arms$ and all $n \ge 1$ we have 
        \[
            \abs*{
                \sum_{k=1}^{n}
                \parens*{ 
                    \unit_{\State(\tau_k^\arm)}
                    - \unit_{\State(\tau_k^\arm+1)}
                } \bias_\arm
            }
            \le
            \vecspan(\bias_\arm)
            \parens*{
                1
                + 
                \sum_{k=1}^{n}
                \indicator{
                    \tau_k^\arm + 1 \ne \tau_{k+1}^\arm
                }
            }
        \]
    \end{lemma}
    \begin{proof}
        The sum of the LHS is \emph{nearly} a telescopic sum. More precisely, this sum \emph{is} telescopic inside any block of consecutive pulls. Only when there is a \emph{gap} ---  when the algorithm leaves arm $\arm$ and come back later --- can the telescopic sum ``break", and this break cost at most $\vecspan(\bias_\arm)$. Formally,
        \begin{align*}
            &\abs*{
                \sum_{k=1}^{n}
                \parens*{ 
                    \unit_{\State(\tau_k^\arm)}
                    - \unit_{\State(\tau_k^\arm+1)}
                } \bias_\arm
            }
            \\
            & =
            \abs*{
                \bias_\arm \parens*{\State(\tau_1^\arm)}
                -
                \bias_\arm \parens*{\State(\tau_{n+1}^\arm)}
                +
                \sum_{k=1}^{n}
                \parens*{ 
                    \unit_{\State(\tau_{k+1}^\arm)}
                    - \unit_{\State(\tau_k^\arm+1)}
                } \bias_\arm
            }
            \\
            & =
            \abs*{
                \bias_\arm \parens*{\State(\tau_1^\arm)}
                -
                \bias_\arm \parens*{\State(\tau_{n+1}^\arm)}
                +
                \sum_{k=1}^{n}
                \indicator{
                    \State(\tau_k^\arm + 1) \ne \State(\tau_{k+1}^\arm)
                }
                \parens*{
                    \unit_{\State(\tau_{k+1}^\arm)}
                    - \unit_{\State(\tau_k^\arm+1)}
                } \bias_\arm
            }
            \\
             & \le
            \vecspan(\bias_\arm)
            \parens*{
                1
                + 
                \sum_{k=1}^{n}
                \indicator{
                    \State(\tau_k^\arm + 1) \ne \State(\tau_{k+1}^\arm)
                }
            }
            \\
            & \le 
            \vecspan(\bias_\arm)
            \parens*{
                1
                + 
                \sum_{k=1}^{n}
                \indicator{
                    \tau_k^\arm + 1 \ne \tau_{k+1}^\arm
                }
            }
            .
            \qedhere
        \end{align*}
    \end{proof}

    \subsection{Auxiliary probability results}

    In paragraph, we recall the version of Hoeffding inequality that we use in the paper. 

    \begin{lemma}[Hoeffding Lemma, \cite{hoeffding_probability_1963}]
    \label[lemma]{lemma_hoeffding}
        Let $(X_k)$ be a sequence of independent random variables of common mean $\mu$ such that $a \le X_k \le b$ a.s.
        Denoting $\hat{\mu}_n := \frac 1n \sum_{k=1}^n X_k$ for empirical mean, for all $n \ge 1$, $x \ge 0$,
        \begin{equation*}
            \Pr \parens*{
                \hat{\mu}_n - \mu
                \ge 
                x
            }
            \le 
            \exp \parens*{
                - \frac{2 n x^2}{(b-a)^2}
            }
            .
        \end{equation*}
    \end{lemma}

    \ifSubfilesClassLoaded{
        
        \bibliographystyle{plainnat}
        \bibliography{biblio}
    }{}
    
\end{document}

\numberwithin{theorem}{section}
\numberwithin{definition}{section}
\numberwithin{equation}{section}

\clearpage
\tableofcontents

\end{document}